\documentclass[english]{article}

\usepackage{algorithmic}
\usepackage{algorithm}
\usepackage{amsfonts}
\usepackage{amsmath}
\usepackage{amssymb}
\usepackage{amsthm}
\usepackage{array}
\usepackage{arydshln}
\usepackage{bm}
\usepackage{cite}
\usepackage{color}
\usepackage{comment}
\usepackage{dsfont}
\usepackage{enumitem}
\usepackage{float}
\usepackage[T1]{fontenc}
\usepackage{geometry}
\usepackage{graphicx}
\usepackage{grffile}
\usepackage{hyperref}\hypersetup{colorlinks,linkcolor=red,anchorcolor=blue,citecolor=blue}
\usepackage[latin9]{inputenc}
\usepackage{letltxmacro}
\usepackage{mathrsfs}
\usepackage{mathtools}
\usepackage{multirow}
\usepackage{nicefrac}
\usepackage{tablefootnote}
\usepackage{verbatim}

\newcommand{\bbf}{\bm{f}}

\newcommand{\br}{\bm{r}}

\newcommand{\bu}{\bm{u}}
\newcommand{\bv}{\bm{v}}
\newcommand{\bw}{\bm{w}}

\newcommand{\by}{\bm{y}}

\newcommand{\bA}{\bm{A}}
\newcommand{\bB}{\bm{B}}

\newcommand{\bD}{\bm{D}}

\newcommand{\bF}{\bm{F}}
\newcommand{\bG}{\bm{G}}
\newcommand{\bH}{\bm{H}}
\newcommand{\bI}{\bm{I}}

\newcommand{\bL}{\bm{L}}
\newcommand{\bM}{\bm{M}}

\newcommand{\bP}{\bm{P}}
\newcommand{\bQ}{\bm{Q}}
\newcommand{\bR}{\bm{R}}
\newcommand{\bS}{\bm{S}}

\newcommand{\bU}{\bm{U}}
\newcommand{\bV}{\bm{V}}
\newcommand{\bW}{\bm{W}}
\newcommand{\bX}{\bm{X}}

\newcommand{\bDelta}{\bm{\Delta}}

\newcommand{\bSigma}{\bm{\Sigma}}

\newcommand{\cA}{\mathcal{A}}

\newcommand{\cE}{\mathcal{E}}

\newcommand{\cH}{\mathcal{H}}
\newcommand{\cI}{\mathcal{I}}

\newcommand{\cL}{\mathcal{L}}
\newcommand{\cM}{\mathcal{M}}
\newcommand{\cN}{\mathcal{N}}

\newcommand{\cP}{\mathcal{P}}

\newcommand{\cS}{\mathcal{S}}

\newcommand{\bcA}{\bm{\mathcal{A}}}

\newcommand{\bcE}{\bm{\mathcal{E}}}

\newcommand{\bcJ}{\bm{\mathcal{J}}}

\newcommand{\bcS}{\bm{\mathcal{S}}}
\newcommand{\bcT}{\bm{\mathcal{T}}}

\newcommand{\bcW}{\bm{\mathcal{W}}}
\newcommand{\bcX}{\bm{\mathcal{X}}}
\newcommand{\bcY}{\bm{\mathcal{Y}}}

\newcommand{\EE}{\mathbb{E}}

\newcommand{\PP}{\mathbb{P}}

\newcommand{\RR}{\mathbb{R}}

\newcommand{\mfk}{\mathfrak} 
\newcommand{\one}{\bm{1}}
\newcommand{\zero}{\bm{0}}

\newcommand{\argmin}{\mathop{\mathrm{argmin}}}

\DeclareMathOperator{\bcdot}{\boldsymbol{\cdot}}
\DeclareMathOperator{\diag}{\mathrm{diag}}
\DeclareMathOperator{\dist}{\mathrm{dist}}
\DeclareMathOperator{\fro}{\mathsf{F}}
\DeclareMathOperator{\GL}{\mathrm{GL}}
\DeclareMathOperator{\op}{\mathsf{}}
\DeclareMathOperator{\Pdiag}{\mathcal{P}_{\mathsf{diag}}}
\DeclareMathOperator{\Poffdiag}{\mathcal{P}_{\mathsf{off-diag}}}
\DeclareMathOperator{\rank}{\mathrm{rank}}

\DeclareMathOperator{\vc}{\mathrm{vec}}

\DeclareMathOperator{\ptb}{\mathrm{p}}
\DeclareMathOperator{\main}{\mathrm{m}}

\allowdisplaybreaks
\makeatletter
\providecommand{\tabularnewline}{\\}
\setlist[itemize]{leftmargin=1em}
\setlist[enumerate]{leftmargin=1em}

\theoremstyle{plain}\newtheorem{lemma}{\textbf{Lemma}} 

\newtheorem{theorem}{\textbf{Theorem}}

\newtheorem{claim}{\textbf{Claim}} 
\theoremstyle{definition}\newtheorem{definition}{\textbf{Definition}}
 
\theoremstyle{remark}

\definecolor{tian}{RGB}{0,150,0}
\definecolor{cm}{RGB}{250,0,200}
\definecolor{yc}{RGB}{255,0,0}

\begin{document}
\title{Scaling and Scalability: Provable Nonconvex Low-Rank Tensor Estimation from Incomplete Measurements}

 \author
 {
 	Tian Tong\thanks{Department of Electrical and Computer Engineering, Carnegie Mellon University, Pittsburgh, PA 15213, USA; Emails:
 		\texttt{\{ttong1,yuejiec\}@andrew.cmu.edu}.} \\
 		CMU \\
 		\and
 	Cong Ma\thanks{Department of Statistics, University of Chicago, Chicago, IL 60637, USA; Email:
 		\texttt{congm@uchicago.edu}.} \\
 		 UChicago \\
 		\and
		Ashley Prater-Bennette\thanks{Air Force Research Laboratory, Rome, NY 13441, USA; Email: \texttt{\{ashley.prater-bennette,erin.tripp.4\}@us.af.mil}.}\\
		AFRL \\
		\and
		Erin Tripp\footnotemark[3]\\
		AFRL \\
		\and
 	Yuejie Chi\footnotemark[1] \\
 	CMU
 }

\date{April 2021; Revised June 2022}

\setcounter{tocdepth}{2}
\maketitle

\begin{abstract}
Tensors, which provide a powerful and flexible model for representing multi-attribute data and multi-way interactions, play an indispensable role in modern data science  across various fields in science and engineering. A fundamental task is to faithfully recover the tensor from highly incomplete measurements in a statistically and computationally efficient manner. Harnessing the low-rank structure of tensors in the Tucker decomposition, this paper develops a scaled gradient descent (ScaledGD) algorithm to directly recover the tensor factors with tailored spectral initializations, and shows that it provably converges at a linear rate independent of the condition number of the ground truth tensor for two canonical problems --- tensor completion and tensor regression --- as soon as the sample size is above the order of $n^{3/2}$ ignoring other parameter dependencies, where $n$ is the dimension of the tensor. This leads to an extremely scalable approach to low-rank tensor estimation compared with prior art, which suffers from at least one of the following drawbacks: extreme sensitivity to ill-conditioning, high per-iteration costs in terms of memory and computation, or poor sample complexity guarantees. To the best of our knowledge, ScaledGD is the first algorithm that achieves near-optimal statistical and computational complexities simultaneously for low-rank tensor completion with the Tucker decomposition. Our algorithm highlights the power of appropriate preconditioning in accelerating nonconvex statistical estimation, where the iteration-varying preconditioners promote desirable invariance properties of the trajectory with respect to the underlying symmetry in low-rank tensor factorization. 
\end{abstract}


\tableofcontents{}


\section{Introduction}

Tensors \cite{kolda2009tensor,sidiropoulos2017tensor}, which provide a powerful and flexible model for representing multi-attribute data and multi-way interactions across various fields, play an indispensable role in modern data science with ubiquitous applications in image inpainting \cite{liu2012tensor}, hyperspectral imaging \cite{dian2017hyperspectral}, collaborative filtering \cite{xiong2010temporal}, topic modeling \cite{anandkumar2014tensor}, network analysis \cite{papalexakis2016tensors}, and many more.

\subsection{Low-rank tensor estimation}

In many problems across science and engineering, the central task can be regarded as tensor estimation from highly incomplete measurements, where the goal is to estimate an order-3 tensor\footnote{For ease of presentation, we focus on 3-way tensors; our algorithm and theory can be generalized to higher-order tensors in a straightforward manner. } $\bcX_{\star} \in \RR^{n_1\times n_2\times n_3}$  from its observations $\by \in \RR^m$ given by
\begin{align*}
\by \approx  \cA(\bcX_{\star}).
\end{align*}
Here, $\cA: \RR^{n_1\times n_2\times n_3} \mapsto \RR^m$ represents a certain linear map modeling the data collection process. Importantly, the number $m$ of observations is often much smaller than the ambient dimension $n_1 n_2 n_3$ of the tensor due to resource or physical constraints, necessitating the need of exploiting low-dimensional structures to allow for meaningful recovery.

One of the most widely adopted low-dimensional structures---which is the focus of this paper---is the low-rank structure under the {\em Tucker} decomposition \cite{tucker1966some}. Specifically, we assume that the ground truth tensor $\bcX_{\star}$ admits the following Tucker decomposition\footnote{Other popular notation for Tucker decomposition in the literature includes $[\![\bcS_{\star};\bU_{\star},\bV_{\star},\bW_{\star}]\!]$ and $\bcS_{\star}\times_{1}\bU_{\star}\times_{2}\bV_{\star}\times_{3}\bW_{\star}$. In this work, we adopt the same notation $(\bU_{\star},\bV_{\star},\bW_{\star})\bcdot\bcS_{\star}$ as in \cite{xia2019polynomial} for convenience of our theoretical developments.}
\begin{align*}
\bcX_{\star} &=  (\bU_{\star},\bV_{\star},\bW_{\star})\bcdot\bcS_{\star},
\end{align*} 
where $\bcS_{\star}\in \RR^{r_1\times r_2\times r_3}$ is the core tensor, and $\bU_{\star} \in\RR^{n_1\times r_1}$, $\bV_{\star}\in\RR^{n_2\times r_2}$, $\bW_{\star} \in\RR^{n_3\times r_3}$ are orthonormal matrices corresponding to the factors of each mode. The tensor $\bcX_{\star} $ is said to be low-multilinear-rank, or simply low-rank, when its multilinear rank $\br=(r_1,r_2,r_3)$ satisfies $r_k\ll n_k$, for all $k=1,2,3$.  Compared with other tensor decompositions such as the CP decomposition \cite{kolda2009tensor} and tensor-SVD \cite{zhang2014novel}, the Tucker decomposition offers several advantages: it allows flexible modeling of  low-rank tensor factors with a small number of parameters, fully exploits the multi-dimensional algebraic structure of a tensor, and admits efficient and stable computation without suffering from degeneracy \cite{paatero2000construction}.

\paragraph{Motivating examples.} We point out two representative settings of tensor recovery that guide our work.
\begin{itemize}
\item {\em Tensor completion.} A widely encountered problem is tensor completion, where one aims to predict the entries in a tensor from only a small subset of its revealed entries. A celebrated application is collaborative filtering, where one aims to predict the users'  evolving preferences from partial observations of a tensor composed of ratings for any triplet of {\em user, item, time} \cite{karatzoglou2010multiverse}. Mathematically, we are given entries
\begin{align*}
\bcX_{\star}(i_1,i_2,i_3), \qquad (i_1,i_2,i_3) \in \Omega,
\end{align*}
in some index set $\Omega$, where $(i_1,i_2,i_3) \in \Omega$ if and only if that entry is observed. The goal is then to recover the low-rank tensor $\bcX_{\star}$ from the observed entries in $\Omega$.

\item {\em Tensor regression.} In machine learning and signal processing, one is often concerned with determining how the covariates relate to the response---a task known as regression. Due to advances in data acquisition, there is no shortage of scenarios where the covariates are available in the form of tensors, for example in medical imaging \cite{zhou2013tensor}. Mathematically, the $i$-th response or observation is given as
\begin{align*}
y_i = \langle \bcA_i, \bcX_{\star} \rangle = \sum_{i_1,i_2,i_3} \bcA_i(i_1,i_2,i_3) \bcX_{\star}(i_1,i_2,i_3), \qquad i =1,2,\ldots, m,
\end{align*}
where $\bcA_i$ is the $i$-th covariate or measurement tensor. The goal is then to recover the low-rank tensor $\bcX_{\star}$ from the responses $\by = \{y_i\}_{i=1}^m$.
\end{itemize}

\subsection{A gradient descent approach?}

Recent years remarkable successes have emerged in developing a plethora of provably efficient algorithms for low-rank \emph{matrix} estimation (i.e.~the special case of order-2 tensors) via both convex and nonconvex optimization. However, unique challenges arise when dealing with tensors, since they have more sophisticated algebraic structures \cite{hackbusch2012tensor}. For instance, while nuclear norm minimization achieves near-optimal statistical guarantees for low-rank matrix estimation \cite{candes2010NearOptimalMC} within a polynomial run time, computing the nuclear norm of a tensor turns out to be NP-hard \cite{friedland2018nuclear}. Therefore, there have been a number of efforts to develop polynomial-time algorithms for tensor recovery, including but not limited to the sum-of-squares hierarchy \cite{barak2016noisy,potechin2017exact}, nuclear norm minimization with unfolding \cite{gandy2011tensor,mu2014square}, regularized gradient descent \cite{han2020optimal}, to name a few; see Section~\ref{sec:prior_arts} for further discussions.

In view of the low-rank Tucker decomposition, a natural approach is to seek to recover the factor quadruple $\bF_{\star} \coloneqq (\bU_{\star},\bV_{\star},\bW_{\star},\bcS_{\star})$ directly by optimizing the unconstrained least-squares loss function: 
\begin{align}\label{eq:loss}
\min_{\bF}\quad \cL(\bF)\coloneqq\frac{1}{2}\left\|\cA\left((\bU,\bV,\bW)\bcdot\bcS \right)-\by\right\|_{2}^2,
\end{align}
where $\bF \coloneqq (\bU,\bV,\bW,\bcS)$ consists of $\bU\in\RR^{n_1\times r_1}$, $\bV\in\RR^{n_2\times r_2}$, $\bW\in\RR^{n_3\times r_3}$, and $\bcS\in\RR^{r_1\times r_2\times r_3}$.  Since the factors have a much lower complexity than the tensor itself due to the low-rank structure, it is expected that manipulating the factors results in more scalable algorithms in terms of both computation and storage. This optimization problem is however, highly nonconvex, since the factors are not uniquely determined.\footnote{For any invertible matrices $\bQ_k \in\RR^{r_k\times r_k}$, $k=1,2,3$, one has $(\bU,\bV,\bW)\bcdot\bcS =(\bU\bQ_1,\bV\bQ_2,\bW\bQ_3)\bcdot ((\bQ_1^{-1},\bQ_2^{-1},\bQ_3^{-1}) \bcdot \bcS)$.}
 Nonetheless, one might be tempted to solve the problem~\eqref{eq:loss} via gradient descent (GD), which updates the factors according to
\begin{align}\label{eq:GD}
\bF_{t+1} =  \bF_t  - \eta \nabla \cL(\bF_t), \qquad t=0,1,\ldots ,
\end{align}
where $\bF_t$ is the estimate at the $t$-th iteration, $\eta>0$ is the step size or learning rate, and $\nabla\cL(\bF)$ is the gradient of $\cL(\bF)$ at $\bF$. Despite a flurry of activities for understanding factored gradient descent in the matrix setting \cite{chi2019nonconvex}, this line of algorithmic thinkings has been severely under-explored for the tensor setting, especially when it comes to provable guarantees for both sample and computational complexities.  

The closest existing theory that one comes across is \cite{han2020optimal} for tensor regression, which adds regularization terms to promote the orthogonality of the factors $\bU,\bV,\bW$: 
\begin{align}\label{eq:reg_loss}
\cL_{\mathsf{reg}}(\bF)\coloneqq\cL(\bF) + \frac{\alpha}{4}\left(\|\bU^{\top}\bU-\beta\bI_{r_1} \|_{\fro}^2 + \| \bV^{\top}\bV-\beta\bI_{r_2} \|_{\fro}^2 + \|\bW^{\top}\bW-\beta\bI_{r_3} \|_{\fro}^2 \right),
\end{align}
and perform GD on the regularized loss. Here, $\alpha,\beta>0$ are two parameters to be specified. While encouraging, theoretical guarantees of this regularized GD algorithm \cite{han2020optimal} still fall short of achieving computational efficiency. In truth, its convergence speed is rather slow:  it takes an order of $\kappa^2\log(1/\varepsilon)$ iterations to attain an $\varepsilon$-accurate estimate of the ground truth tensor, where $\kappa$ is a sort of condition number of  $\bcX_{\star}$ to be defined momentarily. Therefore, the computational efficacy of the regularized GD is severely hampered even when $\bcX_{\star}$ is moderately ill-conditioned, a situation frequently encountered in practice.
In addition, the regularization term introduces additional parameters that may be difficult to tune optimally in practice. 

Turning to  tensor completion, the situation is even worse: to the best of our knowledge, there is {\em no} provably linearly-convergent algorithm that accommodates low-rank tensor completion under the Tucker decomposition. The question is thus:
\begin{itemize}
\item[]{\em Can we develop a factored gradient-based algorithm that converges fast even for highly ill-conditioned tensors with near-optimal sample complexities for tensor completion and tensor regression?}
\end{itemize}
In this paper, we provide an affirmative answer to the above question.

\subsection{A new algorithm: scaled gradient descent}

We propose a novel algorithm---dubbed scaled gradient descent ({ScaledGD})---to solve the tensor recovery problem. More specifically, at the core it performs the following iterative updates\footnote{The matrix inverses in ScaledGD always exist under the assumptions of our theory.} to minimize the loss function~\eqref{eq:loss}:
\begin{align}
\begin{split}
\bU_{t+1} &= \bU_{t} - \eta\nabla_{\bU}\cL(\bF_{t})\big(\breve{\bU}_t^{\top} \breve{\bU}_t \big)^{-1}, \\
\bV_{t+1} &= \bV_{t} - \eta\nabla_{\bV}\cL(\bF_{t})\big(\breve{\bV}_t^{\top} \breve{\bV}_t \big)^{-1}, \\
\bW_{t+1} &= \bW_{t} - \eta\nabla_{\bW}\cL(\bF_{t})\big(\breve{\bW}_t^{\top} \breve{\bW}_t \big)^{-1}, \\
\bcS_{t+1} &= \bcS_{t} - \eta\left((\bU_{t}^{\top}\bU_{t})^{-1},(\bV_{t}^{\top}\bV_{t})^{-1},(\bW_{t}^{\top}\bW_{t})^{-1}\right)\bcdot\nabla_{\bcS}\cL(\bF_{t}),
\end{split}\label{eq:ScaledGD}
\end{align}
where $\nabla_{\bU}\cL(\bF)$, $\nabla_{\bV}\cL(\bF)$, $\nabla_{\bW}\cL(\bF)$, and  $\nabla_{\bcS}\cL(\bF)$ are the partial derivatives of $\cL(\bF)$ with respect to the corresponding variables, and 
\begin{align} \label{eq:breve_uvw}
\begin{split}
\breve{\bU}_{t}&\coloneqq\cM_{1}\left((\bI_{r_1},\bV_{t},\bW_{t})\bcdot\bcS_{t}\right)^{\top} = (\bW_{t}\otimes\bV_{t})\cM_{1}(\bcS_{t})^{\top}, \\
\breve{\bV}_{t}&\coloneqq\cM_{2}\left((\bU_{t},\bI_{r_2},\bW_{t})\bcdot\bcS_{t}\right)^{\top}=(\bW_{t}\otimes\bU_{t})\cM_{2}(\bcS_{t})^{\top},\\
\breve{\bW}_{t}&\coloneqq\cM_{3}\left((\bU_{t},\bV_{t},\bI_{r_3})\bcdot\bcS_{t}\right)^{\top}=(\bV_{t}\otimes\bU_{t})\cM_{3}(\bcS_{t})^{\top}.
\end{split}
\end{align}
Here, $\cM_k(\bcS)$ is the matricization of the tensor $\bcS$ along the $k$-th mode ($k=1,2,3$), and $\otimes$ denotes the Kronecker product. Inspired by its variant in the matrix setting~\cite{tong2021accelerating}, the ScaledGD algorithm \eqref{eq:ScaledGD} exploits the structures of Tucker decomposition and possesses many desirable properties:
\begin{itemize}
\item {\em Low per-iteration cost:} as a preconditioned GD or quasi-Newton algorithm, ScaledGD updates the factors along the descent direction of a scaled gradient, where the preconditioners can be viewed as the inverse of the diagonal blocks of the Hessian for the population loss (i.e.~tensor factorization). As the sizes of the preconditioners are proportional to the multilinear rank, the matrix inverses are cheap to compute with a minimal overhead and the overall per-iteration cost is still low and linear in the time it takes to read the input data.
\item {\em Equivariance to parameterization:} one crucial property of ScaledGD is that if we reparameterize the  factors by some invertible transforms (i.e.~replacing $(\bU_t, \bV_t, \bW_t, \bcS_t)$ by $(\bU_t\bQ_1, \bV_t\bQ_2, \bW_t\bQ_3, (\bQ_1^{-1}, \bQ_2^{-1}, \bQ_3^{-1})\bcdot\bcS_t)$ for some invertible matrices $\{\bQ_k\}_{k=1}^{3}$), the entire trajectory will go through the same reparameterization, leading to an {\em invariant} sequence of low-rank tensor updates $\bcX_t= (\bU_t, \bV_t, \bW_t)\bcdot\bcS_t$ regardless of the parameterization being adopted. 

\item {\em Implicit balancing:} ScaledGD optimizes the natural loss function \eqref{eq:loss} in an {\em unconstrained} manner without requiring additional regularizations or orthogonalizations used in prior literature  \cite{han2020optimal,frandsen2020optimization,kasai2016low}, and the factors stay balanced in an automatic manner---a feature sometimes referred to as implicit regularization \cite{ma2021beyond}.
\end{itemize}

\begin{table}[t]
\centering %
\begin{tabular}{c||c|c|c}
\hline 
Algorithms & Sample complexity & Iteration complexity & Parameter space   \tabularnewline
\hline  \hline 
Unfolding + nuclear norm min.  & \multirow{2}{*}{$n^2 r \log^2 n$} & \multirow{2}{*}{polynomial} & \multirow{2}{*}{tensor}  \tabularnewline 
\cite{huang2015provable} &  &  &  \tabularnewline \hline 
Tensor nuclear norm min.  & \multirow{2}{*}{$ n^{3/2}r^{1/2} \log^{3/2} n$} & \multirow{2}{*}{NP-hard} & \multirow{2}{*}{tensor}  \tabularnewline
\cite{yuan2016tensor} &  &  &  \tabularnewline \hline 
Grassmannian GD  & \multirow{2}{*}{$n^{3/2}r^{7/2}\kappa^4 \log^{7/2}n$} & \multirow{2}{*}{N/A} & \multirow{2}{*}{factor}  \tabularnewline
\cite{xia2019polynomial} &  &  &  \tabularnewline \hline 
 ScaledGD  & \multirow{2}{*}{$n^{3/2}r^{5/2}\kappa^{3}\log^3 n$ } & \multirow{2}{*}{$\log\frac{1}{\varepsilon}$} & \multirow{2}{*}{factor}  \tabularnewline
(this paper)  &  &  &  \tabularnewline
\hline
\end{tabular}\vspace{0.04in}
\caption{Comparisons of ScaledGD with existing algorithms for tensor completion when the tensor is incoherent and low-rank under the Tucker decomposition. Here, we say that the output $\bcX$ of an algorithm reaches $\varepsilon$-accuracy, if it satisfies $\|\bcX-\bcX_{\star}\|_{\fro}\le\varepsilon\sigma_{\min}(\bcX_{\star})$. Here, $\kappa$ and $\sigma_{\min}(\bcX_{\star})$ are the condition number and the minimum singular value of $\bcX_{\star}$ (defined in Section~\ref{sec:models}).  For simplicity, we let $n = \max_{k=1,2,3} n_k$ and $r =\max_{k=1,2,3}r_k$, and assume $r\vee \kappa\ll n^{\delta}$ for some small constant $\delta$ to keep only terms with dominating orders of $n$. 
\label{tab:ScaledGD-tensor-completion}  }
\end{table}

\begin{table}[t]
\centering %
\begin{tabular}{c||c|c|c}
\hline 
Algorithms  & Sample complexity & Iteration complexity & Parameter space   \tabularnewline
\hline  \hline 
Unfolding + nuclear norm min. & \multirow{2}{*}{$n^2 r$} & \multirow{2}{*}{polynomial} & \multirow{2}{*}{tensor}  \tabularnewline  
\cite{mu2014square}  &  &  &  \tabularnewline 
\hline 
Projected GD   & \multirow{2}{*}{$n^2 r$} & \multirow{2}{*}{$\kappa^2\log\frac{1}{\varepsilon}$} & \multirow{2}{*}{tensor}  \tabularnewline
\cite{chen2019non} &  &  &  \tabularnewline
\hline 
Regularized GD  & \multirow{2}{*}{$n^{3/2}r\kappa^{4}$} & \multirow{2}{*}{$\kappa^2\log\frac{1}{\varepsilon}$} & \multirow{2}{*}{factor}  \tabularnewline
\cite{han2020optimal} &  &  &  \tabularnewline 
\hline 
Riemannian Gauss-Newton & \multirow{2}{*}{$n^{3/2}r^{3/2}\kappa^{4}$} & \multirow{2}{*}{$\log\log\frac{1}{\varepsilon}$} & \multirow{2}{*}{tensor}  \tabularnewline
\cite{zhang2021low} (concurrent)\tablefootnote{\cite[Theorem~3]{zhang2021low} states the sample complexity $n^{3/2}\sqrt{r}\kappa^{2} \|\bcX_{\star}\|_{\fro}^{2} /\sigma_{\min}^{2}(\bcX_{\star})$, where $\|\bcX_{\star}\|_{\fro}^{2} /\sigma_{\min}^{2}(\bcX_{\star})$ has an order of $r\kappa^{2}$.} &  &  &  \tabularnewline 
\hline 
ScaledGD  & \multirow{2}{*}{$n^{3/2}r^{3/2}\kappa^{2}$} & \multirow{2}{*}{$\log\frac{1}{\varepsilon}$}   & \multirow{2}{*}{factor}  \tabularnewline
(this paper)  &  &  &  \tabularnewline
\hline
\end{tabular}\vspace{0.04in}
\caption{Comparisons of ScaledGD with existing algorithms for tensor regression when the tensor is low-rank under the Tucker decomposition. Here, we say that the output $\bcX$ of an algorithm reaches $\varepsilon$-accuracy, if it satisfies $\|\bcX-\bcX_{\star}\|_{\fro}\le\varepsilon\sigma_{\min}(\bcX_{\star})$. Here,  $\kappa$ and $\sigma_{\min}(\bcX_{\star})$ are the condition number and minimum singular value of $\bcX_{\star}$ (defined in Section~\ref{sec:models}). For simplicity, we let $n = \max_{k=1,2,3} n_k$, and $r =\max_{k=1,2,3}r_k$, and assume $r\vee \kappa\ll n^{\delta}$ for some small constant $\delta$ to keep only terms with dominating orders of $n$.
\label{tab:ScaledGD-tensor-regression}  }
\end{table}

\paragraph{Theoretical guarantees.} We investigate the theoretical properties of ScaledGD for both tensor completion and tensor regression, which are notably more challenging than the matrix counterpart. It is demonstrated that ScaledGD---when initialized properly using appropriate spectral methods ---achieves linear convergence at a rate {\em independent} of the condition number of the ground truth tensor with near-optimal sample complexities. In other words, ScaledGD needs no more than $O(\log (1/\varepsilon))$ iterations to reach $\varepsilon$-accuracy; together with its low computational and memory costs by operating in the factor space, this makes ScaledGD a highly scalable method for a wide range of low-rank tensor estimation tasks.  More specifically, we have the following guarantees (assume $n = \max_{k=1,2,3} n_k$ and $r =\max_{k=1,2,3}r_k$): 
\begin{itemize}
\item {\em Tensor completion.} Under the Bernoulli sampling model, ScaledGD (with an additional scaled projection step) succeeds with high probability as long as the sample complexity is above the order of $n^{3/2}r^{5/2}\kappa^{3}\log^3 n$. Connected to some well-reckoned conjecture on computational barriers, it is widely believed that no polynomial-time algorithm will be successful if the sample complexity is less than the order of $n^{3/2}$ for tensor completion \cite{barak2016noisy}, which suggests the near-optimality of the sample complexity of ScaledGD. Compared with existing approaches (cf.~Table~\ref{tab:ScaledGD-tensor-completion}), ScaledGD provides the first computationally efficient algorithm with a near-linear run time at the near-optimal sample complexity. 
\item {\em Tensor regression.} Under the Gaussian design, ScaledGD succeeds with high probability as long as the sample complexity is above the order of $n^{3/2} r^{3/2}\kappa^{2}$. Our analysis of local convergence is more general, based on the tensor restricted isometry property (TRIP) \cite{rauhut2017low}, and is therefore applicable to various measurement ensembles that satisfy TRIP. Compared with existing approaches (cf.~Table~\ref{tab:ScaledGD-tensor-regression}), ScaledGD achieves competitive performance guarantees in terms of sample and iteration complexities with a low per-iteration cost in the factor space. 
\end{itemize}  

Figure~\ref{fig:TC_kappa} illustrates the number of iterations needed to achieve a relative error $\|\bcX-\bcX_{\star}\|_{\fro}\le 10^{-3} \|\bcX_{\star}\|_{\fro}$ for ScaledGD and regularized GD \cite{han2020optimal} under different condition numbers for tensor completion under the Bernoulli sampling model (see Section~\ref{sec:numerical} for experimental settings). Clearly, the iteration complexity of GD deteriorates at a super linear rate with respect to the condition number $\kappa$, while ScaledGD enjoys an iteration complexity that is independent of $\kappa$ as predicted by our theory. Indeed, with a seemingly small modification, ScaledGD takes merely $17$ iterations to achieve the desired accuracy over the entire range of $\kappa$, while GD takes thousands of iterations even with a moderate condition number! 
\begin{figure}[t]
\centering
\includegraphics[width=0.5\textwidth]{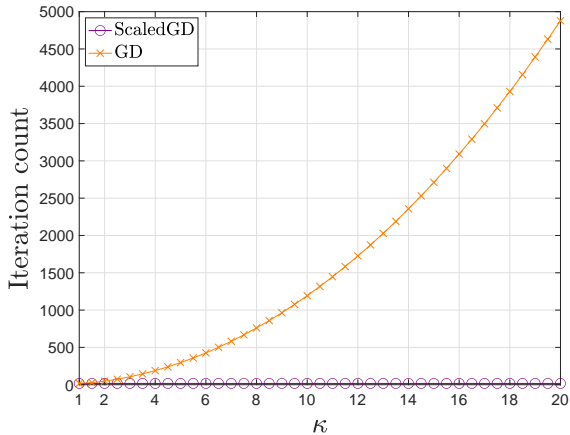} 
\caption{The iteration complexities of ScaledGD (this paper) and regularized GD to achieve $\|\bcX-\bcX_{\star}\|_{\fro}\le 10^{-3} \|\bcX_{\star}\|_{\fro}$ with respect to different condition numbers for low-rank tensor completion with $n_1=n_2=n_3=100$, $r_1=r_2=r_3=5$, and the probability of observation $p=0.1$.}\label{fig:TC_kappa}
\end{figure}

\subsection{Additional related works}\label{sec:prior_arts}

\paragraph{Comparison with \cite{tong2021accelerating}.}
While the proposed ScaledGD algorithm is inspired by its matrix variant in  \cite{tong2021accelerating} by utilizing the same principle of preconditioning, the exact form of preconditioning for tensor factorization needs to be designed carefully and is not trivially obtainable.
There are many technical novelty in our analysis compared to \cite{tong2021accelerating}. In the matrix case, the low-rank matrix is factorized as $\boldsymbol{L}\boldsymbol{R}^{\top}$, and only two factors are needed to be estimated. In contrast, in the tensor case, the low-rank tensor is factorized as $(\boldsymbol{U},\boldsymbol{V},\boldsymbol{W}) \boldsymbol{\cdot} \boldsymbol{\mathcal{S}}$, and four factors are needed to be estimated, leading to a much more complicated nonconvex landscape than the matrix case. In fact, when specialized to matrix completion, our ScaledGD algorithm does not degenerate to the same matrix variant in \cite{tong2021accelerating}, due to overparamterization and estimating four factors at once, but still maintains the near-optimal performance guarantees. In addition, the tensor algebra possesses unique algebraic properties that requires much more delicate treatments in the analysis. For the local convergence, we establish new concentration properties regarding tensors, which are more challenging compared to the matrix counterparts; for spectral initialization, we establish the effectiveness of a second-order spectral method in the Tucker setting for the first time.

\paragraph{Low-rank tensor estimation with Tucker decomposition.} \cite{frandsen2020optimization} analyzed the landscape of Tucker decomposition for tensor factorization, and showed benign landscape properties with suitable regularizations.  \cite{gandy2011tensor,mu2014square} developed convex relaxation algorithms based on minimizing the nuclear norms of unfolded tensors for tensor regression, and similar approaches were developed in \cite{huang2015provable} for robust tensor completion. However, unfolding-based approaches typically result in sub-optimal sample complexities since they do not fully exploit the tensor structure. \cite{yuan2016tensor} studied directly minimizing the nuclear norm of the tensor, which regrettably is not computationally tractable. \cite{xia2019polynomial} proposed a Grassmannian gradient descent algorithm over the factors other than the core tensor for exact tensor completion, whose iteration complexity is not characterized. The statistical rates of tensor completion, together with a spectral method, were investigated in \cite{zhang2018tensor,xia2021statistically}, and uncertainty quantifications were recently dealt with in \cite{xia2020inference}. Besides the entrywise i.i.d.~observation models for tensor completion, \cite{zhang2019cross,krishnamurthy2013low} considered tailored or adaptive observation patterns to improve the sample complexity. In addition, for low-rank tensor regression, \cite{raskutti2019convex} proposed a general convex optimization approach based on decomposable regularizers, and \cite{rauhut2017low} developed an iterative hard thresholding algorithm. \cite{chen2019non} proposed projected gradient descent algorithms with respect to the tensors, which have larger computation and memory footprints than the factored gradient descent approaches taken in this paper. \cite{ahmed2020tensor} proposed a tensor regression model where the tensor is simultaneously low-rank and sparse in the Tucker decomposition. A concurrent work \cite{zhang2021low} proposed a Riemannian Gauss-Newton algorithm, and obtained an impressive quadratic convergence rate for tensor regression (see Table~\ref{tab:ScaledGD-tensor-regression}). Compared with ScaledGD, this algorithm runs in the tensor space, and the update rule is more sophisticated with higher per-iteration cost by solving a least-squares problem and performing a truncated HOSVD every iteration. Another recent work \cite{cai2021generalized} studies the Riemannian gradient descent algorithm which also achieves an iteration complexity free of condition number, however, the initialization scheme was not studied therein. After the initial appearance of the current paper, another work \cite{wang2021implicit} proposes an algorithm based again on Riemmannian gradient descent for low-rank tensor completion with Tucker decomposition, coming with an appealing entrywise convergence guarantee at a constant rate.

Last but not least, many scalable algorithms for low-rank tensor estimation have been proposed in the literature of numerical optimization \cite{xu2013block,goldfarb2014robust}, where preconditioning has long been recognized as a key idea to accelerate convergence \cite{kasai2016low,kressner2014low}. In particular, if we constrain $\bU,\bV,\bW$ to be orthonormal, i.e.~on the Grassmanian manifold, the preconditioners used in ScaledGD degenerate to the ones investigated in \cite{kasai2016low}, which was a Riemannian manifold gradient algorithm under a scaled metric. On the other hand, ScaledGD does not assume orthonormality of the factors, therefore is conceptually simpler to understand and avoids complicated manifold operations (e.g.~geodesics, retraction). Furthermore, none of the prior algorithmic developments 
\cite{kasai2016low,kressner2014low} are endowed with the type of global performance guarantees with linear convergence rate as developed herein.

\paragraph{Provable low-rank tensor estimation with other decompositions.} Complementary to ours, there have also been a growing number of algorithms proposed for estimating a low-rank tensor adopting the CP decomposition. Examples include sum-of-squares hierarchy \cite{barak2016noisy,potechin2017exact}, gradient descent \cite{cai2019nonconvex,cai2020uncertainty,hao2020sparse}, alternating minimization \cite{jain2014provable,liu2020tensor}, spectral methods \cite{montanari2018spectral,chen2021spectral,cai2019subspace}, atomic norm minimization \cite{li2015overcomplete,ghadermarzy2019near}, to name a few. \cite{ge2020optimization} studied the optimization landscape of overcomplete CP tensor decomposition. Beyond the CP decomposition, \cite{zhang2016exact} developed exact tensor completion algorithms under the so-called tensor-SVD \cite{zhang2014novel}, and \cite{liu2019low,lu2018exact} studied low-tubal-rank tensor recovery. We will not elaborate further since these algorithms are not directly comparable to ours due to the difference in models.

\paragraph{Nonconvex optimization for statistical estimation.} Our work contributes to the recent strand of works that develop provable nonconvex methods for statistical estimation, including but not limited to low-rank matrix estimation \cite{sun2016guaranteed,chen2015fast,ma2017implicit,charisopoulos2019low,ma2021beyond,park2017non,chen2019nonconvex,xia2021statistical}, phase retrieval \cite{candes2015phase,wang2017solving,chen2015solving,zhang2017reshaped,zhang2016provable,chen2019gradient}, quadratic sampling \cite{li2019nonconvex}, dictionary learning \cite{sun2017complete,sun2017trust,bai2018subgradient}, neural network training \cite{buchanan2020deep,fu2020guaranteed,hand2019global}, and blind deconvolution \cite{li2019rapid,ma2017implicit,shi2021manifold}; the readers are referred to the overviews \cite{chi2019nonconvex,chen2018harnessing,zhang2020symmetry} for further references. 

\subsection{A primer on tensor algebra and notation}
 \label{sec:notation}

We end this section with a primer on some useful tensor algebra; for a more detailed exposition, see \cite{kolda2009tensor,sidiropoulos2017tensor}. Throughout this paper, we use boldface calligraphic letters (e.g.~$\bcX$) to denote tensors, and boldface capitalized letters (e.g.~$\bX$) to denote matrices. For any matrix $\bM$, we use $\sigma_{i}(\bM)$ to denote its $i$-th largest singular value, and $\sigma_{\max}(\bM)$ (resp.~$\sigma_{\min}(\bM)$) to denote its largest (resp.~smallest) nonzero singular value. $\|\bM\|_{\op}$, $\|\bM\|_{\fro}$, $\|\bM\|_{2,\infty}$, and $\|\bM\|_{\infty}$ stand for the spectral norm (i.e.~the largest singular value), the Frobenius norm, the $\ell_{2,\infty}$ norm (i.e.~the largest $\ell_2$ norm of the rows), and the entrywise $\ell_{\infty}$ norm (the largest magnitude of all entries) of a matrix $\bM$. Let $\Pdiag(\bM)$ denote the projection that keeps only the diagonal entries of $\bM$, and $\Poffdiag(\bM)=\bM-\Pdiag(\bM)$, for a square matrix $\bM$. Let $\bM(i,:)$ and $\bM(:,j)$ denote the $i$-th row and $j$-th column of $\bM$, respectively. The $r\times r$ identity matrix is denoted by $\bI_{r}$. The set of invertible matrices in $\RR^{r\times r}$ is denoted by $\GL(r)$.

We define the unfolding (i.e.~flattening) operations of tensors and matrices as following. 
\begin{itemize}
\item The mode-$1$ matricization $\cM_1(\bcX) \in \RR^{n_1\times (n_2n_3)}$ of a tensor $\bcX \in \RR^{n_1\times n_2\times n_3}$ is given by 
$[\cM_1(\bcX)]\big(i_1, i_2 + (i_3-1)n_2\big) = \bcX(i_1,i_2,i_3)$, for $1\le i_k \le n_k$, $k=1,2,3$; $\cM_2(\bcX)$ and $\cM_3(\bcX)$ can be defined in a similar manner.
\item The vectorization $\vc(\bcX)\in\RR^{n_1n_2n_3}$ of a tensor $\bcX\in\RR^{n_1\times n_2\times n_3}$ is given by $[\vc(\bcX)]\big(i_1 + (i_2-1)n_1 + (i_3-1)n_1n_2\big) = \bcX(i_1,i_2,i_3)$ for $1\le i_k \le n_k$, $k=1,2,3$.
\item The vectorization $\vc(\bM)\in\RR^{n_1n_2}$ of a matrix $\bM\in\RR^{n_1\times n_2}$ is given by $[\vc(\bM)]\big(i_1 + (i_2-1)n_1\big)=\bM (i_1,i_2)$ for $1\le i_k \le n_k$, $k=1,2$.
\end{itemize}
The vectorization of a tensor is related to the Kronecker product as 
\begin{subequations}
\begin{align} \label{eq:tensor_vec}
\vc((\bU,\bV,\bW)\bcdot\bcS) = \vc\left(\bU\cM_{1}(\bcS)(\bW\otimes\bV)^{\top}\right)=(\bW\otimes\bV\otimes\bU)\vc(\bcS).
\end{align}
The inner product between two tensors is defined as 
$$\langle\bcX_{1},\bcX_{2}\rangle = \sum_{i_1,i_2,i_3} \bcX_{1} (i_1,i_2,i_3) \bcX_{2} (i_1,i_2,i_3).$$ A useful relation is that 
\begin{align}\label{eq:tensor_inner}
\langle\bcX_{1},\bcX_{2}\rangle=\langle\cM_{k}(\bcX_{1}),\cM_{k}(\bcX_{2})\rangle , \quad k=1,2,3,
\end{align}
which allows one to move between the tensor representation and the unfolded matrix representation.
The Frobenius norm of a tensor is defined as $\|\bcX\|_{\fro}=\sqrt{\langle\bcX,\bcX\rangle}$. In addition, the following basic relations, which follow straightforwardly from analogous matrix relations after applying matricizations, will be proven useful: 
\begin{align} 
(\bU,\bV,\bW)\bcdot \big((\bQ_{1},\bQ_{2},\bQ_{3})\bcdot\bcS \big) &=(\bU\bQ_{1},\bV\bQ_{2},\bW\bQ_{3})\bcdot\bcS, \label{eq:tensor_properties_c} \\
\left\langle(\bU,\bV,\bW)\bcdot\bcS, \bcX\right\rangle &= \left\langle\bcS, (\bU^{\top},\bV^{\top},\bW^{\top})\bcdot\bcX\right\rangle, \label{eq:tensor_properties_d} \\
\left\|(\bQ_{1},\bQ_{2},\bQ_{3})\bcdot\bcS\right\|_{\fro} &\le \|\bQ_{1}\|_{\op}\|\bQ_{2}\|_{\op}\|\bQ_{3}\|_{\op}\|\bcS\|_{\fro}, \label{eq:tensor_properties_e}
\end{align}
\end{subequations}
where $\bQ_k \in \RR^{r_k\times r_k}$, $k=1,2,3$. 
Define the $\ell_{\infty}$ norm of $\bcX$ as $\|\bcX\|_{\infty} = \max_{i_1,i_2,i_3}|\bcX(i_1,i_2,i_3)|$. With slight abuse of terminology, denote 
\begin{align*}
\sigma_{\max}(\bcX) = \max_{k=1,2,3} \sigma_{\max}(\cM_k(\bcX)), \quad\mbox{ and} \quad \sigma_{\min}(\bcX) = \min_{k=1,2,3} \sigma_{\min}(\cM_k(\bcX))
\end{align*} 
as the maximum and minimum nonzero singular values of $\bcX$.
In addition, define the spectral norm of a tensor $\bcX$ as
\begin{align*}
\|\bcX\|_{\op} = \sup_{\bu_k\in\RR^{n_k}:\, \|\bu_k\|_2\le 1}\left|\left\langle\bcX, (\bu_1,\bu_2,\bu_3)\bcdot 1\right\rangle\right|. 
\end{align*}
Note that $\|\bcX\|_{\op} \neq \sigma_{\max}(\bcX)$ in general. For a tensor $\bcX$ of multilinear rank at most $\br=(r_1,r_2,r_3)$, its spectral norm is related to the Frobenius norm as \cite{jiang2017tensor,li2018orthogonal}
\begin{align} \label{eq:tensor_spec_frob}
\|\bcX\|_{\fro} \le \sqrt{\frac{r_1 r_2 r_3}{r}} \|\bcX\|_{\op}, \qquad\mbox{where}~ r= \max_{k=1,2,3} r_k.
\end{align}


\paragraph{Higher-order SVD.} For a general tensor $\bcX$, define $\cH_{\br}(\bcX)$ as the top-$\br$ higher-order SVD (HOSVD) of $\bcX$ with $\br= (r_1,r_2,r_3)$, given by 
\begin{subequations}
\begin{align} \label{eq:HOSVD}
\cH_{\br}(\bcX)=(\bU,\bV,\bW)\bcdot\bcS,
\end{align}
where $\bU$ is the top-$r_1$ left singular vectors of $\cM_{1}(\bcX)$, $\bV$ is the top-$r_2$ left singular vectors of $\cM_{2}(\bcX)$, $\bW$ is the top-$r_3$ left singular vectors of $\cM_{3}(\bcX)$, and $\bcS = (\bU^{\top},\bV^{\top},\bW^{\top})\bcdot\bcX$ is the core tensor. Equivalently, we denote 
\begin{align} \label{eq:HOSVD_factor}
(\bU,\bV,\bW,\bcS) = \mathrm{HOSVD}_{\br}(\bcX) 
\end{align}
\end{subequations}
as the output to the HOSVD procedure described above with the multilinear rank $\br$. In contrast to the matrix case, HOSVD is not guaranteed to yield the optimal rank-$\br$ approximation of $\bcX$ (which is NP-hard \cite{hillar2013most} to find). Nevertheless, it yields a quasi-optimal approximation \cite{hackbusch2012tensor} in the sense that 
\begin{align} \label{eq:HOSVD_quasi_optimal}
\|\bcX-\cH_{\br}(\bcX)\|_{\fro} \le \sqrt{3}\inf_{\widetilde{\bcX}: \, \rank(\cM_k(\widetilde{\bcX}))\le r_k}\|\bcX-\widetilde{\bcX}\|_{\fro}.
\end{align}
There are many variants or alternatives of HOSVD in the literature, e.g.~successive HOSVD, alternating least squares (ALS), higher-order orthogonal iteration (HOOI) \cite{de2000multilinear,de2000best}, etc. These methods compute truncated singular value decompositions in successive or alternating manners, to either reduce the computational costs or pursue a better (but still quasi-optimal) approximation. We will not delve into the details of these variants; interested readers can consult \cite{hackbusch2012tensor}.

\paragraph{Additional notation.} Let $a\vee b=\max\{a,b\}$ and $a\wedge b=\min\{a,b\}$. Throughout, $f(n)\lesssim g(n)$ or $f(n)=O(g(n))$ means $|f(n)|/|g(n)|\le C$ for some constant $C>0$, $f(n)\gtrsim g(n)$ means $|f(n)|/|g(n)|\ge C$ for some constant $C>0$, and $f(n)\asymp g(n)$ means $C_1\le |f(n)|/|g(n)|\le C_2$ for some constants $C_1,C_2>0$. Additionally, $f(n)\ll g(n)$ indicates $|f(n)|/|g(n)|\le c$ for some sufficient small constant $c>0$, and $f(n)\gg g(n)$ indicates $|f(n)|/|g(n)|\ge C$ for some sufficient large constant $C>0$. We use $C, C_1, C_2, c, c_1, c_2\dots$ to represent positive constants, whose values may differ from line to line. Last but not least, we use the terminology ``with overwhelming probability'' to denote the event happens with probability at least $1-c_{1}n^{-c_{2}}$.


\section{Main Results}

\subsection{Models and assumptions}
\label{sec:models}

We assume the ground truth tensor $\bcX_{\star} = [ \bcX_{\star}(i_1,i_2,i_3)]\in \RR^{n_1\times n_2\times n_3}$ admits the following Tucker decomposition
\begin{align}\label{eq:Tucker_truth}
\bcX_{\star} (i_1,i_2,i_3) = \sum_{j_1=1}^{r_1}\sum_{j_2=1}^{r_2}\sum_{j_3=1}^{r_3} \bU_{\star} (i_1,j_1) \bV_{\star}(i_2,j_2) \bW_{\star} (i_3,j_3) \bcS_{\star}(j_1,j_2,j_3), \quad 1\le i_k \le n_k,
\end{align}
or more compactly,
\begin{align}\label{eq:Tucker_truth_compact}
\bcX_{\star} &=  (\bU_{\star},\bV_{\star},\bW_{\star})\bcdot\bcS_{\star},
\end{align} 
where $\bcS_{\star}=[\bcS_{\star}(j_1,j_2,j_3)]\in \RR^{r_1\times r_2\times r_3}$ is the core tensor of multilinear rank $\br=(r_1,r_2,r_3)$, and $\bU_{\star}=[ \bU_{\star} (i_1,j_1)] \in\RR^{n_1\times r_1}$, $\bV_{\star}=[ \bV_{\star}(i_2,j_2)] \in\RR^{n_2\times r_2}$, $\bW_{\star}=[ \bW_{\star} (i_3,j_3)] \in\RR^{n_3\times r_3}$ are the factor matrices of each mode. Let $\cM_k(\bcX_{\star})$ be the mode-$k$ matricization of $\bcX_{\star}$, we have
\begin{subequations}\label{eq:matricization}
\begin{align}
\cM_1(\bcX_{\star}) &= \bU_{\star}\cM_1(\bcS_{\star})(\bW_{\star}\otimes\bV_{\star})^{\top}, \\ \cM_2(\bcX_{\star}) &= \bV_{\star}\cM_2(\bcS_{\star})(\bW_{\star}\otimes\bU_{\star})^{\top}, \\ \cM_3(\bcX_{\star}) &= \bW_{\star}\cM_3(\bcS_{\star})(\bV_{\star}\otimes\bU_{\star})^{\top}.
\end{align}
\end{subequations}
It is straightforward to see that the Tucker decomposition is not uniquely specified: for any invertible matrices $\bQ_k \in\RR^{r_k\times r_k}$, $k=1,2,3$, one has 
\begin{align*}
(\bU_{\star},\bV_{\star},\bW_{\star})\bcdot\bcS_{\star} =(\bU_{\star}\bQ_1,\bV_{\star}\bQ_2,\bW_{\star}\bQ_3)\bcdot ((\bQ_1^{-1},\bQ_2^{-1},\bQ_3^{-1}) \bcdot \bcS_{\star}).
\end{align*}
We shall fix the ground truth factors such that $\bU_{\star}$, $\bV_{\star}$ and $\bW_{\star}$ are orthonormal matrices consisting of left singular vectors in each mode. Furthermore, the core tensor $\bcS_{\star}$ is related to the singular values in each mode as
\begin{align}\label{eq:ground_truth_condition}
\cM_{k}(\bcS_{\star})\cM_{k}(\bcS_{\star})^{\top} = \bSigma_{\star,k}^2, \qquad k=1,2,3,
\end{align}
where $\bSigma_{\star,k} \coloneqq \diag[\sigma_{1}(\cM_{k}(\bcX_{\star})),\dots,\sigma_{r_k}(\cM_{k}(\bcX_{\star}))]$ is a diagonal matrix where the diagonal elements are composed of the nonzero singular values of $\cM_{k}(\bcX_{\star})$  and $r_k = \rank(\cM_k(\bcX_{\star}))$ for $k=1,2,3$.

\paragraph{Key parameters.}
Of particular interest is a sort of condition number of $\bcX_{\star}$,  which plays an important role in governing the computational efficiency of first-order algorithms. 

\begin{definition}[Condition number]\label{def:kappa} The condition number of $\bcX_{\star}$ is defined as
\begin{align} \label{eq:kappa}
\kappa \coloneqq \frac{\sigma_{\max}(\bcX_{\star})}{\sigma_{\min}(\bcX_{\star})} = \frac{\max_{k=1,2,3}  \sigma_{1}(\cM_k(\bcX_{\star}))}{\min_{k=1,2,3} \sigma_{r_k}(\cM_k(\bcX_{\star}))}.
\end{align}
\end{definition}

Another parameter is the incoherence parameter, which plays an important role in governing the well-posedness of low-rank tensor completion.
\begin{definition}[Incoherence]\label{def:mu} The incoherence parameter of $\bcX_{\star}$ is defined as 
\begin{align}
\mu \coloneqq \max\left\{\frac{n_1}{r_1}\|\bU_{\star}\|_{2,\infty}^2,\, \frac{n_2}{r_2}\|\bV_{\star}\|_{2,\infty}^2, \, \frac{n_3}{r_3}\|\bW_{\star}\|_{2,\infty}^2\right\}.\label{eq:mu}
\end{align}
\end{definition}
Roughly speaking, a small incoherence parameter ensures that the energy of the tensor is evenly distributed across its entries, so that a small random subset of its elements still reveals substantial information about the latent structure of the entire tensor.

\subsection{ScaledGD for tensor completion}
Assume that we have observed a subset of entries in $\bcX_{\star}$, given as 
$\bcY= \cP_{\Omega}(\bcX_{\star})$, where $\cP_{\Omega}:\RR^{n_{1}\times n_{2}\times n_3}\mapsto\RR^{n_{1}\times n_{2}\times n_3}$ is a projection such that
\begin{align}
[\cP_{\Omega}(\bcX_{\star})] (i_1,i_2,i_3) =\begin{cases} \bcX_{\star} (i_1,i_2,i_3), & \mbox{if }(i_1,i_2,i_3)\in\Omega, \\
0, & \mbox{otherwise}.\end{cases}
\end{align}
Here, $\Omega$ is generated according to the Bernoulli observation model in the sense that 
\begin{align} \label{eq:bernoulli_model}
(i_1,i_2,i_3) \in \Omega ~~\mbox{ independently with probability}~p\in (0,1].
\end{align}
The goal of tensor completion is to recover the tensor $\bcX_{\star}$ from its partial observation $\cP_{\Omega}(\bcX_{\star})$, which can be achieved by minimizing the loss function
\begin{align}\label{eq:loss_TC}
\min_{\bF=(\bU,\bV,\bW,\bcS)}\; \cL(\bF)\coloneqq\frac{1}{2p}\left\|\cP_{\Omega}\big((\bU,\bV,\bW)\bcdot\bcS \big) - \bcY \right\|_{\fro}^2. 
\end{align}

\paragraph{Preparation: a scaled projection operator.} 
To guarantee faithful recovery from partial observations, the underlying low-rank tensor $\bcX_{\star}$ needs to be incoherent (cf.~Definition~\ref{def:mu}) to avoid ill-posedness. One typical strategy, frequently employed in the matrix setting, to ensure the incoherence condition is to trim the rows of the factors \cite{chen2015fast} after the gradient update. For ScaledGD, this needs to be done in a careful manner to preserve the equivariance with respect to invertible transforms. Motivated by \cite{tong2021accelerating}, we introduce the scaled projection as follows, 
 \begin{align}\label{eq:scaled_proj}
(\bU, \bV, \bW, \bcS) = \cP_{B}(\bU_{+},{\bV}_{+},{\bW}_{+}, {\bcS}_{+}),\
\end{align}
where $B>0$ is the projection radius, and
\begin{align*}
\bU(i_1,:) & = \left(1 \wedge \frac{B}{\sqrt{n_1}  \| {\bU}_{+}(i_1,:) \breve{\bU}_{+}^{\top}  \|_2}\right)  \bU_+(i_1,:), \qquad 1\le i_1\le n_1; \\
\bV(i_2,:) & = \left(1 \wedge \frac{B}{\sqrt{n_2}\| {\bV}_{+}(i_2,:) \breve{\bV}_{+}^{\top}\|_2}\right) \bV_{+}(i_2,:), \qquad 1\le i_2\le n_2; \\
\bW(i_3,:)  &= \left(1 \wedge \frac{B}{\sqrt{n_3}\| \bW_{+}(i_3,:) \breve{\bW}_{+}^{\top}\|_2}\right) {\bW}_{+}(i_3,:),  \qquad 1\le i_3\le n_3; \\
\bcS & = \bcS_{+}.
\end{align*}
Here, we recall $\breve{\bU}_{+}$, $\breve{\bV}_{+}$, $\breve{\bW}_{+}$ are analogously defined in \eqref{eq:breve_uvw} using $(\bU_{+}, \bV_{+},\bW_{+},\bcS_{+})$.
As can be seen, each row of ${\bU}_{+}$ (resp.~${\bV}_{+}$ and ${\bW}_{+}$) is scaled by a scalar based on the row $\ell_2$ norms of ${\bU}_{+}\breve{\bU}_{+}^{\top}$ (resp.~${\bV}_{+}\breve{\bV}_{+}^{\top}$ and ${\bW}_{+}\breve{\bW}_{+}^{\top}$), which is the mode-1 (resp.~mode-2 and mode-3) matricization of the tensor $(\bU_{+}, \bV_{+}, \bW_{+} )\bcdot\bcS_{+}$. It is a straightforward observation that the projection can be computed efficiently.

\paragraph{Algorithm description.}
With the scaled projection $\cP_B(\cdot)$ defined in hand, we are in a position to describe the details of the proposed ScaledGD algorithm, summarized in Algorithm~\ref{alg:TC}. It consists of two stages: spectral initialization followed by iterative refinements using the scaled projected gradient updates in 
\eqref{eq:iterates_TC}. It is worth emphasizing that all the factors are updated simultaneously, which can be achieved in a parallel manner to accelerate computation run time.

For the spectral initialization, we take advantage of the subspace estimators proposed in \cite{cai2019subspace,xia2021statistically} for highly unbalanced matrices. Specifically, we estimate the subspace spanned by $\bU_{\star}$ by that spanned by top-$r_1$ eigenvectors $\bU_{+}$ of the diagonally-deleted Gram matrix of $p^{-1}\cM_{1}(\bcY)$, denoted as
\begin{align*}
\Poffdiag(p^{-2}\cM_{1}(\bcY)\cM_{1}(\bcY)^{\top}),
\end{align*}
and the other two factors $\bV_{+}$ and $\bW_{+}$ are estimated similarly. The core tensor is then estimated as
\begin{align*}
\bcS_{+} = p^{-1} (\bU_{+}^{\top},\bV_{+}^{\top},\bW_{+}^{\top})\bcdot \bcY,
\end{align*}
which is consistent with its estimation in the HOSVD procedure.
To ensure the initialization is incoherent, we pass it through the scaled projection operator to obtain the final initial estimate:
\begin{align*}
(\bU_{0},\bV_{0},\bW_{0},\bcS_{0})  = \cP_{B} \big(\bU_{+},\bV_{+},\bW_{+},\bcS_{+}\big).
\end{align*}

\begin{algorithm}[t]
\caption{ScaledGD for low-rank tensor completion}\label{alg:TC} 
\begin{algorithmic} 
\STATE \textbf{Input parameters:} step size $\eta$, multilinear rank $\br = (r_1,r_2,r_3)$, probability of observation $p$, projection radius $B$.
\STATE \textbf{Spectral initialization:} Let $\bU_{+}$ be the top-$r_1$ eigenvectors of
$\Poffdiag(p^{-2}\cM_{1}(\bcY)\cM_{1}(\bcY)^{\top})$,
and similarly for $\bV_{+},\bW_{+}$, and $\bcS_{+} = p^{-1} (\bU_{+}^{\top},\bV_{+}^{\top},\bW_{+}^{\top})\bcdot \bcY$. 
Set $(\bU_{0},\bV_{0},\bW_{0},\bcS_{0})  = \cP_{B} \big(\bU_{+},\bV_{+},\bW_{+},\bcS_{+}\big)$.
\STATE \textbf{Scaled projected gradient updates:} for $t=0,1,2,\dots,T-1$ \textbf{do}
\begin{align}
\begin{split} 
{\bU}_{t+} &= \bU_{t}- \frac{\eta}{p}\cM_{1}\left(\cP_{\Omega}\big((\bU_{t},\bV_{t},\bW_{t})\bcdot\bcS_{t} \big) - \bcY \right)\breve{\bU}_t \big(\breve{\bU}_t^{\top} \breve{\bU}_t \big)^{-1}, \\
{\bV}_{t+} &= \bV_{t}- \frac{\eta}{p}\cM_{2}\left(\cP_{\Omega}\big((\bU_{t},\bV_{t},\bW_{t})\bcdot\bcS_{t} \big) - \bcY \right)\breve{\bV}_t\big(\breve{\bV}_t^{\top} \breve{\bV}_t \big)^{-1}, \\
{\bW}_{t+} &= \bW_{t} -  \frac{\eta}{p}\cM_{3}\left(\cP_{\Omega}\big((\bU_{t},\bV_{t},\bW_{t})\bcdot\bcS_{t} \big) - \bcY \right)\breve{\bW}_t\big(\breve{\bW}_t^{\top} \breve{\bW}_t \big)^{-1}, \\
{\bcS}_{t+} &= \bcS_{t} - \frac{\eta}{p}\left((\bU_{t}^{\top}\bU_{t})^{-1}\bU_{t}^{\top},(\bV_{t}^{\top}\bV_{t})^{-1}\bV_{t}^{\top},(\bW_{t}^{\top}\bW_{t})^{-1}\bW_{t}^{\top}\right)\bcdot  \left(\cP_{\Omega}\big((\bU_{t},\bV_{t},\bW_{t})\bcdot\bcS_{t} \big) - \bcY \right),
\end{split}\label{eq:iterates_TC}
\end{align}
where $\breve{\bU}_t$, $\breve{\bV}_t$, and $\breve{\bW}_t$ are defined in \eqref{eq:breve_uvw}. 
Set $(\bU_{t+1},\bV_{t+1},\bW_{t+1},\bcS_{t+1}) = \cP_{B}({\bU}_{t+}, {\bV}_{t+}, {\bW}_{t+}, {\bcS}_{t+})$.
\end{algorithmic} 
\end{algorithm}

\paragraph{Theoretical guarantees.} The following theorem establishes the performance guarantee of ScaledGD for tensor completion, as soon as the sample size is sufficiently large.

\begin{theorem}[ScaledGD for tensor completion]\label{thm:TC} Let $n = \max_{k=1,2,3} n_k$ and $r =\max_{k=1,2,3}r_k$. Suppose that $\bcX_{\star}$ is $\mu$-incoherent, $n_k\gtrsim \epsilon_{0}^{-1}\mu r_k^{3/2}\kappa^2$ for $k=1,2,3$, and that $p$ satisfies 
\begin{align*}
pn_1n_2n_3 \gtrsim \epsilon_{0}^{-1}\sqrt{n_1n_2n_3}\mu^{3/2}r^{5/2}\kappa^{3} \log^{3} n + \epsilon_{0}^{-2}n\mu^{3} r^{4}\kappa^{6}\log^{5} n
\end{align*}
for some small constant $\epsilon_0>0$. Set the projection radius as $B=C_{B}\sqrt{\mu r}\sigma_{\max}(\bcX_{\star})$ for some constant $C_{B}\ge (1+\epsilon_{0})^{3}$. If the step size obeys $0<\eta\le2/5$, then with probability at least $1-c_{1}n^{-c_{2}}$ for universal constants $c_1,c_2>0$, for all $t\ge0$, the iterates of Algorithm~\ref{alg:TC} satisfy
\begin{align*}
\left\|(\bU_{t},\bV_{t},\bW_{t})\bcdot\bcS_{t}-\bcX_{\star}\right\|_{\fro}\le3\epsilon_{0}(1-0.6\eta)^{t}\sigma_{\min}(\bcX_{\star}).
\end{align*} 
\end{theorem}

Theorem~\ref{thm:TC} ensures that ScaledGD finds an $\varepsilon$-accurate estimate, i.e.~$\left\|(\bU_{t},\bV_{t},\bW_{t})\bcdot\bcS_{t}-\bcX_{\star}\right\|_{\fro} \le \varepsilon \sigma_{\min}(\bcX_{\star})$, in at most $O(\log (1/\varepsilon))$ iterations, which is independent of the condition number of $\bcX_{\star}$, as long as the sample complexity is large enough. Assuming that $\mu= O(1)$ and $r \vee \kappa\ll n^{\delta}$ for some small constant $\delta$ to keep only terms with dominating orders of $n$, the sample complexity simplifies to
\begin{align*}
pn_1n_2n_3 \gtrsim n^{3/2}r^{5/2}\kappa^{3}\log^3 n,
\end{align*}
which is near-optimal in view of the conjecture that no polynomial-time algorithm will be successful if the sample complexity is less than the order of $n^{3/2}$ for tensor completion \cite{barak2016noisy}. Compared with existing algorithms collected in Table~\ref{tab:ScaledGD-tensor-completion}, ScaledGD is the {\em first} algorithm that simultaneously achieves a near-optimal sample complexity and a near-linear run time complexity in a provable manner. In particular, while \cite{yuan2016tensor,xia2019polynomial} achieve a sample complexity comparable to ours, the tensor nuclear norm minimization algorithm in \cite{yuan2016tensor} is NP-hard to compute, and the Grassmannian GD in \cite{xia2019polynomial} does not offer an explicit iteration complexity, except that each iteration can be computed in a polynomial time.

\subsection{ScaledGD for tensor regression}
Now we move on to another tensor recovery problem---tensor regression with Gaussian design. Assume that we have access to a set of observations given as 
\begin{align}
y_i = \langle\bcA_{i},\bcX_{\star}\rangle, \quad  i=1,\dots,m, \quad \mbox{ or concisely, } \qquad \by = \cA(\bcX_{\star}),
\end{align}
where $\bcA_i \in \RR^{n_1\times n_2 \times n_3}$ is the $i$-th measurement tensor composed of i.i.d.~Gaussian entries drawn from $\cN(0,1/m)$, and $\cA(\bcX) = \{ \langle\bcA_{i},\bcX\rangle \}_{i=1}^m$ is a linear map from $\RR^{n_1\times n_2 \times n_3}$ to $\RR^m$, whose adjoint operator is given by $\cA^*(\by) = \sum_{i=1}^m y_i \bcA_i$. The goal of tensor regression is to recover $\bcX_{\star}$ from $\by$, by leveraging the low-rank structure of $\bcX_{\star}$. This can be achieved by minimizing the following loss function
\begin{align}\label{eq:loss_TS}
\min_{\bF=(\bU,\bV,\bW,\bcS)}\; \cL(\bF)\coloneqq\frac{1}{2}\left\|\cA((\bU,\bV,\bW)\bcdot\bcS)-\by\right\|_{2}^2.
\end{align}

The proposed ScaledGD  algorithm to minimize \eqref{eq:loss_TS} is described in Algorithm~\ref{alg:TR}, where the algorithm is initialized by applying HOSVD to $\cA^*(\by)$, followed by scaled gradient updates given in \eqref{eq:iterates_TR}. 

\begin{algorithm}[t]
\caption{ScaledGD for low-rank tensor regression}\label{alg:TR} 
\begin{algorithmic} 
\STATE \textbf{Input parameters:} step size $\eta$, multilinear rank $\br = (r_1,r_2,r_3)$. 
\STATE \textbf{Spectral initialization:} Let $(\bU_{0},\bV_{0},\bW_{0}, \bcS_{0}) = \mathrm{HOSVD}_{\br}(\cA^{*}(\by))$ defined in \eqref{eq:HOSVD_factor}.
\STATE \textbf{Scaled gradient updates:} for $t=0,1,2,\dots,T-1$  
\begin{align}
\begin{split} 
\bU_{t+1} &= \bU_{t}-\eta \cM_{1}\left(\cA^{*}(\cA((\bU_{t},\bV_{t},\bW_{t})\bcdot\bcS_{t})-\by)\right)\breve{\bU}_t^{\top} \big(\breve{\bU}_t^{\top} \breve{\bU}_t \big)^{-1}, \\
\bV_{t+1} &= \bV_{t}-\eta \cM_{2}\left(\cA^{*}(\cA((\bU_{t},\bV_{t},\bW_{t})\bcdot\bcS_{t})-\by)\right)\breve{\bV}_t^{\top}\big(\breve{\bV}_t^{\top} \breve{\bV}_t \big)^{-1}, \\
\bW_{t+1} &= \bW_{t} - \eta \cM_{3}\left(\cA^{*}(\cA((\bU_{t},\bV_{t},\bW_{t})\bcdot\bcS_{t})-\by)\right)\breve{\bW}_t^{\top}\big(\breve{\bW}_t^{\top} \breve{\bW}_t \big)^{-1}, \\
\bcS_{t+1} &= \bcS_{t} - \eta\left((\bU_{t}^{\top}\bU_{t})^{-1}\bU_{t}^{\top},(\bV_{t}^{\top}\bV_{t})^{-1}\bV_{t}^{\top},(\bW_{t}^{\top}\bW_{t})^{-1}\bW_{t}^{\top}\right)\bcdot \cA^{*}(\cA((\bU_{t},\bV_{t},\bW_{t})\bcdot\bcS_{t})-\by),
\end{split}\label{eq:iterates_TR}
\end{align}
where $\breve{\bU}_t$, $\breve{\bV}_t$, and $\breve{\bW}_t$ are defined in \eqref{eq:breve_uvw}.
\end{algorithmic} 
\end{algorithm}

\paragraph{Theoretical guarantees.} Encouragingly, we can guarantee that ScaledGD provably recovers the ground truth tensor as long as the sample size is sufficiently large, which is given in the following theorem.

\begin{theorem}[ScaledGD for tensor regression]\label{thm:TR} Let $n = \max_{k=1,2,3} n_k$ and $r =\max_{k=1,2,3}r_k$. For tensor regression with Gaussian design, suppose that $m$ satisfies 
\begin{align*}
m \gtrsim \epsilon_{0}^{-1}\sqrt{n_1n_2n_3}r^{3/2}\kappa^2 + \epsilon_{0}^{-2}(nr^2\kappa^4\log n + r^{4}\kappa^{2})
\end{align*}
for some small constant $\epsilon_0>0$. If the step size obeys $0 < \eta \le 2/5$, then with probability at least $1-c_{1}n^{-c_{2}}$ for universal constants $c_1,c_2>0$, for all $t\ge 0$, the iterates of Algorithm~\ref{alg:TR} satisfy
\begin{align*}
\left\|(\bU_{t},\bV_{t},\bW_{t})\bcdot\bcS_{t}-\bcX_{\star}\right\|_{\fro} & \le 3\epsilon_0 (1-0.6\eta)^{t}\sigma_{\min}(\bcX_{\star}).
\end{align*}
\end{theorem}
Theorem~\ref{thm:TR} ensures that ScaledGD finds an $\varepsilon$-accurate estimate, i.e.~$\left\|(\bU_{t},\bV_{t},\bW_{t})\bcdot\bcS_{t}-\bcX_{\star}\right\|_{\fro} \le \varepsilon \sigma_{\min}(\bcX_{\star})$, in at most $O(\log (1/\varepsilon))$ iterations, which is independent of the condition number of $\bcX_{\star}$, as long as the sample complexity satisfies 
\begin{align*}
m \gtrsim n^{3/2}r^{3/2}\kappa^2,
\end{align*}
where again we keep only terms with dominating orders of $n$.
Compared with the regularized GD \cite{han2020optimal}, ScaledGD achieves a low computation complexity with robustness to ill-conditioning, improving its iteration complexity by a factor of $\kappa^2$, and does not require any explicit regularization.

%

\section{Analysis}\label{sec:analysis}

In this section, we provide some intuitions and sketch the proof of our main theorems. Before continuing, we highlight an important property of ScaledGD: if starting from an equivalent estimate 
\begin{align*}
\widetilde{\bU}_t = \bU_{t}\bQ_{1}, \quad \widetilde{\bV}_t =\bV_{t}\bQ_{2}, \quad \widetilde{\bW}_t = \bW_{t}\bQ_{3}, \quad \widetilde{\bcS}_t = (\bQ_{1}^{-1},\bQ_{2}^{-1},\bQ_{3}^{-1})\bcdot\bcS_{t}
\end{align*}
for some invertible matrices $\bQ_{k}\in\GL(r_k)$ (i.e.~replacing $\bU_t$ by $\bU_t\bQ_1$, and so on), by plugging the above estimate in \eqref{eq:ScaledGD} it is easy to check that the next iterate of ScaledGD is covariant with respect to invertible transforms, meaning
\begin{align*}
\widetilde{\bU}_{t+1} = \bU_{t+1}\bQ_{1}, \quad \widetilde{\bV}_{t+1} =\bV_{t+1}\bQ_{2}, \quad \widetilde{\bW}_{t+1} = \bW_{t+1}\bQ_{3}, \quad \widetilde{\bcS}_{t+1} = (\bQ_{1}^{-1},\bQ_{2}^{-1},\bQ_{3}^{-1})\bcdot\bcS_{t+1}.
\end{align*}
In other words, ScaledGD produces an invariant sequence of low-rank tensor estimates
\begin{align*}
\bcX_t =(\bU_t, \bV_t, \bW_t) \bcdot \bcS_t = (\widetilde{\bU}_t , \widetilde{\bV}_t ,\widetilde{\bW}_t ) \bcdot \widetilde{\bcS}_t
\end{align*}
regardless of the representation of the tensor factors with respect to the underlying symmetry group. This is one of the key reasons behind the insensitivity of ScaledGD to ill-conditioning and factor imbalance.

\paragraph{A key scaled distance metric.} To track the progress of ScaledGD throughout the entire trajectory, one needs a distance metric that properly takes account of the factor ambiguity due to invertible transforms, as well as the effect of scaling. To that end, we define the scaled distance between factor quadruples $\bF=(\bU,\bV,\bW,\bcS)$ and $\bF_{\star}=(\bU_{\star},\bV_{\star},\bW_{\star},\bcS_{\star})$ as
\begin{align}
\dist^2(\bF,\bF_{\star}) \coloneqq \inf_{\bQ_{k}\in\GL(r_k)}\; & \left\|(\bU\bQ_{1}-\bU_{\star})\bSigma_{\star,1}\right\|_{\fro}^{2}+\left\|(\bV\bQ_{2}-\bV_{\star})\bSigma_{\star,2}\right\|_{\fro}^{2}+\left\|(\bW\bQ_{3}-\bW_{\star})\bSigma_{\star,3}\right\|_{\fro}^{2} \nonumber\\
&\qquad \qquad+\left\|(\bQ_{1}^{-1},\bQ_{2}^{-1},\bQ_{3}^{-1})\bcdot\bcS-\bcS_{\star}\right\|_{\fro}^2. \label{eq:dist}
\end{align}
The distance is closely related to the $\ell_2$ distances between the corresponding tensors. In fact, it can be shown that as long as $\bF$ and $\bF_{\star}$ are not too far apart, i.e.~$\dist(\bF,\bF_{\star})\le 0.2\sigma_{\min}(\bcX_{\star})$,
it holds that $\dist(\bF,\bF_{\star})\asymp \|(\bU,\bV,\bW)\bcdot\bcS-\bcX_{\star}\|_{\fro}$ in the sense that (see Appendix~\ref{subsec:scaled_distance} for proofs):
\begin{align*}
\tfrac{1}{3} \left\|(\bU,\bV,\bW)\bcdot\bcS-\bcX_{\star}\right\|_{\fro} \le  \dist(\bF,\bF_{\star}) \le (\sqrt{2}+1)^{3/2}\left\|(\bU,\bV,\bW)\bcdot\bcS-\bcX_{\star}\right\|_{\fro}.
\end{align*}

\subsection{A warm-up case: ScaledGD for tensor factorization}

To shed light on the design insights as well as the proof techniques, we now introduce the ScaledGD algorithm for the tensor factorization problem, which aims to minimize the following loss function:
\begin{align}\label{eq:loss_TF}
\cL(\bF)\coloneqq\frac{1}{2}\|(\bU,\bV,\bW)\bcdot\bcS-\bcX_{\star}\|_{\fro}^{2}  = \frac{1}{2} \| \cM_k \left( (\bU,\bV,\bW)\bcdot\bcS-\bcX_{\star} \right) \|_{\fro}^{2} , \quad k=1,2,3,
\end{align}
where the last equality follows from \eqref{eq:tensor_inner}. Recalling the update rule~\eqref{eq:ScaledGD}, ScaledGD proceeds as
\begin{align}
\begin{split}
\bU_{t+1} &= \bU_{t} - \eta\cM_{1}\left( \bcX_t -\bcX_{\star}\right)\breve{\bU}_t^{\top}\big(\breve{\bU}_t^{\top} \breve{\bU}_t \big)^{-1}, \\
\bV_{t+1} &= \bV_{t} - \eta\cM_{2}\left( \bcX_t -\bcX_{\star}\right)\breve{\bV}_t^{\top}\big(\breve{\bV}_t^{\top} \breve{\bV}_t \big)^{-1}, \\
\bW_{t+1} &= \bW_{t} - \eta\cM_{3}\left(\bcX_t -\bcX_{\star}\right)\breve{\bW}_t^{\top}\big(\breve{\bW}_t^{\top} \breve{\bW}_t \big)^{-1}, \\
\bcS_{t+1} &= \bcS_{t} - \eta\left((\bU_{t}^{\top}\bU_{t})^{-1}\bU_{t}^{\top}, (\bV_{t}^{\top}\bV_{t})^{-1}\bV_{t}^{\top}, (\bW_{t}^{\top}\bW_{t})^{-1}\bW_{t}^{\top}\right)\bcdot\left( \bcX_t -\bcX_{\star}\right),
\end{split}\label{eq:iterates_TF}
\end{align}
where $ \bcX_t =(\bU_t, \bV_t, \bW_t) \bcdot \bcS_t$, with $\breve{\bU}_t$, $\breve{\bV}_t$,  and $\breve{\bW}_t$ defined in \eqref{eq:breve_uvw}.

\paragraph{ScaledGD as a quasi-Newton algorithm.} One way to think of ScaledGD is through the lens of quasi-Newton methods, by equivalently rewriting the ScaledGD update \eqref{eq:iterates_TF} as
\begin{align}
\vc(\bF_{t+1}) = \vc(\bF_{t}) - \eta \bH_t^{-1} \nabla_{\vc(\bF)}\cL(\bF_{t}),\label{eq:iterates_TF_vec}
\end{align}
where $\bH_t \coloneqq \diag \big[ \nabla^{2}_{\vc(\bU),\vc(\bU)}\cL(\bF_{t}) ,\, \nabla^{2}_{\vc(\bV),\vc(\bV)}\cL(\bF_{t}) , \, \nabla^{2}_{\vc(\bW),\vc(\bW)}\cL(\bF_{t}) ,\,  \nabla^{2}_{\vc(\bcS),\vc(\bcS)}\cL(\bF_{t}) \big]$.
To see this, it is straightforward to check that the diagonal blocks of the Hessian of the loss function \eqref{eq:loss_TF} are given precisely as
\begin{align}
\begin{split}
\nabla^2_{\vc(\bU),\vc(\bU)} \cL(\bF_t) &= (\breve{\bU}_t^{\top}\breve{\bU}_t)\otimes\bI_{n_1}, \\
\nabla^2_{\vc(\bV),\vc(\bV)} \cL(\bF_t) &= (\breve{\bV}_t^{\top}\breve{\bV}_t)\otimes\bI_{n_2}, \\
\nabla^2_{\vc(\bW),\vc(\bW)} \cL(\bF_t) &= (\breve{\bW}_t^{\top}\breve{\bW}_t)\otimes\bI_{n_3}, \\
\nabla^2_{\vc(\bcS),\vc(\bcS)} \cL(\bF_t) &= (\bW_t^{\top}\bW_t)\otimes(\bV_t^{\top}\bV_t)\otimes(\bU_t^{\top}\bU_t).
\end{split}\label{eq:Hessian_diags}
\end{align}
Therefore, by vectorization of \eqref{eq:iterates_TF}, ScaledGD can be regarded as a quasi-Newton method where the preconditioner is designed as the inverse of the diagonal approximation of the Hessian.

\paragraph{Guarantees for tensor factorization.} Fortunately, ScaledGD admits a $\kappa$-independent convergence rate for tensor factorization, as long as the initialization is not too far from the ground truth. This is summarized in Theorem~\ref{thm:TF}, whose proof can be found in Appendix~\ref{sec:proof_TF}.

\begin{theorem}\label{thm:TF} For tensor factorization \eqref{eq:loss_TF}, suppose that the initialization satisfies $\dist(\bF_{0},\bF_{\star}) \le \epsilon_0\sigma_{\min}(\bcX_{\star})$ for some small constant $\epsilon_0>0$, then for all $t\ge 0$, the iterates of {ScaledGD}  in \eqref{eq:iterates_TF} satisfy
\begin{align*}
\dist(\bF_{t},\bF_{\star}) \le (1-0.7\eta)^{t}\epsilon_0\sigma_{\min}(\bcX_{\star}), \quad\mbox{and}\quad \left\|(\bU_{t},\bV_{t},\bW_{t})\bcdot\bcS_{t}-\bcX_{\star}\right\|_{\fro}\le3\epsilon_{0}(1-0.7\eta)^{t}\sigma_{\min}(\bcX_{\star}),
\end{align*}
as long as the step size satisfies $0 < \eta \le 2/5$.
\end{theorem}

\paragraph{Intuition of the proof.} Let us provide some intuitions to facilitate understanding by examining a toy case, where all factors become scalars, and the loss function with respect to the factor $\bbf=[u,v,w,s]^{\top}$ becomes
\begin{align*}
\cL(\bbf)=\frac{1}{2}(uvws- u_{\star} v_{\star} w_{\star }s_{\star})^2 =\frac{1}{2}(uvws-s_{\star})^2,
\end{align*}
where $u_{\star}=v_{\star}=w_{\star}=1$, and the ground truth is $\bbf_{\star}= [1,\,1,\,1,\, s_{\star}]^{\top}$. The gradient and the diagonal entries of the Hessian are given respectively as 
\begin{align*}
\nabla \cL(\bbf) &= (uvws-s_{\star})[vws,\, uws,\, uvs, \, uvw]^{\top}, \\
\Pdiag(\nabla^2\cL(\bbf)) &= \diag[(vws)^2,\, (uws)^2, \,(uvs)^2,\, (uvw)^2].
\end{align*}
Moreover, the Hessian matrix at the ground truth is given by 
\begin{align*}
\nabla^2 \cL(\bbf_{\star})= [s_{\star}, \,s_{\star},\, s_{\star}, \,1]^{\top} [s_{\star}, \,s_{\star}, \,s_{\star}, \,1].
\end{align*}
With these in mind, the ScaledGD update rule in \eqref{eq:iterates_TF} and the scaled distance in \eqref{eq:dist} reduce respectively to 
\begin{align*}
\bbf_{t+1} & = \bbf_{t}-\eta\Pdiag^{-1}(\nabla^2\cL(\bbf_{t}))\nabla\cL(\bbf_{t}),  \\
\dist(\bbf,\bbf_{\star})& =\inf_{\bQ=\diag[q_1,q_2,q_3,(q_1q_2q_3)^{-1}]} \left\|\Pdiag^{1/2}(\nabla^2\cL(\bbf_{\star}))(\bQ\bbf-\bbf_{\star})\right\|_2.
\end{align*}
Consequently, we can bound the distance between $\bbf_{t+1}$ and $\bbf_{\star}$ as
\begin{align*}
\dist(\bbf_{t+1}, \bbf_{\star}) &\overset{\mathrm{(i)}}{\le} 
\left\|\Pdiag^{1/2}(\nabla^2\cL(\bbf_{\star}))\left(\bQ_{t}\left(\bbf_{t}-\eta\Pdiag^{-1}(\nabla^2\cL(\bbf_{t}))\nabla\cL(\bbf_{t})\right)-\bbf_{\star}\right)\right\|_2 \\
&\overset{\mathrm{(ii)}}{=} \left\|\Pdiag^{1/2}(\nabla^2\cL(\bbf_{\star}))\left(\bQ_{t}\bbf_{t}-\eta\Pdiag^{-1}(\nabla^2\cL(\bQ_{t}\bbf_{t}))\nabla\cL(\bQ_{t}\bbf_{t})-\bbf_{\star}\right)\right\|_2 \\
&\overset{\mathrm{(iii)}}{\approx} \left\|\Big(\bI-\eta\Pdiag^{-1/2}(\nabla^2\cL(\bbf_{\star}))\nabla^2\cL(\bbf_{\star})\Pdiag^{-1/2}(\nabla^2\cL(\bbf_{\star}))\Big)\Pdiag^{1/2}(\nabla^2\cL(\bbf_{\star}))(\bQ_{t}\bbf_{t}-\bbf_{\star})\right\|_2\\
&\overset{\mathrm{(iv)}}{=} \left\| (\bI-\eta \one\one^{\top} )\Pdiag^{1/2}(\nabla^2\cL(\bbf_{\star}))(\bQ_{t}\bbf_{t}-\bbf_{\star})\right\|_2
\end{align*}
where (i) follows from replacing $\bQ$ by the optimal alignment matrix $\bQ_{t}$ between $\bbf_{t}$ and $\bbf_{\star}$, (ii) follows from the scaling invariance of the iterates, and (iii) holds approximately  as long as $\bQ_{t}\bbf_{t}$ is sufficiently close to $\bbf_{\star}$, which is made precise in the formal proof. The last line (iv) follows from that the scaled Hessian matrix obeys
\begin{align*}
\Pdiag^{-1/2}(\nabla^2\cL(\bbf_{\star}))\nabla^2\cL(\bbf_{\star}) \Pdiag^{-1/2}(\nabla^2\cL(\bbf_{\star}))=\one\one^{\top}.
\end{align*}
By the optimality condition for $\bQ_t$ (see Lemma~\ref{lemma:Q_criterion}), it follows that $\Pdiag^{1/2}(\nabla^2\cL(\bbf_{\star}))(\bQ_{t}\bbf_{t}-\bbf_{\star})$ is approximately parallel to $\one$. Thus, $\dist(\bbf_{t+1}, \bbf_{\star})$ contracts at a constant rate as long as the step size $\eta$ is set as a small constant obeying $0 < \eta \le 2/5$.

\subsection{Proof outline for tensor completion (Theorem~\ref{thm:TC})}
Armed with the insights from the tensor factorization case, we now provide a proof outline of our main theorems on tensor completion and tensor regression, both of which can be viewed as perturbations of tensor factorization with incomplete measurements, combined with properly designed initialization schemes. 
We start with the guarantee for the spectral initialization for tensor completion. 
\begin{lemma}[Initialization for tensor completion]\label{lemma:init_TC} Suppose that $\bcX_{\star}$ is $\mu$-incoherent, $n_k\gtrsim \epsilon_{0}^{-1}\mu r_k^{3/2}\kappa^2$ for $k=1,2,3$, and that $p$ satisfies
\begin{align*}
pn_1n_2n_3 \gtrsim\epsilon_{0}^{-1}\sqrt{n_1n_2n_3}\mu^{3/2}r^{5/2}\kappa^{2}\log^{3} n + \epsilon_{0}^{-2}n\mu^{2} r^{4}\kappa^{4} \log^{5} n
\end{align*}
for some small constant $\epsilon_{0} > 0$. Then with overwhelming probability (i.e.~at least $1-c_1n^{-c_2}$), the spectral initialization before projection $\bF_{+}=(\bU_{+},\bV_{+},\bW_{+},\bS_{+})$ for low-rank tensor completion in Algorithm~\ref{alg:TC} satisfies
\begin{align*}
\dist(\bF_{+},\bF_{\star}) \le \epsilon_{0}\sigma_{\min}(\bcX_{\star}).
\end{align*}
\end{lemma} 

Under a suitable sample size condition, Lemma~\ref{lemma:init_TC} guarantees that $\dist(\bF_{+},\bF_{\star})\le \epsilon_{0} \sigma_{\min}(\bcX_{\star})$ for some small constant $\epsilon_0$. To proceed, we need to know what would happen for the spectral estimate $\bF_0 =\cP_{B} \big(\bF_{+}\big)$ after projection. In fact, the scaled projection is non-expansive w.r.t.~the scaled distance. More importantly, the output is guaranteed to be incoherent. Both properties are stated in the following lemma.

\begin{lemma}[Properties of scaled projection]\label{lemma:scaled_proj} Suppose that $\bcX_{\star}$ is $\mu$-incoherent, and $\dist({\bF}_{+},\bF_{\star})\le \epsilon\sigma_{\min}(\bcX_{\star})$ for some $\epsilon<1$. Set $B = C_B\sqrt{\mu r}\sigma_{\max}(\bcX_{\star})$ for some constant $C_B\ge (1+\epsilon)^{3}$, then $\bF=(\bU,\bV,\bW,\bcS)\coloneqq\cP_{B}({\bF}_{+})$ satisfies the non-expansiveness property
\begin{align*}
\dist(\bF,\bF_{\star})\le\dist({\bF}_{+},\bF_{\star}), 
\end{align*}
and the incoherence condition
\begin{align} \label{eq:TC_cond_2inf}
\sqrt{n_{1}}\|\bU\breve{\bU}^{\top}\|_{2,\infty}\vee \sqrt{n_{2}}\|\bV\breve{\bV}^{\top}\|_{2,\infty} \vee \sqrt{n_{3}}\|\bW\breve{\bW}^{\top}\|_{2,\infty}\le B.
\end{align}
\end{lemma}

Now we are ready to state the following lemma that ensures the linear contraction of the iterative refinements given by the ScaledGD updates.

\begin{lemma}[Local refinements for tensor completion]\label{lemma:contraction_TC} Suppose that $\bcX_{\star}$ is $\mu$-incoherent, and that $p$ satisfies
\begin{align*}
pn_1n_2n_3 \gtrsim \sqrt{n_1n_2n_3}\mu^{3/2}r^{2}\kappa^{3}\log^3 n + n\mu^3 r^4\kappa^{6} \log^5 n.
\end{align*}
Under an event $\cE$ which happens with overwhelming probability, for all $t\geq 0$, if the $t$-th iterate satisfies $\dist(\bF_{t},\bF_{\star}) \le \epsilon\sigma_{\min}(\bcX_{\star})$ for some small constant $\epsilon$, then $\|(\bU_{t},\bV_{t},\bW_{t})\bcdot\bcS_{t}-\bcX_{\star}\|_{\fro} \le 3\dist(\bF_{t},\bF_{\star})$. In addition, if the $t$-th iterate satisfies the incoherence condition 
\begin{align*}
\sqrt{n_1}\|\bU_{t}\breve{\bU}_{t}^{\top}\|_{2,\infty} \vee \sqrt{n_2}\|\bV_{t}\breve{\bV}_{t}^{\top}\|_{2,\infty} \vee \sqrt{n_3}\|\bW_{t}\breve{\bW}_{t}^{\top}\|_{2,\infty} \le B,
\end{align*}
with $B = C_B\sqrt{\mu r}\sigma_{\max}(\bcX_{\star})$ for some constant $C_B\ge (1+\epsilon)^{3}$, then the $(t+1)$-th iterate of Algorithm~\ref{alg:TC} satisfies
\begin{align*}
\dist(\bF_{t+1},\bF_{\star}) \le (1-0.6\eta) \dist(\bF_{t},\bF_{\star}),
\end{align*}
and the incoherence condition
\begin{align*}
\sqrt{n_1}\|\bU_{t+1}\breve{\bU}_{t+1}^{\top}\|_{2,\infty} \vee \sqrt{n_2}\|\bV_{t+1}\breve{\bV}_{t+1}^{\top}\|_{2,\infty} \vee \sqrt{n_3}\|\bW_{t+1}\breve{\bW}_{t+1}^{\top}\|_{2,\infty} \le B.
\end{align*}
\end{lemma}

By combining Lemma~\ref{lemma:init_TC} and Lemma~\ref{lemma:scaled_proj}, we can ensure that the spectral initialization $\bF_{0}=\cP_{B}(\bF_{+})$ satisfies the conditions required in Lemma~\ref{lemma:contraction_TC}, which further enables us to repetitively apply Lemma~\ref{lemma:contraction_TC} to finish the proof of Theorem~\ref{thm:TC}.
The proofs of the above three lemmas are provided in Appendix~\ref{sec:proof_TC}.

\subsection{Proof outline for tensor regression (Theorem~\ref{thm:TR})}

Now we turn to the proof outline for tensor regression (cf.~Theorem~\ref{thm:TR}). To begin with, we show that the local linear convergence of ScaledGD can be guaranteed more generally, as long as the measurement operator $\cA(\cdot)$ satisfies the so-called tensor restricted isometry property (TRIP), which is formally defined as follows.

\begin{definition}[TRIP~\cite{rauhut2017low}]\label{def:TRIP} The linear map $\cA:\RR^{n_1\times n_2\times n_3}\mapsto\RR^{m}$ is said to obey the rank-$\br$ TRIP with $\delta_{\br}\in(0,1)$, if for all tensor $\bcX\in\RR^{n_1\times n_2\times n_3}$ of multilinear rank at most $\br=(r_1,r_2,r_3)$, one has
\begin{align*}
(1-\delta_{\br})\|\bcX\|_{\fro}^2 \le \|\cA(\bcX)\|_{\fro}^2 \le (1+\delta_{\br})\|\bcX\|_{\fro}^2.
\end{align*}
\end{definition}
If $\cA(\cdot)$ satisfies rank-$2\br$ TRIP with $\delta_{2\br}\in (0,1)$, then for any two tensors $\bcX_1, \bcX_2 \in\RR^{n_1\times n_2\times n_3}$ of multilinear rank at most $\br=(r_1,r_2,r_3)$, we have 
\begin{align*}
(1-\delta_{2\br})\|\bcX_1 -\bcX_2 \|_{\fro}^2 \le \|\cA(\bcX_1 -\bcX_2)\|_{\fro}^2 \le (1+\delta_{2\br})\|\bcX_1 -\bcX_2 \|_{\fro}^2.
\end{align*}
In other words, the distance between any pair of rank-$\br$ tensors $\bcX_1$ and $\bcX_2$ is approximately preserved after the linear map $\cA(\cdot)$. The TRIP has been investigated extensively, where \cite[Theorem~2]{rauhut2017low} stated that if $\bcA_i$'s are composed of i.i.d.~sub-Gaussian entries, TRIP holds with high probability provided that $m \gtrsim \delta_{\br}^{-2}  (nr+r^3)$. TRIP also holds for more structured measurement ensembles such as the random Fourier mapping \cite{rauhut2017low}. With the TRIP of $\cA(\cdot)$ in hand, we have the following theorem regarding the local linear convergence of ScaledGD as long as the iterates are close to the ground truth.
\begin{lemma}[Local refinements for tensor regression]\label{lemma:contraction_TR} Suppose that $\cA(\cdot)$ obeys the $2\br$-TRIP with a small constant $\delta_{2\br}\lesssim 1$. If the $t$-th iterate satisfies $\dist(\bF_{t},\bF_{\star}) \le \epsilon \sigma_{\min}(\bcX_{\star})$ for some small constant $\epsilon$, then $\|(\bU_{t},\bV_{t},\bW_{t})\bcdot\bcS_{t}-\bcX_{\star}\|_{\fro} \le 3\dist(\bF_{t},\bF_{\star})$. In addition, if the step size obeys $0 < \eta < 2/5$, then the $(t+1)$-th iterate of Algorithm~\ref{alg:TR} satisfies 
\begin{align*}
\dist(\bF_{t+1},\bF_{\star}) \le (1-0.6\eta) \dist(\bF_{t},\bF_{\star}).
\end{align*}
\end{lemma}

Therefore, ScaledGD converges linearly as long as the sample size $m\gtrsim nr+r^3$ under the Gaussian design, when initialized properly. Unfortunately, obtaining a desired initialization turns out to be a major roadblock and requires a substantially higher sample size, which has been studied extensively for tensor regression \cite{zhang2021low,han2020optimal,zhang2020islet}. Under the Gaussian design, we have the following guarantee for the spectral initialization scheme that invokes HOSVD in Algorithm~\ref{alg:TR}.

\begin{lemma}[Initialization for tensor regression]\label{lemma:init_TR} Suppose that $\{\bcA_{i}\}_{i=1}^{m}$ are composed of i.i.d.~$\cN(0,1/m)$ entries, and that $m$ satisfies 
\begin{align*}
m \gtrsim \epsilon_{0}^{-1}\sqrt{n_1n_2n_3}r^{3/2}\kappa^2 + \epsilon_{0}^{-2}(nr^2\kappa^4\log n + r^{4}\kappa^{2})
\end{align*}
for some small constant $\epsilon_{0} > 0$. Then with overwhelming probability, the spectral initialization for low-rank tensor regression in Algorithm~\ref{alg:TR} satisfies
\begin{align*}
\dist(\bF_{0},\bF_{\star}) \le \epsilon_{0}\sigma_{\min}(\bcX_{\star}). 
\end{align*}
\end{lemma}

Combining Lemma~\ref{lemma:contraction_TR} and Lemma~\ref{lemma:init_TR} finishes the proof of Theorem~\ref{thm:TR}. Their proofs can be found in Appendix~\ref{sec:proof_TR}.


\section{Numerical Experiments}
\label{sec:numerical}

In this section, we provide numerical experiments to corroborate our theoretical findings, with the codes available at 
\begin{center}
\url{https://github.com/Titan-Tong/ScaledGD}.
\end{center} 
The simulations are performed in Matlab with a 3.6 GHz Intel Xeon Gold 6244 CPU.

We illustrate the numerical performance of ScaledGD for tensor completion to corroborate our findings, especially its computational advantage over the regularized GD algorithm \cite{han2020optimal} that is closest to our design. Their algorithm was originally proposed for tensor regression, nevertheless, it naturally applies to tensor completion and exhibits similar results. 
Since the scaled projection does not visibly impact the performance, we implement ScaledGD without performing the projection. Also, we empirically find that the regularization used in \cite{han2020optimal} has no visible benefits, hence we implement GD without the regularization.
For simplicity, we set 
$n_1=n_2=n_3=n$, and $r_1=r_2=r_3=r$. Each entry of the tensor is observed i.i.d.~with probability $p\in (0,1]$. 

\paragraph{Phase transition of ScaledGD.}

We construct the ground truth tensor $\bcX_{\star}=(\bU_{\star},\bV_{\star},\bW_{\star})\bcdot\bcS_{\star}$ by generating $\bU_{\star}$, $\bV_{\star}$ and $\bW_{\star}$ as random orthonormal matrices, and the core tensor $\bcS_{\star}$ composed of i.i.d.~standard Gaussian entries, i.e. $\bcS_{\star}(j_1,j_2,j_3)\sim \cN(0,1)$ for $1\le j_k \le r$, $k=1,2,3$. For each set of parameters, we run $100$ random tests and count the success rate, where the recovery is regarded as successful if the recovered tensor has a relative error $\|\bcX_{T}-\bcX_{\star}\|_{\fro}/\|\bcX_{\star}\|_{\fro}\le 10^{-3}$. Figure~\ref{fig:TC_success} illustrates the success rate with respect to the (scaled) sample size for different tensor sizes $n$, which implies that the recovery is successful when the sample size is moderately large. 

\begin{figure}[ht]
\centering
\includegraphics[width=0.5\textwidth]{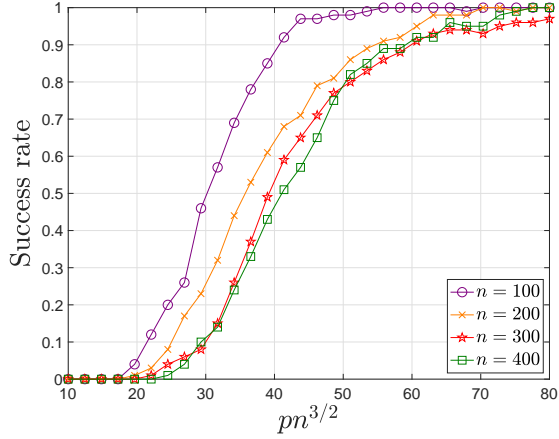} 
\caption{The success rate of ScaledGD with respect to the scaled sample size for tensor completion with $r=5$, when the core tensor is composed of i.i.d.~standard Gaussian entries, for various tensor size $n$. }\label{fig:TC_success}
\end{figure}

\paragraph{Comparison with GD.} 

We next compare the performance of ScaledGD with GD. For a fair comparison, both ScaledGD and GD start from the same spectral initialization, and we use the following update rule of GD as
\begin{align}
\begin{split}
\bU_{t+1} &= \bU_{t} - \eta \sigma^{-2}_{\max}(\bcX_{\star})\nabla_{\bU}\cL(\bF_{t}),  \\
\bV_{t+1} &= \bV_{t} - \eta \sigma^{-2}_{\max}(\bcX_{\star})\nabla_{\bV}\cL(\bF_{t}),  \\
\bW_{t+1} &= \bW_{t} - \eta \sigma^{-2}_{\max}(\bcX_{\star})\nabla_{\bW}\cL(\bF_{t}),  \\
\bcS_{t+1} &= \bcS_{t} - \eta\nabla_{\bcS}\cL(\bF_{t}).
\end{split}\label{eq:iterate_GD}
\end{align}
Throughout the experiments, we used the ground truth value $\sigma_{\max}(\bcX_{\star})$ in running \eqref{eq:iterate_GD}, while in practice, this parameter needs to estimated; to put it differently, the step size of GD is not {\em scale-invariant}, whereas the step size of ScaledGD is.

To ensure the ground truth tensor $\bcX_{\star}=(\bU_{\star},\bV_{\star},\bW_{\star})\bcdot\bcS_{\star}$ has a prescribed condition number $\kappa$, we generate the core tensor $\bcS_{\star}\in\mathbb{R}^{r\times r\times r}$ according to $\bcS_{\star}(j_1,j_2,j_3)=\sigma_{j_1}/\sqrt{r}$ if $j_1+j_2+j_3 \equiv 0 \pmod{r}$ and $0$ otherwise, where $\{\sigma_{j_1}\}_{1\le j_1 \le r}$ take values spaced equally from $1$ to $1/\kappa$. It then follows that $\sigma_{\max}(\bcX_{\star})=1$, $\sigma_{\min}(\bcX_{\star})=1/\kappa$, and the condition number of $\bcX_{\star}$ is exactly $\kappa$.  

\begin{figure}[!ht]
\centering
\includegraphics[width=0.5\textwidth]{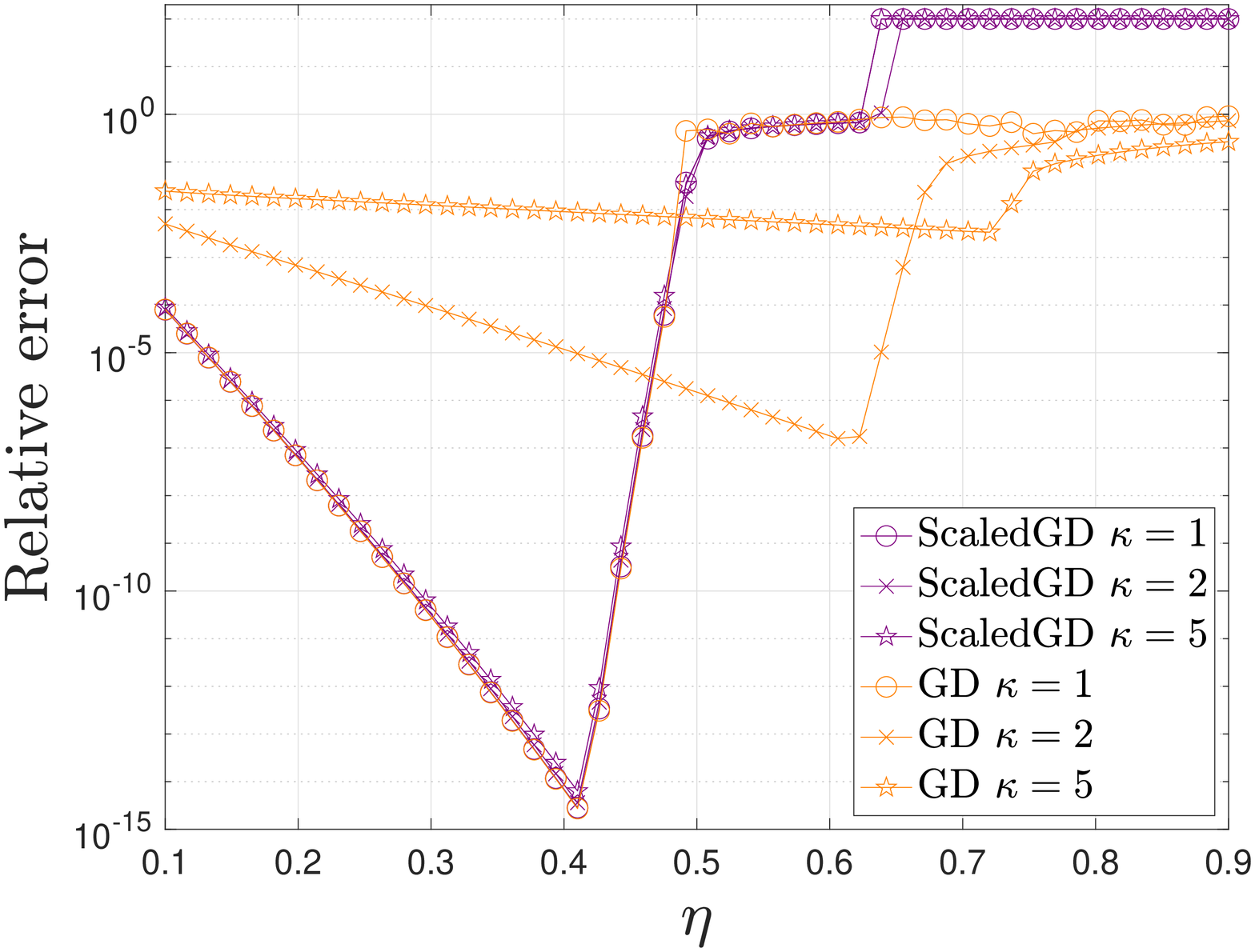} 
\caption{The relative errors of ScaledGD and GD after $80$ iterations with respect to different step sizes $\eta$ from $0.1$ to $0.9$ for tensor completion with $n=100$, $r=5$, $p=0.1$.}\label{fig:TC_stepsizes}
\end{figure}  

Figure~\ref{fig:TC_stepsizes} illustrates the convergence speed of ScaledGD and GD under different step sizes, where we plot the relative error after at most $80$ iterations (the algorithm is terminated if the relative error exceeds $10^2$ following an excessive step size). It can be seen that ScaledGD outperforms GD quite significantly even when the step size of GD is optimized for its performance. Hence, we will fix $\eta = 0.3$ for the rest of the comparisons for both ScaledGD and GD without hurting the conclusions.

\begin{figure}[!ht]
\centering
\begin{tabular}{cc}
 \includegraphics[width=0.5\textwidth]{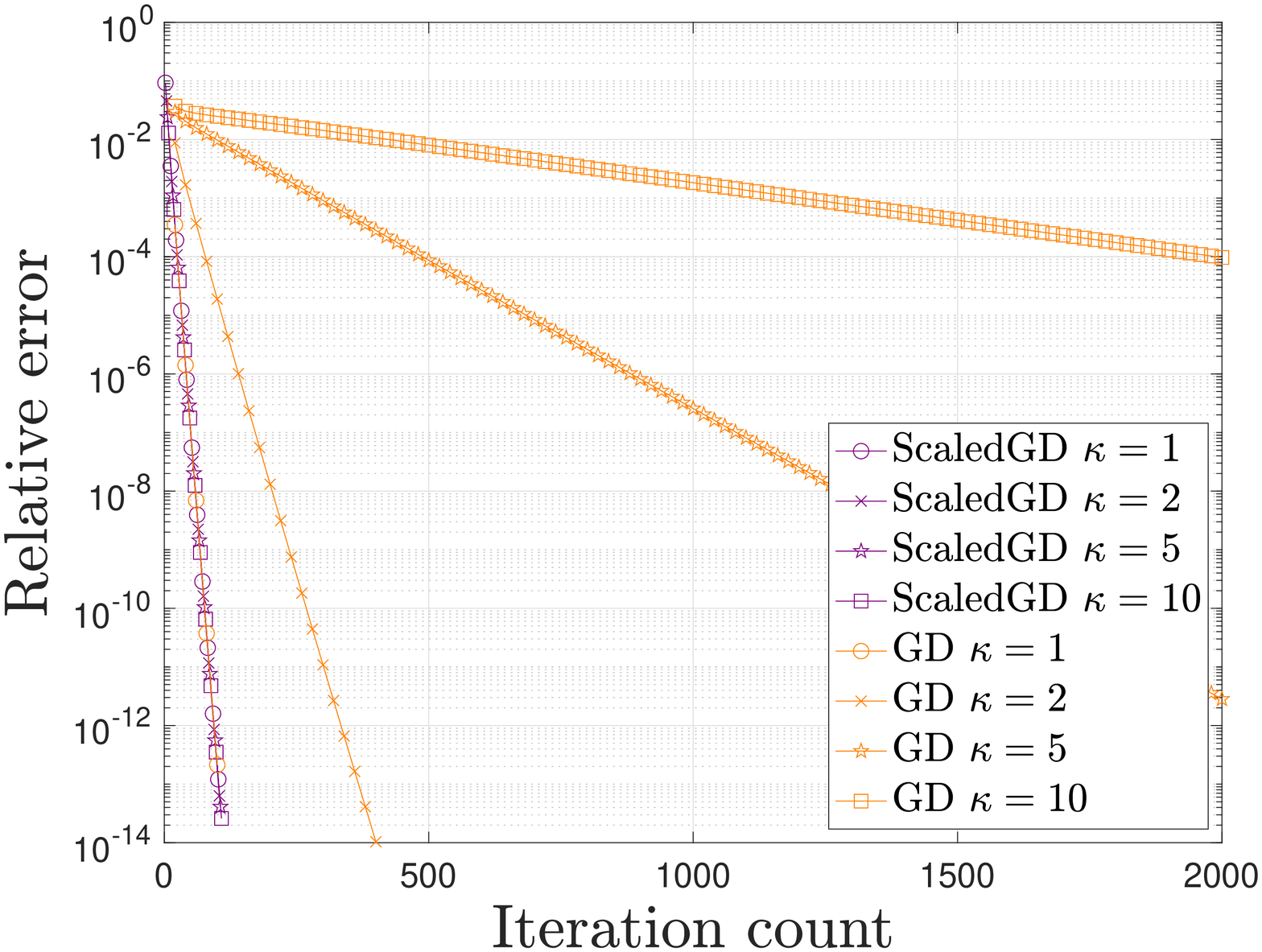} & 
 \includegraphics[width=0.5\textwidth]{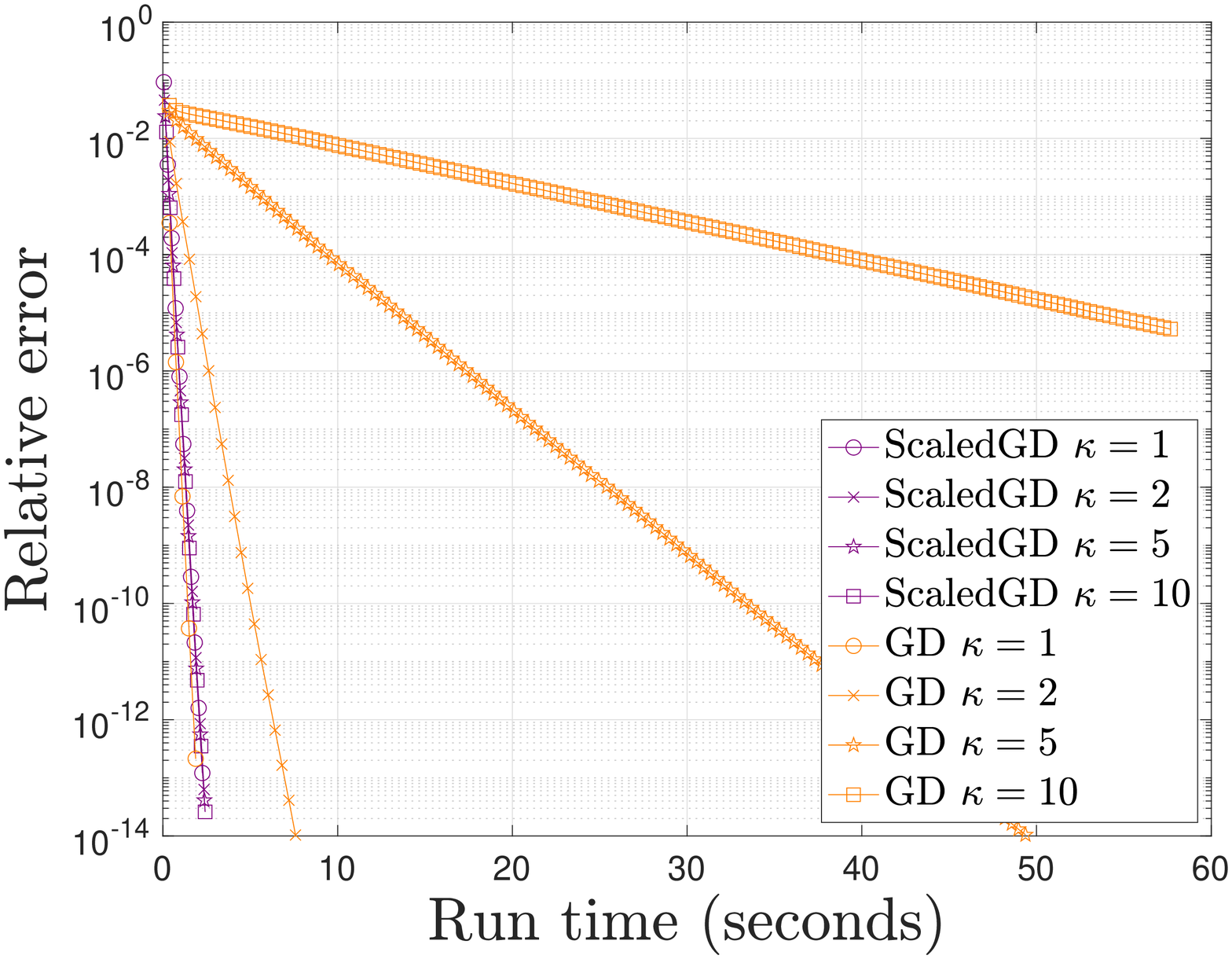} \\
	(a)    & (b) 
\end{tabular}
\caption{The relative errors of ScaledGD and GD with respect to (a) the iteration count and (b) run time (in seconds) under different condition numbers $\kappa=1,2,5,10$ for tensor completion with $n=100$, $r=5$, and $p=0.1$.} \label{fig:TC_time}
\end{figure}

Figure~\ref{fig:TC_time} compares the relative errors of ScaledGD and GD for tensor completion with respect to the iteration count and run time (in seconds) under different condition numbers $\kappa = 1,2,5,10$. This experiment verifies that ScaledGD converges rapidly at a rate independent of the condition number, and matches the fastest rate of GD with perfect conditioning $\kappa=1$. In contrast, the convergence rate of GD deteriorates quickly with the increase of $\kappa$ even at a moderate level. The advantage of ScaledGD carries over to the run time as well, since the scaled gradient only adds a negligible overhead to the gradient computation. 

\begin{figure}[!ht]
\centering
\includegraphics[width=0.5\textwidth]{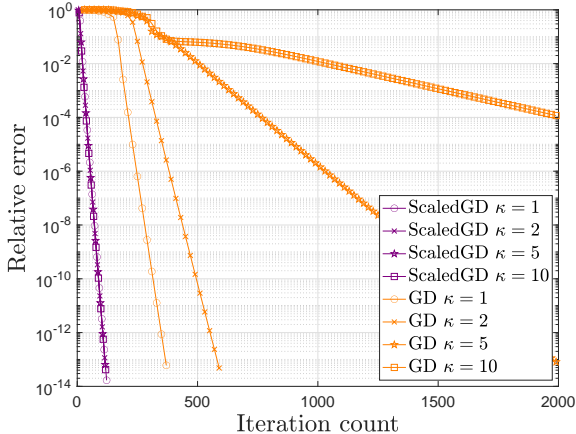} 
\caption{The relative errors of random-initialized ScaledGD and GD with respect to the iteration count under different condition numbers $\kappa=1,2,5,10$  for tensor completion with $n=100$, $r=5$, $p=0.1$.} \label{fig:TC_init} 
\end{figure} 

We next examine the performance of ScaledGD and GD when randomly initialized.
Here, we initialize $\bU_{0},\bV_{0},\bW_{0}$ composed of i.i.d.~entries sampled from $\cN(0, 1/n)$, and $\bcS_{0}$ composed of i.i.d.~entries sampled from $\cN(0, \|\bcY\|_{\fro}^{2}/(pr^3))$. Figure~\ref{fig:TC_init} plots the relative errors of ScaledGD and GD under different condition numbers $\kappa = 1,2,5,10$, using the same random initialization. Surprisingly, while GD gets stuck in a flat region before entering the phase of linear convergence, ScaledGD seems to be quite insensitive to the choice of initialization, and converges almost in the same fashion as the case with spectral initialization.

\begin{figure}[!ht]
\centering
\includegraphics[width=0.5\textwidth]{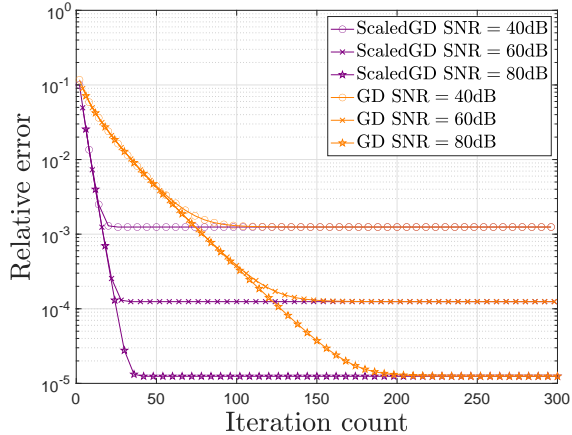} 
\caption{The relative errors of ScaledGD and GD with respect to the iteration count under signal-to-noise ratios $\mathrm{SNR}=40,60,80\mathrm{dB}$ for tensor completion with $n=100$, $r=5$, and $p=0.1$.} \label{fig:TC_noise}
\end{figure}

Finally, we examine the performance of ScaledGD when the observations are corrupted by additive noise, where we assume the noisy observations are given by $\bcY = \cP_{\Omega}(\bcX_{\star} + \bcW)$, with $\bcW(i_1,i_2,i_3)\sim \cN(0, \sigma_w^2)$ composed of i.i.d.~Gaussian entries. Denote the signal-to-noise ratio as 
$\mathrm{SNR}\coloneqq 10\log_{10}\tfrac{\|\bcX_{\star}\|_{\fro}^2}{n^3\sigma_{w}^2}$ in dB.
Figure~\ref{fig:TC_noise} demonstrates the robustness of ScaledGD, by plotting 
the relative errors with respect to the iteration count under $\mathrm{SNR}=40,60,80\mathrm{dB}$. Here, the ground truth tensor $\bcX_{\star}$ is constructed in the same manner as Figure~\ref{fig:TC_success}, where its condition number is approximately $\kappa\approx 2.6$.
It can been seen that ScaledGD reaches the same statistical error as GD, but at a much faster rate. In addition, the convergence speeds are not impacted by the noise levels.


\section{Discussions}

This paper develops a scaled gradient descent algorithm over the factor space for low-rank tensor estimation (i.e.~completion and regression) with provable sample and computational guarantees, leading to a highly scalable approach especially when the ground truth tensor is ill-conditioned and high-dimensional. There are several future directions that are worth exploring, which we briefly discuss below.

\begin{itemize}

\item {\em Preconditioning for other tensor decompositions.} The use of preconditioning will likely also accelerate vanilla gradient descent for low-rank tensor estimation using other decomposition models, such as CP decomposition \cite{cai2019nonconvex}, which is worth investigating.

\item {\em Entrywise error control for tensor completion.} In this paper, we focused on controlling the $\ell_2$ error of the reconstructed tensor in tensor completion, whereas another strong form of statistical guarantees deals with the $\ell_\infty$ error, as done in \cite{ma2017implicit} for matrix completion and in \cite{cai2019nonconvex} for tensor completion with CP decomposition. It is hence of interest to develop similar strong entrywise error guarantees of ScaledGD for tensor completion with Tucker decomposition.

\item {\em Stable and robust low-rank tensor estimation.} In practice, the observations are corrupted by noise and even outliers \cite{li2020non}, therefore, it is necessary to examine the stability and robustness of ScaledGD in more depths; see some initial efforts in \cite{tong2022accelerating} on extending the scaled subgradient algorithm \cite{tong2021low} for robust low-rank tensor regression, and in \cite{dong2022fast} on tensor robust principal component analysis.

\item {\em Random initialization?} As evident from the numerical experiment in Figure~\ref{fig:TC_init}, ScaledGD works remarkably well even from a random initialization, which requires us to go beyond the local geometry and pursue a further understanding of the global landscape of the optimization geometry.

\end{itemize}

\section*{Acknowledgements}
 
The work of T.~Tong and Y.~Chi is supported in part by Office of Naval Research under N00014-18-1-2142 and N00014-19-1-2404, by Army Research Office under W911NF-18-1-0303, by Air Force Research Laboratory under FA8750-20-2-0504, and by National Science Foundation under CAREER ECCS-1818571, CCF-1806154, CCF-1901199 and ECCS-2126634. The U.S. Government is authorized to reproduce and distribute reprints for Governmental purposes notwithstanding any copyright notation thereon.

\bibliographystyle{alphaabbr}
\bibliography{bibfileTensor,bibfileNonconvexScaledGD}

\appendix
\section{Preliminaries}

This section gathers several technical lemmas that will be used later in the proof. More specifically, Section~\ref{subsec:scaled_distance} is devoted to understanding the scaled distance defined in the equation~\eqref{eq:dist}, and in Section~\ref{subsec:perturbation_bounds}, we derive several useful perturbation bounds related to the tensor factors and the tensor itself. All the proofs are collected in the end of each subsection. 

\subsection{Understanding the scaled distance} \label{subsec:scaled_distance}

To begin, recall the scaled distance between $\bF=(\bU,\bV,\bW,\bcS)$ and $\bF_{\star}=(\bU_{\star},\bV_{\star},\bW_{\star},\bcS_{\star})$:
\begin{align}
\dist^{2}(\bF,\bF_{\star})\coloneqq\inf_{\bQ_{k}\in\GL(r_{k})}\; & \left\Vert (\bU\bQ_{1}-\bU_{\star})\bSigma_{\star,1}\right\Vert _{\fro}^{2}+\left\Vert (\bV\bQ_{2}-\bV_{\star})\bSigma_{\star,2}\right\Vert _{\fro}^{2}+\left\Vert (\bW\bQ_{3}-\bW_{\star})\bSigma_{\star,3}\right\Vert _{\fro}^{2}\nonumber \\
 & \qquad\qquad+\left\Vert (\bQ_{1}^{-1},\bQ_{2}^{-1},\bQ_{3}^{-1})\bcdot\bcS-\bcS_{\star}\right\Vert _{\fro}^{2},\label{eq:dist-appendix}
\end{align}
where we call the matrices $\{\bQ_{k}\}_{k=1,2,3}$ (if exist) that attain the infimum the optimal alignment matrices between $\bF$ and $\bF_{\star}$; in particular, $\bF$ and $\bF_{\star}$ are said to be aligned if the optimal alignment matrices are identity matrices. 

In what follows, we provide several useful lemmas whose proof can be found at the end of this subsection. We start with a lemma that ensures the attainability of the infimum in the definition~\eqref{eq:dist-appendix} as long as $\dist(\bF,\bF_{\star})$ is sufficiently small. 

\begin{lemma}\label{lemma:Q_existence} Fix any factor quadruple $\bF=(\bU,\bV,\bW,\bcS)$. Suppose that $\dist(\bF,\bF_{\star})<\sigma_{\min}(\bcX_{\star})$, then the infimum of \eqref{eq:dist-appendix} is attained at some
$\bQ_{k}\in\GL(r_{k})$, i.e.~the alignment matrices between $\bF$ and $\bF_{\star}$ exist. 
\end{lemma}

With the existence of the optimal alignment matrices in place, the following lemma delineates the optimality conditions they need to satisfy.

\begin{lemma}\label{lemma:Q_criterion} The optimal alignment matrices $\{\bQ_{k}\}_{k=1,2,3}$ between $\bF$ and $\bF_{\star}$, if exist, must satisfy 
\begin{align*}
(\bU\bQ_{1})^{\top}(\bU\bQ_{1}-\bU_{\star})\bSigma_{\star,1}^{2}=\cM_{1}\left((\bQ_{1}^{-1},\bQ_{2}^{-1},\bQ_{3}^{-1})\bcdot\bcS-\bcS_{\star}\right)\cM_{1}\left((\bQ_{1}^{-1},\bQ_{2}^{-1},\bQ_{3}^{-1})\bcdot\bcS\right)^{\top},\\
(\bV\bQ_{2})^{\top}(\bV\bQ_{2}-\bV_{\star})\bSigma_{\star,2}^{2}=\cM_{2}\left((\bQ_{1}^{-1},\bQ_{2}^{-1},\bQ_{3}^{-1})\bcdot\bcS-\bcS_{\star}\right)\cM_{2}\left((\bQ_{1}^{-1},\bQ_{2}^{-1},\bQ_{3}^{-1})\bcdot\bcS\right)^{\top},\\
(\bW\bQ_{3})^{\top}(\bW\bQ_{3}-\bW_{\star})\bSigma_{\star,3}^{2}=\cM_{3}\left((\bQ_{1}^{-1},\bQ_{2}^{-1},\bQ_{3}^{-1})\bcdot\bcS-\bcS_{\star}\right)\cM_{3}\left((\bQ_{1}^{-1},\bQ_{2}^{-1},\bQ_{3}^{-1})\bcdot\bcS\right)^{\top}.
\end{align*}
\end{lemma} 

The next lemma relates the scaled distance between the factors to the Euclidean distance between the tensors.

\begin{lemma}\label{lemma:Procrustes} For any factor quadruple $\bF=(\bU,\bV,\bW,\bcS)$, the scaled distance \eqref{eq:dist-appendix} satisfies 
\begin{align*}
\dist(\bF,\bF_{\star})\le(\sqrt{2}+1)^{3/2}\left\Vert (\bU,\bV,\bW)\bcdot\bcS-\bcX_{\star}\right\Vert _{\fro}.
\end{align*}
\end{lemma}

\subsubsection{Proof of Lemma~\ref{lemma:Q_existence}}

This proof mimics that of \cite[Lemma 9]{tong2021accelerating}. The high level idea is to translate the optimization problem \eqref{eq:dist-appendix} into an equivalent continuous optimization problem over a \emph{compact} set. Then an application of the Weierstrass extreme value theorem ensures the existence of the minimizer. 

Under the condition $\dist(\bF,\bF_{\star})<\sigma_{\min}(\bcX_{\star})$, one knows that there exist matrices $\bar{\bQ}_{k}\in\GL(r_{k})$ such that 
\begin{multline*}
\Big(\left\Vert (\bU\bar{\bQ}_{1}-\bU_{\star})\bSigma_{\star,1}\right\Vert _{\fro}^{2}+\left\Vert (\bV\bar{\bQ}_{2}-\bV_{\star})\bSigma_{\star,2}\right\Vert _{\fro}^{2}+\left\Vert (\bW\bar{\bQ}_{3}-\bW_{\star})\bSigma_{\star,3}\right\Vert _{\fro}^{2}\\
+\left\Vert (\bar{\bQ}_{1}^{-1},\bar{\bQ}_{2}^{-1},\bar{\bQ}_{3}^{-1})\bcdot\bcS-\bcS_{\star}\right\Vert _{\fro}^{2}\Big)^{1/2}\le\epsilon\sigma_{\min}(\bcX_{\star}),
\end{multline*}
for some $\epsilon$ obeying $0<\epsilon<1$. The above relation further implies that 
\begin{align*}
\left\Vert \bU\bar{\bQ}_{1}-\bU_{\star}\right\Vert _{\op}\vee\left\Vert \bV\bar{\bQ}_{2}-\bV_{\star}\right\Vert _{\op}\vee\left\Vert \bW\bar{\bQ}_{3}-\bW_{\star}\right\Vert _{\op} \vee \left\Vert (\bar{\bQ}_{3}^{-1}\otimes\bar{\bQ}_{2}^{-1})\cM_{1}(\bcS)^{\top}\bar{\bQ}_{1}^{-\top}\bSigma_{\star,1}^{-1}-\cM_{1}(\bcS_{\star})^{\top}\bSigma_{\star,1}^{-1}\right\Vert _{\op} & \le\epsilon.
\end{align*}
Invoke Weyl's inequality, and use the fact that $\bU_{\star},\bV_{\star},\bW_{\star},\cM_{1}(\bcS_{\star})^{\top}\bSigma_{\star,1}^{-1}$ have orthonormal columns to obtain 
\begin{align}
\sigma_{\min}(\bU\bar{\bQ}_{1})\wedge\sigma_{\min}(\bV\bar{\bQ}_{2})\wedge\sigma_{\min}(\bW\bar{\bQ}_{3})\wedge\sigma_{\min}\left((\bar{\bQ}_{3}^{-1}\otimes\bar{\bQ}_{2}^{-1})\cM_{1}(\bcS)^{\top}\bar{\bQ}_{1}^{-\top}\bSigma_{\star,1}^{-1}\right)\ge1-\epsilon.\label{eq:sigma_min_UQ1}
\end{align}
In addition, it is straightforward to see that the minimization problem on the right hand side of \eqref{eq:dist-appendix} is equivalent to 
\begin{multline}
\inf_{\bH_{k}\in\GL(r_{k})}\;\left\Vert (\bU\bar{\bQ}_{1}\bH_{1}-\bU_{\star})\bSigma_{\star,1}\right\Vert _{\fro}^{2}+\left\Vert (\bV\bar{\bQ}_{2}\bH_{2}-\bV_{\star})\bSigma_{\star,2}\right\Vert _{\fro}^{2}+\left\Vert (\bW\bar{\bQ}_{3}\bH_{3}-\bW_{\star})\bSigma_{\star,3}\right\Vert _{\fro}^{2}\\
+\left\Vert \left(\bH_{1}^{-1}\bar{\bQ}_{1}^{-1},\bH_{2}^{-1}\bar{\bQ}_{2}^{-1},\bH_{3}^{-1}\bar{\bQ}_{3}^{-1}\right)\bcdot\bcS-\bcS_{\star}\right\Vert _{\fro}^{2}.\label{eq:second_inf}
\end{multline}
Therefore, it suffices to establish the infimum is attainable for the above problem instead. By the optimality of $\bar{\bQ}_{k}\bH_{k}$ over $\bar{\bQ}_{k}$, to yield a smaller distance than $\bar{\bQ}_{k}$, $\bH_{k}$ must obey 
\begin{multline*}
\Big(\left\Vert (\bU\bar{\bQ}_{1}\bH_{1}-\bU_{\star})\bSigma_{\star,1}\right\Vert _{\fro}^{2}+\left\Vert (\bV\bar{\bQ}_{2}\bH_{2}-\bV_{\star})\bSigma_{\star,2}\right\Vert _{\fro}^{2}+\left\Vert (\bW\bar{\bQ}_{3}\bH_{3}-\bW_{\star})\bSigma_{\star,3}\right\Vert _{\fro}^{2}\\
+\left\Vert \left(\bH_{1}^{-1}\bar{\bQ}_{1}^{-1},\bH_{2}^{-1}\bar{\bQ}_{2}^{-1},\bH_{3}^{-1}\bar{\bQ}_{3}^{-1}\right)\bcdot\bcS-\bcS_{\star}\right\Vert _{\fro}^{2}\Big)^{1/2}\le\epsilon\sigma_{\min}(\bcX_{\star}).
\end{multline*}
Follow similar reasoning and invoke Weyl's inequality again to obtain 
\begin{multline*}
\sigma_{\max}(\bU\bar{\bQ}_{1}\bH_{1})\vee\sigma_{\max}(\bV\bar{\bQ}_{2}\bH_{2})\vee\sigma_{\max}(\bW\bar{\bQ}_{3}\bH_{3}) \\
\vee \sigma_{\max}\left((\bH_{3}^{-1}\otimes\bH_{2}^{-1})(\bar{\bQ}_{3}^{-1}\otimes\bar{\bQ}_{2}^{-1})\cM_{1}(\bcS)^{\top}\bar{\bQ}_{1}^{-\top}\bH_{1}^{-\top}\bSigma_{\star,1}^{-1}\right) \le 1+\epsilon.
\end{multline*}
Use the relation $\sigma_{\min}(\bA)\sigma_{\max}(\bB)\le\sigma_{\max}(\bA\bB)$, combined with \eqref{eq:sigma_min_UQ1}, to further obtain 
\begin{align*}
\sigma_{\max}(\bH_{k}) & \le\frac{1+\epsilon}{1-\epsilon},\quad k=1,2,3,\\
\sigma_{\max}\left(\bSigma_{\star,1}\bH_{1}^{-\top}\bSigma_{\star,1}^{-1}\right)\sigma_{\max}(\bH_{2}^{-1})\sigma_{\max}(\bH_{3}^{-1})\le\frac{1+\epsilon}{1-\epsilon} & \;\implies\;\sigma_{\min}\left(\bSigma_{\star,1}\bH_{1}\bSigma_{\star,1}^{-1}\right)\sigma_{\min}(\bH_{2})\sigma_{\min}(\bH_{3})\ge\frac{1-\epsilon}{1+\epsilon}.
\end{align*}
As a result, the minimization problem \eqref{eq:second_inf} is equivalent to the constrained problem: 
\begin{align*}
 & \min_{\bH_{k}\in\GL(r_{k})}\;\left\Vert (\bU\bar{\bQ}_{1}\bH_{1}-\bU_{\star})\bSigma_{\star,1}\right\Vert _{\fro}^{2}+\left\Vert (\bV\bar{\bQ}_{2}\bH_{2}-\bV_{\star})\bSigma_{\star,2}\right\Vert _{\fro}^{2}+\left\Vert (\bW\bar{\bQ}_{3}\bH_{3}-\bW_{\star})\bSigma_{\star,3}\right\Vert _{\fro}^{2}\\
 & \qquad\qquad\qquad\qquad\qquad+\left\Vert \left(\bH_{1}^{-1}\bar{\bQ}_{1}^{-1},\bH_{2}^{-1}\bar{\bQ}_{2}^{-1},\bH_{3}^{-1}\bar{\bQ}_{3}^{-1}\right)\bcdot\bcS-\bcS_{\star}\right\Vert _{\fro}^{2}\\
 & \quad\mbox{s.t.}\quad\sigma_{\max}(\bH_{k})\le\frac{1+\epsilon}{1-\epsilon},\quad\sigma_{\min}\left(\bSigma_{\star,1}\bH_{1}\bSigma_{\star,1}^{-1}\right)\sigma_{\min}(\bH_{2})\sigma_{\min}(\bH_{3})\ge\frac{1-\epsilon}{1+\epsilon},\quad k=1,2,3.
\end{align*}
Since this is a continuous optimization problem over a compact set, applying the Weierstrass extreme value theorem finishes the proof.

\subsubsection{Proof of Lemma \ref{lemma:Q_criterion}}

Set the gradient of the expression on the right hand side of \eqref{eq:dist-appendix} with respect to $\bQ_{1}$ as zero to see 
\begin{align*}
\bU^{\top}(\bU\bQ_{1}-\bU_{\star})\bSigma_{\star,1}^{2}-\bQ_{1}^{-\top}\cM_{1}\left((\bQ_{1}^{-1},\bQ_{2}^{-1},\bQ_{3}^{-1})\bcdot\bcS-\bcS_{\star}\right)\cM_{1}\left((\bQ_{1}^{-1},\bQ_{2}^{-1},\bQ_{3}^{-1})\bcdot\bcS\right)^{\top}=\zero.
\end{align*}
We conclude the proof by similarly setting the gradient with respect to $\bQ_{2}$ or $\bQ_{3}$ to zero. 

\subsubsection{Proof of Lemma~\ref{lemma:Procrustes}}

We first state a lemma from \cite[Lemma 11]{tong2021accelerating}, which will be used repeatedly for matricization over different modes.
\begin{lemma}[\cite{tong2021accelerating}]\label{lemma:Procrustes_matrix} Suppose that $\bX_{\star}\in\RR^{n_{1}\times n_{2}}$ has the compact rank-$r$ SVD $\bX_{\star}=\bU_{\star}\bSigma_{\star}\bV_{\star}^{\top}$. For any $\bL\in\RR^{n_{1}\times r}$ and $\bR\in\RR^{n_{2}\times r}$, one has 
\begin{align*}
\inf_{\bQ\in\GL(r)}\left\Vert \bL\bQ\bSigma_{\star}^{1/2}-\bU_{\star}\bSigma_{\star}\right\Vert _{\fro}^{2}+\left\Vert \bR\bQ^{-\top}\bSigma_{\star}^{1/2}-\bV_{\star}\bSigma_{\star}\right\Vert _{\fro}^{2}\le(\sqrt{2}+1)\|\bL\bR^{\top}-\bX_{\star}\|_{\fro}^{2}.
\end{align*}
\end{lemma}

We begin by applying the mode-$1$ matricization (see \eqref{eq:matricization}), and invoking Lemma~\ref{lemma:Procrustes_matrix} with $\bL\coloneqq\bU$, $\bR\coloneqq(\bW\otimes\bV)\cM_{1}(\bcS)^{\top}$, $\bX_{\star}\coloneqq\bU_{\star}\cM_{1}(\bcS_{\star})(\bW_{\star}\otimes\bV_{\star})^{\top}$ 
to arrive at 
\begin{align*}
 & \left\Vert (\bU,\bV,\bW)\bcdot\bcS-\bcX_{\star}\right\Vert _{\fro}^{2}=\left\Vert \bU\cM_{1}(\bcS)(\bW\otimes\bV)^{\top}-\bU_{\star}\cM_{1}(\bcS_{\star})(\bW_{\star}\otimes\bV_{\star})^{\top}\right\Vert _{\fro}^{2}\\
 & \quad\ge(\sqrt{2}-1)\inf_{\bQ\in\GL(r_{1})}\left\Vert \bU\bQ\bSigma_{\star,1}^{1/2}-\bU_{\star}\bSigma_{\star,1}\right\Vert _{\fro}^{2}+\left\Vert (\bW\otimes\bV)\cM_{1}(\bcS)^{\top}\bQ^{-\top}\bSigma_{\star,1}^{1/2}-(\bW_{\star}\otimes\bV_{\star})\cM_{1}(\bcS_{\star})^{\top}\right\Vert _{\fro}^{2}\\
 & \quad=(\sqrt{2}-1)\inf_{\bQ_{1}\in\GL(r_{1})}\left\Vert (\bU\bQ_{1}-\bU_{\star})\bSigma_{\star,1}\right\Vert _{\fro}^{2}+\left\Vert (\bW\otimes\bV)\cM_{1}(\bcS)^{\top}\bQ_{1}^{-\top}-(\bW_{\star}\otimes\bV_{\star})\cM_{1}(\bcS_{\star})^{\top}\right\Vert _{\fro}^{2}\\
 & \quad=(\sqrt{2}-1)\inf_{\bQ_{1}\in\GL(r_{1})}\left\Vert (\bU\bQ_{1}-\bU_{\star})\bSigma_{\star,1}\right\Vert _{\fro}^{2}+\left\Vert (\bQ_{1}^{-1},\bV,\bW)\bcdot\bcS-(\bI_{r_{1}},\bV_{\star},\bW_{\star})\bcdot\bcS_{\star}\right\Vert _{\fro}^{2},
\end{align*}
where we have applied a change-of-variable as $\bQ_{1}=\bQ\bSigma_{\star,1}^{-1/2}$ in the third line, and converted back to the tensor space in the last line. Continue in a similar manner, by applying the mode-$2$ matricization to the second term (see \eqref{eq:matricization}), and invoke Lemma~\ref{lemma:Procrustes_matrix} with $\bL\coloneqq\bV$, $\bR\coloneqq(\bW\otimes\bQ_{1}^{-1})\cM_{2}(\bcS)^{\top}$, $\bX_{\star}\coloneqq\bV_{\star}\cM_{2}(\bcS_{\star})(\bW_{\star}\otimes\bI_{r_{1}})^{\top}$ to arrive at 
\begin{align*}
 & \left\Vert (\bQ_{1}^{-1},\bV,\bW)\bcdot\bcS-(\bI_{r_{1}},\bV_{\star},\bW_{\star})\bcdot\bcS_{\star}\right\Vert _{\fro}^{2}=\left\Vert \bV\cM_{2}(\bcS)(\bW\otimes\bQ_{1}^{-1})^{\top}-\bV_{\star}\cM_{2}(\bcS_{\star})(\bW_{\star}\otimes\bI_{r_{1}})^{\top}\right\Vert _{\fro}^{2}\\
 & \quad\ge(\sqrt{2}-1)\inf_{\bQ\in\GL(r_{2})}\left\Vert \bV\bQ\bSigma_{\star,2}^{1/2}-\bV_{\star}\bSigma_{\star,2}\right\Vert _{\fro}^{2}+\left\Vert (\bW\otimes\bQ_{1}^{-1})\cM_{2}(\bcS)^{\top}\bQ^{-\top}\bSigma_{\star,2}^{1/2}-(\bW_{\star}\otimes\bI_{r_{1}})\cM_{2}(\bcS_{\star})^{\top}\right\Vert _{\fro}^{2}\\
 & \quad=(\sqrt{2}-1)\inf_{\bQ_{2}\in\GL(r_{2})}\left\Vert (\bV\bQ_{2}-\bV_{\star})\bSigma_{\star,2}\right\Vert _{\fro}^{2}+\left\Vert (\bQ_{1}^{-1},\bQ_{2}^{-1},\bW)\bcdot\bcS-(\bI_{r_{1}},\bI_{r_{2}},\bW_{\star})\bcdot\bcS_{\star}\right\Vert _{\fro}^{2}.
\end{align*}
where we have applied a change-of-variable as $\bQ_{2}=\bQ\bSigma_{\star,2}^{-1/2}$ as well as tensorization in the last line. Repeating the same argument by applying the mode-$3$ matricization to the second term, we obtain
\begin{align*}
 & \left\Vert (\bQ_{1}^{-1},\bQ_{2}^{-1},\bW)\bcdot\bcS-(\bI_{r_{1}},\bI_{r_{2}},\bW_{\star})\bcdot\bcS_{\star}\right\Vert _{\fro}^{2}=\left\Vert \bW\cM_{3}(\bcS)(\bQ_{2}^{-1}\otimes\bQ_{1}^{-1})^{\top}-\bW_{\star}\cM_{3}(\bcS_{\star})\right\Vert _{\fro}^{2}\\
 & \quad\ge(\sqrt{2}-1)\inf_{\bQ_{3}\in\GL(r_{3})}\left\Vert (\bW\bQ_{3}-\bW_{\star})\bSigma_{\star,3}\right\Vert _{\fro}^{2}+\left\Vert (\bQ_{1}^{-1},\bQ_{2}^{-1},\bQ_{3}^{-1})\bcdot\bcS-\bcS_{\star}\right\Vert _{\fro}^{2}.
\end{align*}
Finally, combine these results to conclude 
\begin{align*}
\left\Vert (\bU,\bV,\bW)\bcdot\bcS-\bcX_{\star}\right\Vert _{\fro}^{2} & \ge\inf_{\bQ_{k}\in\GL(r_{k})}(\sqrt{2}-1)\left\Vert (\bU\bQ_{1}-\bU_{\star})\bSigma_{\star,1}\right\Vert _{\fro}^{2}+(\sqrt{2}-1)^{2}\left\Vert (\bV\bQ_{2}-\bV_{\star})\bSigma_{\star,2}\right\Vert _{\fro}^{2}\\
 & \qquad+(\sqrt{2}-1)^{3}\left\Vert (\bW\bQ_{3}-\bW_{\star})\bSigma_{\star,3}\right\Vert _{\fro}^{2}+(\sqrt{2}-1)^{3}\left\Vert (\bQ_{1}^{-1},\bQ_{2}^{-1},\bQ_{3}^{-1})\bcdot\bcS-\bcS_{\star}\right\Vert _{\fro}^{2}\\
 & \ge(\sqrt{2}-1)^{3}\dist^{2}(\bF,\bF_{\star}),
\end{align*}
where the last relation uses the definition of $\dist^{2}(\bF,\bF_{\star})$.

\subsection{Several perturbation bounds}\label{subsec:perturbation_bounds}

We now collect several perturbation bounds that will be used repeatedly in the proof. Without loss of generality, assume that $\bF=(\bU,\bV,\bW,\bcS)$ and $\bF_{\star}=(\bU_{\star},\bV_{\star},\bW_{\star},\bcS_{\star})$
are aligned, and introduce the following notation that will be used repeatedly:
\begin{align}
\bDelta_{U} & \coloneqq\bU-\bU_{\star},\qquad\bDelta_{V}\coloneqq\bV-\bV_{\star}, &  & \quad\bDelta_{W}\coloneqq\bW-\bW_{\star}, &  & \bDelta_{\cS}\coloneqq\bcS-\bcS_{\star},\nonumber \\
\breve{\bU} & \coloneqq(\bW\otimes\bV)\cM_{1}(\bcS)^{\top}, & \breve{\bV} & \coloneqq(\bW\otimes\bU)\cM_{2}(\bcS)^{\top}, & \breve{\bW} & \coloneqq(\bV\otimes\bU)\cM_{3}(\bcS)^{\top},\nonumber \\
\breve{\bU}_{\star} & \coloneqq(\bW_{\star}\otimes\bV_{\star})\cM_{1}(\bcS_{\star})^{\top}, & \breve{\bV}_{\star} & \coloneqq(\bW_{\star}\otimes\bU_{\star})\cM_{2}(\bcS_{\star})^{\top}, & \breve{\bW}_{\star} & \coloneqq(\bV_{\star}\otimes\bU_{\star})\cM_{3}(\bcS_{\star})^{\top},\label{eq:short_notations}\\
\bcT_{U} & \coloneqq(\bU_{\star}^{\top}\bDelta_{U},\bI_{r_{2}},\bI_{r_{3}})\bcdot\bcS_{\star}, & \bcT_{V} & \coloneqq(\bI_{r_{1}},\bV_{\star}^{\top}\bDelta_{V},\bI_{r_{3}})\bcdot\bcS_{\star}, & \bcT_{W} & \coloneqq(\bI_{r_{1}},\bI_{r_{2}},\bW_{\star}^{\top}\bDelta_{W})\bcdot\bcS_{\star},\nonumber \\
\bD_{U} & \coloneqq(\bU^{\top}\bU)^{-1/2}\bU^{\top}\bDelta_{U}\bSigma_{\star,1}, & \bD_{V} & \coloneqq(\bV^{\top}\bV)^{-1/2}\bV^{\top}\bDelta_{V}\bSigma_{\star,2}, & \bD_{W} & \coloneqq(\bW^{\top}\bW)^{-1/2}\bW^{\top}\bDelta_{W}\bSigma_{\star,3}.\nonumber 
\end{align}

Now we are ready to state the lemma on perturbation bounds. 
\begin{lemma}\label{lemma:perturb_bounds}
Suppose $\bF=(\bU,\bV,\bW,\bcS)$ and $\bF_{\star}=(\bU_{\star},\bV_{\star},\bW_{\star},\bcS_{\star})$ are aligned and satisfy $\dist(\bF,\bF_{\star})\le\epsilon\sigma_{\min}(\bcX_{\star})$ for some $\epsilon<1$. Then the following bounds hold regarding the spectral norm: \begin{subequations} 
\begin{align}
\|\bDelta_{U}\|_{\op}\vee\|\bDelta_{V}\|_{\op}\vee\|\bDelta_{W}\|_{\op} & \vee\|\cM_{k}(\bDelta_{\cS})^{\top}\bSigma_{\star,k}^{-1}\|_{\op}\le\epsilon,\qquad k=1,2,3;\label{eq:perturb}\\
\|\bU(\bU^{\top}\bU)^{-1}\|_{\op} & \le(1-\epsilon)^{-1};\label{eq:perturb_Uinv}\\
\left\Vert \bU(\bU^{\top}\bU)^{-1}-\bU_{\star}\right\Vert _{\op} & \le\frac{\sqrt{2}\epsilon}{1-\epsilon};\label{eq:perturb_Uinv_d}\\
\left\Vert (\bU^{\top}\bU)^{-1}\right\Vert _{\op} & \le(1-\epsilon)^{-2};\label{eq:perturb_Uinv_2}\\
\left\Vert (\breve{\bU}-\breve{\bU}_{\star})\bSigma_{\star,1}^{-1}\right\Vert _{\op} & \le3\epsilon+3\epsilon^{2}+\epsilon^{3};\label{eq:perturb_R}\\
\left\Vert \breve{\bU}(\breve{\bU}^{\top}\breve{\bU})^{-1}\bSigma_{\star,1}\right\Vert _{\op} & \le(1-\epsilon)^{-3};\label{eq:perturb_Rinv}\\
\left\Vert \breve{\bU}(\breve{\bU}^{\top}\breve{\bU})^{-1}\bSigma_{\star,1}-\breve{\bU}_{\star}\bSigma_{\star,1}^{-1}\right\Vert _{\op} & \le\frac{\sqrt{2}(3\epsilon+3\epsilon^{2}+\epsilon^{3})}{(1-\epsilon)^{3}};\label{eq:perturb_Rinv_d}\\
\left\Vert \bSigma_{\star,1}(\breve{\bU}^{\top}\breve{\bU})^{-1}\bSigma_{\star,1}\right\Vert _{\op} & \le(1-\epsilon)^{-6};\label{eq:perturb_Rinv_2} \\
\left\Vert \bSigma_{\star,1}(\breve{\bU}^{\top}\breve{\bU})^{-1}\cM_{1}(\bcS) \right\Vert _{\op} &\le (1-\epsilon)^{-5}. \label{eq:TC_SRinv}
\end{align}
\end{subequations} By symmetry, a corresponding set of bounds holds for $\bV,\breve{\bV}$ and $\bW,\breve{\bW}$. 

In addition, the following bounds hold regarding the Frobenius norm:
\begin{subequations} 
\begin{align}
\left\Vert (\bU,\bV,\bW)\bcdot\bcS-\bcX_{\star}\right\Vert _{\fro} & \le(1+\frac{3}{2}\epsilon+\epsilon^{2}+\frac{\epsilon^{3}}{4})\left(\|\bDelta_{U}\bSigma_{\star,1}\|_{\fro}+\|\bDelta_{V}\bSigma_{\star,2}\|_{\fro}+\|\bDelta_{W}\bSigma_{\star,3}\|_{\fro}+\|\bDelta_{\cS}\|_{\fro}\right);\label{eq:perturb_T_fro}\\
\left\Vert (\bU,\bV,\bW)\bcdot\bcS_{\star}-\bcX_{\star}\right\Vert _{\fro} & \le(1+\epsilon+\frac{\epsilon^{2}}{3})\left(\|\bDelta_{U}\bSigma_{\star,1}\|_{\fro}+\|\bDelta_{V}\bSigma_{\star,2}\|_{\fro}+\|\bDelta_{W}\bSigma_{\star,3}\|_{\fro}\right);\label{eq:perturb_S_fro}\\
\left\Vert \breve{\bU}-\breve{\bU}_{\star}\right\Vert _{\fro} & \le(1+\epsilon+\frac{\epsilon^{2}}{3})\left(\|\bDelta_{V}\bSigma_{\star,2}\|_{\fro}+\|\bDelta_{W}\bSigma_{\star,3}\|_{\fro}+\|\bDelta_{\cS}\|_{\fro}\right).\label{eq:perturb_R_fro}
\end{align}
\end{subequations} 
As a straightforward consequence of \eqref{eq:perturb_T_fro}, the following important relation holds when $\epsilon \le 0.2$:
\begin{align}
\left\Vert (\bU,\bV,\bW)\bcdot\bcS-\bcX_{\star}\right\Vert _{\fro}\le2(1+\frac{3}{2}\epsilon+\epsilon^{2}+\frac{\epsilon^{3}}{4})\dist(\bF,\bF_{\star})\le3\dist(\bF,\bF_{\star}).\label{eq:tensor2factors}
\end{align}
\end{lemma}
Hence, the scaled distance serves as a metric to gauge the quality of the tensor recovery.

\subsubsection{Proof of Lemma~\ref{lemma:perturb_bounds}}
\paragraph{Proof of spectral norm perturbation bounds.}

To begin, recalling the notation in \eqref{eq:short_notations}, \eqref{eq:perturb} follows directly from the definition 
\begin{align*}
\dist(\bF_{t},\bF_{\star})=\sqrt{\|\bDelta_{U}\bSigma_{\star,1}\|_{\fro}^{2}+\|\bDelta_{V}\bSigma_{\star,2}\|_{\fro}^{2}+\|\bDelta_{W}\bSigma_{\star,3}\|_{\fro}^{2}+\|\bDelta_{\cS}\|_{\fro}^{2}}\le\epsilon\sigma_{\min}(\bcX_{\star})
\end{align*}
together with the relation $\|\bA\bB\|_{\fro}\ge\|\bA\|_{\fro}\sigma_{\min}(\bB)$.

For \eqref{eq:perturb_Uinv}, Weyl's inequality tells $\sigma_{\min}(\bU) \ge \sigma_{\min}(\bU_{\star}) - \|\bDelta_{U}\|_{\op} \ge 1-\epsilon$, and use that 
\begin{align*}
\left\Vert \bU(\bU^{\top}\bU)^{-1}\right\Vert _{\op} = \frac{1}{\sigma_{\min}(\bU)} \le \frac{1}{1-\epsilon}. 
\end{align*}
For \eqref{eq:perturb_Uinv_d}, decompose
\begin{align*}
\bU(\bU^{\top}\bU)^{-1}-\bU_{\star} = -\bU(\bU^{\top}\bU)^{-1}\bDelta_{U}^{\top}\bU_{\star} + \left(\bI_{n_1}-\bU(\bU^{\top}\bU)^{-1}\bU^{\top}\right)\bDelta_{U},
\end{align*}
and use that the two terms are orthogonal to obtain
\begin{align*}
\left\|\bU(\bU^{\top}\bU)^{-1}-\bU_{\star}\right\|_{\op}^{2} &\le \left\|\bU(\bU^{\top}\bU)^{-1}\bDelta_{U}^{\top}\bU_{\star}\right\|_{\op}^{2} + \left\|\left(\bI_{n_1}-\bU(\bU^{\top}\bU)^{-1}\bU^{\top}\right)\bDelta_{U}\right\|_{\op}^{2} \\
&\le \|\bU(\bU^{\top}\bU)^{-1}\|_{\op}^{2}\|\bDelta_{U}\|_{\op}^{2} + \|\bDelta_{U}\|_{\op}^{2} \\
&\le \left((1-\epsilon)^{-2}+1\right)\epsilon^{2}.
\end{align*}
It follows from $\epsilon<1$ that
\begin{align*}
\left\|\bU(\bU^{\top}\bU)^{-1}-\bU_{\star}\right\|_{\op} \le \frac{\sqrt{2}\epsilon}{1-\epsilon}.
\end{align*}

For \eqref{eq:perturb_Uinv_2}, recognizing that 
\begin{align*}
(\bU^{\top}\bU)^{-1}=(\bU(\bU^{\top}\bU)^{-1})^{\top}\bU(\bU^{\top}\bU)^{-1}\qquad\implies\qquad\|(\bU^{\top}\bU)^{-1}\|_{\op}=\|\bU(\bU^{\top}\bU)^{-1}\|_{\op}^{2}\leq\frac{1}{(1-\epsilon)^{2}},
\end{align*}
where the last inequality follows from \eqref{eq:perturb_Uinv}.

For \eqref{eq:perturb_R}, we first expand the expression as 
\begin{align}
\breve{\bU}-\breve{\bU}_{\star} & =(\bW\otimes\bV)\cM_{1}(\bcS)^{\top}-(\bW_{\star}\otimes\bV_{\star})\cM_{1}(\bcS_{\star})^{\top}\nonumber \\
 & =(\bW\otimes\bV-\bW_{\star}\otimes\bV_{\star})\cM_{1}(\bcS_{\star})^{\top}+(\bW\otimes\bV)\cM_{1}(\bcS)^{\top}-(\bW\otimes\bV)\cM_{1}(\bcS_{\star})^{\top}\nonumber \\
 & =(\bW\otimes\bDelta_{V}+\bDelta_{W}\otimes\bV_{\star})\cM_{1}(\bcS_{\star})^{\top}+(\bW\otimes\bV)\cM_{1}(\bDelta_{\cS})^{\top}.\label{eq:decomp_R}
\end{align}
Apply the triangle inequality to obtain 
\begin{align*}
\|(\breve{\bU}-\breve{\bU}_{\star})\bSigma_{\star,1}^{-1}\|_{\op} & \le\left\Vert (\bW\otimes\bDelta_{V}+\bDelta_{W}\otimes\bV_{\star})\cM_{1}(\bcS_{\star})^{\top}\bSigma_{\star,1}^{-1}\right\Vert _{\op}+\left\Vert (\bW\otimes\bV)\cM_{1}(\bDelta_{\cS})^{\top}\bSigma_{\star,1}^{-1}\right\Vert _{\op}\\
 & \le\left(\|\bW\|_{\op}\|\bDelta_{V}\|_{\op}+\|\bDelta_{W}\|_{\op}\|\bV_{\star}\|_{\op}\right)\|\cM_{1}(\bcS_{\star})^{\top}\bSigma_{\star,1}^{-1}\|_{\op}+\|\bW\|_{\op}\|\bV\|_{\op}\|\cM_{1}(\bDelta_{\cS})^{\top}\bSigma_{\star,1}^{-1}\|_{\op}\\
 & \le(1+\epsilon)\epsilon+\epsilon+(1+\epsilon)^{2}\epsilon=3\epsilon+3\epsilon^{2}+\epsilon^{3},
\end{align*}
where we have used $\eqref{eq:perturb}$ and the fact $\|\cM_{1}(\bcS_{\star})^{\top}\bSigma_{\star,1}^{-1}\|_{\op}=1$
(see \eqref{eq:ground_truth_condition}) in the last line.

\eqref{eq:perturb_Rinv} follows from combining 
\begin{align*}
\sigma_{\min}\big(\breve{\bU}\bSigma_{\star,1}^{-1}\big) \ge \sigma_{\min}(\bV)\sigma_{\min}(\bW)\sigma_{\min}\left(\cM_{1}(\bcS)\bSigma_{\star,1}^{-1}\right) &\ge (1-\epsilon)^{3}, \\
\quad\mbox{and}\quad \left\Vert \breve{\bU}(\breve{\bU}^{\top}\breve{\bU})^{-1}\bSigma_{\star,1}\right\Vert _{\op} = \frac{1}{\sigma_{\min}\big(\breve{\bU}\bSigma_{\star,1}^{-1}\big)} &\le \frac{1}{(1-\epsilon)^{3}}.
\end{align*}
With regard to \eqref{eq:perturb_Rinv_d}, repeat the same proof as \eqref{eq:perturb_Uinv_d}, decompose 
\begin{align*}
\breve{\bU}(\breve{\bU}^{\top}\breve{\bU})^{-1}\bSigma_{\star,1}-\breve{\bU}_{\star}\bSigma_{\star,1}^{-1} = -\breve{\bU}(\breve{\bU}^{\top}\breve{\bU})^{-1}(\breve{\bU}-\breve{\bU}_{\star})^{\top}\breve{\bU}_{\star}\bSigma_{\star,1}^{-1} + \left(\bI_{n_2n_3}-\breve{\bU}(\breve{\bU}^{\top}\breve{\bU})^{-1}\breve{\bU}^{\top}\right)(\breve{\bU}-\breve{\bU}_{\star})\bSigma_{\star,1}^{-1},
\end{align*}
and use that the two terms are orthogonal to obtain
\begin{align*}
\left\Vert \breve{\bU}(\breve{\bU}^{\top}\breve{\bU})^{-1}\bSigma_{\star,1}-\breve{\bU}_{\star}\bSigma_{\star,1}^{-1}\right\Vert _{\op}^{2} &\le \left\Vert \breve{\bU}(\breve{\bU}^{\top}\breve{\bU})^{-1}(\breve{\bU}-\breve{\bU}_{\star})^{\top}\breve{\bU}_{\star}\bSigma_{\star,1}^{-1}\right\Vert _{\op}^{2} + \left\Vert \big(\bI_{n_2n_3}-\breve{\bU}(\breve{\bU}^{\top}\breve{\bU})^{-1}\breve{\bU}^{\top}\big)(\breve{\bU}-\breve{\bU}_{\star})\bSigma_{\star,1}^{-1}\right\Vert _{\op}^{2} \\
&\le \|\breve{\bU}(\breve{\bU}^{\top}\breve{\bU})^{-1}\bSigma_{\star,1}\|_{\op}^{2}\|(\breve{\bU}-\breve{\bU}_{\star})\bSigma_{\star,1}^{-1}\|_{\op}^{2} + \|(\breve{\bU}-\breve{\bU}_{\star})\bSigma_{\star,1}^{-1}\|_{\op}^{2} \\
&\le \left((1-\epsilon)^{-6}+1\right)(3\epsilon+3\epsilon^2+\epsilon^3)^2.
\end{align*}
It follows from $\epsilon<1$ that
\begin{align*}
\left\Vert \breve{\bU}(\breve{\bU}^{\top}\breve{\bU})^{-1}\bSigma_{\star,1}-\breve{\bU}_{\star}\bSigma_{\star,1}^{-1}\right\Vert _{\op} \le \frac{\sqrt{2}(3\epsilon+3\epsilon^2+\epsilon^3)}{(1-\epsilon)^{3}}.
\end{align*}
The relation \eqref{eq:perturb_Rinv_2} follows from \eqref{eq:perturb_Rinv} and the relation: 
\begin{align*}
\left\Vert \bSigma_{\star,1}(\breve{\bU}^{\top}\breve{\bU})^{-1}\bSigma_{\star,1}\right\Vert _{\op}=\left\Vert \bSigma_{\star,1}(\breve{\bU}^{\top}\breve{\bU})^{-1}\breve{\bU}^{\top}\breve{\bU}(\breve{\bU}^{\top}\breve{\bU})^{-1}\bSigma_{\star,1}\right\Vert _{\op}=\left\Vert \breve{\bU}(\breve{\bU}^{\top}\breve{\bU})^{-1}\bSigma_{\star,1}\right\Vert _{\op}^{2}.
\end{align*}
With regard to \eqref{eq:TC_SRinv}, we have  
\begin{align*}
\left\Vert \bSigma_{\star,1}(\breve{\bU}^{\top}\breve{\bU})^{-1}\cM_{1}(\bcS) \right\Vert _{\op} &= \left\Vert \bSigma_{\star,1}(\breve{\bU}^{\top}\breve{\bU})^{-1}\breve{\bU}^{\top}\left(\bW(\bW^{\top}\bW)^{-1}\otimes\bV(\bV^{\top}\bV)^{-1}\right) \right\Vert _{\op} \\
&\le \left\Vert \breve{\bU}(\breve{\bU}^{\top}\breve{\bU})^{-1}\bSigma_{\star,1} \right\Vert _{\op} \left\Vert \bW(\bW^{\top}\bW)^{-1}\right\Vert _{\op} \left\Vert \bV(\bV^{\top}\bV)^{-1}\right\Vert _{\op} \\
&\le (1-\epsilon)^{-5},
\end{align*}
where the first line follows from 
\begin{align}
\breve{\bU}^{\top}=\cM_{1}(\bS)(\bW\otimes\bV)^{\top} \qquad\implies\qquad \cM_{1}(\bcS)=\breve{\bU}^{\top}\left(\bW(\bW^{\top}\bW)^{-1}\otimes\bV(\bV^{\top}\bV)^{-1}\right),\label{eq:breve_U_top}
\end{align} 
and the last inequality uses \eqref{eq:perturb_Uinv_d} and \eqref{eq:perturb_Rinv}.

\paragraph{Proof of Frobenius norm perturbation bounds.}

We proceed to prove the perturbation bounds regarding the Frobenius norm. For \eqref{eq:perturb_T_fro}, we begin with the following decomposition
\begin{align}
(\bU,\bV,\bW)\bcdot\bcS-\bcX_{\star} & =(\bU,\bV,\bW)\bcdot\bcS-(\bU_{\star},\bV_{\star},\bW_{\star})\bcdot\bcS_{\star}\nonumber \\
 & =(\bU,\bV,\bW)\bcdot\bDelta_{\cS}+(\bDelta_{U},\bV,\bW)\bcdot\bcS_{\star}+(\bU_{\star},\bDelta_{V},\bW)\bcdot\bcS_{\star}+(\bU_{\star},\bV_{\star},\bDelta_{W})\bcdot\bcS_{\star}.\label{eq:decomp_T}
\end{align}
Apply the triangle inequality, together with the invariance of the
Frobenius norm to matricization, to obtain 
\begin{align*}
\left\Vert (\bU,\bV,\bW)\bcdot\bcS-\bcX_{\star}\right\Vert _{\fro} & \le\left\Vert (\bU,\bV,\bW)\bcdot\bDelta_{\cS}\right\Vert _{\fro}+\left\Vert \bDelta_{U}\cM_{1}(\bcS_{\star})(\bW\otimes\bV)^{\top}\right\Vert _{\fro}\\
 & \qquad\qquad+\left\Vert \bDelta_{V}\cM_{2}(\bcS_{\star})(\bW\otimes\bU_{\star})^{\top}\right\Vert _{\fro}+\left\Vert \bDelta_{W}\cM_{3}(\bcS_{\star})(\bV_{\star}\otimes\bU_{\star})^{\top}\right\Vert _{\fro}\\
 & \le\|\bU\|_{\op}\|\bV\|_{\op}\|\bW\|_{\op}\|\bDelta_{\cS}\|_{\fro}+\|\bDelta_{U}\cM_{1}(\bcS_{\star})\|_{\fro}\|\bW\|_{\op}\|\bV\|_{\op}\\
 & \qquad\qquad+\|\bDelta_{V}\cM_{2}(\bcS_{\star})\|_{\fro}\|\bW\|_{\op}\|\bU_{\star}\|_{\op}+\|\bDelta_{W}\cM_{3}(\bcS_{\star})\|_{\fro}\|\bV_{\star}\|_{\op}\|\bU_{\star}\|_{\op}\\
 & \le(1+\epsilon)^{3}\|\bDelta_{\cS}\|_{\fro}+(1+\epsilon)^{2}\|\bDelta_{U}\bSigma_{\star,1}\|_{\fro}+(1+\epsilon)\|\bDelta_{V}\bSigma_{\star,2}\|_{\fro}+\|\bDelta_{W}\bSigma_{\star,3}\|_{\fro},
\end{align*}
where the second inequality follows from \eqref{eq:tensor_properties_e}, and the last inequality follows from \eqref{eq:ground_truth_condition} and \eqref{eq:perturb}. By symmetry, one can permute the occurrence of $\bDelta_{U},\bDelta_{V},\bDelta_{W},\bDelta_{\cS}$ in the decomposition \eqref{eq:decomp_T}. For example, invoking another viable decomposition
of $(\bU,\bV,\bW)\bcdot\bcS-\bcX_{\star}$ as
\begin{align*}
(\bU,\bV,\bW)\bcdot\bcS-\bcX_{\star} & =(\bU,\bDelta_{V},\bW)\bcdot\bcS+(\bU,\bV_{\star},\bDelta_{W})\bcdot\bcS+(\bU,\bV_{\star},\bW_{\star})\bcdot\bDelta_{\cS}+(\bDelta_{U},\bV_{\star},\bW_{\star})\bcdot\bcS_{\star}
\end{align*}
leads to the perturbation bound 
\begin{align*}
\left\Vert (\bU,\bV,\bW)\bcdot\bcS-\bcX_{\star}\right\Vert _{\fro}\le(1+\epsilon)^{3}\|\bDelta_{V}\bSigma_{\star,2}\|_{\fro}+(1+\epsilon)^{2}\|\bDelta_{W}\bSigma_{\star,3}\|_{\fro}+(1+\epsilon)\|\bDelta_{\cS}\|_{\fro}+\|\bDelta_{U}\bSigma_{\star,1}\|_{\fro}.
\end{align*}
To complete the proof of \eqref{eq:perturb_T_fro}, we take an average of all viable bounds from $4!=24$ permutations to balance their coefficients as 
\begin{align*}
\frac{1}{4}\left((1+\epsilon)^{3}+(1+\epsilon)^{2}+(1+\epsilon)+1\right)=1+\frac{3}{2}\epsilon+\epsilon^{2}+\frac{1}{4}\epsilon^{3},
\end{align*}
thus we obtain 
\begin{align*}
\left\Vert (\bU,\bV,\bW)\bcdot\bcS-\bcX_{\star}\right\Vert _{\fro}\le(1+\frac{3}{2}\epsilon+\epsilon^{2}+\frac{1}{4}\epsilon^{3})\big(\|\bDelta_{U}\bSigma_{\star,1}\|_{\fro}+\|\bDelta_{V}\bSigma_{\star,2}\|_{\fro}+\|\bDelta_{W}\bSigma_{\star,3}\|_{\fro}+\|\bDelta_{\cS}\|_{\fro}\big).
\end{align*}
The relation \eqref{eq:perturb_S_fro} can be proved in a similar fashion; for the sake of brevity, we omit its proof.

Turning to \eqref{eq:perturb_R_fro}, apply the triangle inequality to \eqref{eq:decomp_R} to obtain 
\begin{align}
\|\breve{\bU}-\breve{\bU}_{\star}\|_{\fro}\le\left\Vert (\bW\otimes\bDelta_{V})\cM_{1}(\bcS_{\star})^{\top}\right\Vert _{\fro}+\left\Vert (\bDelta_{W}\otimes\bV_{\star})\cM_{1}(\bcS_{\star})^{\top}\right\Vert _{\fro}+\left\Vert (\bW\otimes\bV)\cM_{1}(\bDelta_{\cS})\right\Vert _{\fro}.\label{eq:breveU_frob_error}
\end{align}
To bound the first term, change the mode of matricization (see \eqref{eq:matricization}) to arrive at 
\begin{align*}
\left\Vert (\bW\otimes\bDelta_{V})\cM_{1}(\bcS_{\star})^{\top}\right\Vert _{\fro} & =\left\Vert (\bI_{r_{1}},\bDelta_{V},\bW)\bcdot\bcS_{\star}\right\Vert _{\fro}=\left\Vert \bDelta_{V}\cM_{2}(\bcS_{\star})(\bW\otimes\bI_{r_{1}})^{\top}\right\Vert _{\fro}\\
 & \le\|\bDelta_{V}\cM_{2}(\bcS_{\star})\|_{\fro}\|\bW\|_{\op}\le(1+\epsilon)\|\bDelta_{V}\cM_{2}(\bcS_{\star})\|_{\fro},
\end{align*}
where the last inequality uses \eqref{eq:perturb}. Similarly, the last two terms in \eqref{eq:breveU_frob_error} can be bounded as
\begin{align*}
\left\Vert (\bDelta_{W}\otimes\bV_{\star})\cM_{1}(\bcS_{\star})^{\top}\right\Vert _{\fro} & \le\|\bDelta_{W}\cM_{3}(\bcS_{\star})\|_{\fro},\quad\mbox{and}\quad\left\Vert (\bW\otimes\bV)\cM_{1}(\bDelta_{\cS})\right\Vert _{\fro}\le(1+\epsilon)^{2}\|\bDelta_{\cS}\|_{\fro}.
\end{align*}
Plugging the above bounds back to \eqref{eq:breveU_frob_error}, we have 
\begin{align*}
\|\breve{\bU}-\breve{\bU}_{\star}\|_{\fro}\le(1+\epsilon)\|\bDelta_{V}\cM_{2}(\bcS_{\star})\|_{\fro}+\|\bDelta_{W}\cM_{3}(\bcS_{\star})\|_{\fro}+(1+\epsilon)^{2}\|\bDelta_{\cS}\|_{\fro}.
\end{align*}
Using a similar symmetrization trick as earlier, by permuting the occurrences of $\bDelta_{V},\bDelta_{W},\bDelta_{\cS}$ in the decomposition \eqref{eq:decomp_R}, we arrive at the final advertised bound \eqref{eq:perturb_R_fro}.

\section{Proof for Tensor Factorization (Theorem~\ref{thm:TF})}

\label{sec:proof_TF}

We prove Theorem~\ref{thm:TF} via induction. Suppose that for some
$t\ge0$, one has $\dist(\bF_{t},\bF_{\star})\le\epsilon\sigma_{\min}(\bcX_{\star})$
for some sufficiently small $\epsilon$ whose size will be specified
later in the proof. 
Our goal is to bound the scaled distance from
the ground truth to the next iterate, i.e.~$\dist(\bF_{t+1},\bF_{\star})$.

Since $\dist(\bF_{t},\bF_{\star})\le\epsilon\sigma_{\min}(\bcX_{\star})$,
Lemma~\ref{lemma:Q_existence} ensures that the optimal alignment
matrices $\{\bQ_{t,k}\}_{k=1,2,3}$ between $\bF_{t}$ and $\bF_{\star}$
exist. Therefore, in view of the definition of $\dist(\bF_{t+1},\bF_{\star})$,
one has 
\begin{align}
\dist^{2}(\bF_{t+1},\bF_{\star}) & \le\left\Vert (\bU_{t+1}\bQ_{t,1}-\bU_{\star})\bSigma_{\star,1}\right\Vert _{\fro}^{2}+\left\Vert (\bV_{t+1}\bQ_{t,2}-\bV_{\star})\bSigma_{\star,2}\right\Vert _{\fro}^{2}+\left\Vert (\bW_{t+1}\bQ_{t,3}-\bW_{\star})\bSigma_{\star,3}\right\Vert _{\fro}^{2}\nonumber \\
 & \quad+\left\Vert (\bQ_{t,1}^{-1},\bQ_{t,2}^{-1},\bQ_{t,3}^{-1})\bcdot\bcS_{t+1}-\bcS_{\star}\right\Vert _{\fro}^{2}.\label{eq:TF_expand}
\end{align}
To avoid notational clutter, we denote $\bF\coloneqq(\bU,\bV,\bW,\bcS)$
with 
\begin{align} \label{eq:additional_notation_aligned}
\bU & \coloneqq\bU_{t}\bQ_{t,1},\qquad\bV\coloneqq\bV_{t}\bQ_{t,2},\qquad\bW\coloneqq\bW_{t}\bQ_{t,3},\qquad\bcS\coloneqq(\bQ_{t,1}^{-1},\bQ_{t,2}^{-1},\bQ_{t,3}^{-1})\bcdot\bcS_{t},
\end{align}
and adopt the set of notation defined in \eqref{eq:short_notations}
for the rest of the proof. Clearly, $\bm{F}$ is aligned with $\bm{F}_{\star}$. With these notation, we can rephrase the consequences of Lemma~\ref{lemma:Q_criterion} as:
\begin{align}
\begin{split}
\bU^{\top}\bDelta_{U}\bSigma_{\star,1}^{2} &=\cM_{1}(\bDelta_{\cS})\cM_{1}(\bcS)^{\top},\\ \bV^{\top}\bDelta_{V}\bSigma_{\star,2}^{2} & =\cM_{2}(\bDelta_{\cS})\cM_{2}(\bcS)^{\top}, \\
 \bW^{\top}\bDelta_{W}\bSigma_{\star,3}^{2}& =\cM_{3}(\bDelta_{\cS})\cM_{3}(\bcS)^{\top}.
\end{split}\label{eq:alignment}
\end{align}

We aim to establish the following bounds for the four terms in \eqref{eq:TF_expand} as long as $\eta<1$: \begin{subequations}\label{eq:final-bounds}
\begin{align}
\left\Vert (\bU_{t+1}\bQ_{t,1}-\bU_{\star})\bSigma_{\star,1}\right\Vert _{\fro}^{2} & \le(1-\eta)^{2}\|\bDelta_{U}\bSigma_{\star,1}\|_{\fro}^{2}\nonumber \\
 & \qquad-2\eta(1-\eta)\left\langle \bcT_{U},\bcT_{U}+\bcT_{V}+\bcT_{W}\right\rangle +\eta^{2}\left\Vert \bcT_{U}+\bcT_{V}+\bcT_{W}\right\Vert _{\fro}^{2}\nonumber \\
 & \qquad+2\eta(1-\eta)C_{1}\epsilon\dist^{2}(\bF_{t},\bF_{\star})+\eta^{2}C_{2}\epsilon\dist^{2}(\bF_{t},\bF_{\star});\label{eq:U-bound}\\
\left\Vert (\bV_{t+1}\bQ_{t,2}-\bV_{\star})\bSigma_{\star,2}\right\Vert _{\fro}^{2} & \le(1-\eta)^{2}\|\bDelta_{V}\bSigma_{\star,2}\|_{\fro}^{2}\nonumber \\
 & \qquad-2\eta(1-\eta)\left\langle \bcT_{V},\bcT_{U}+\bcT_{V}+\bcT_{W}\right\rangle +\eta^{2}\left\Vert \bcT_{U}+\bcT_{V}+\bcT_{W}\right\Vert _{\fro}^{2}\nonumber \\
 & \qquad+2\eta(1-\eta)C_{1}\epsilon\dist^{2}(\bF_{t},\bF_{\star})+\eta^{2}C_{2}\epsilon\dist^{2}(\bF_{t},\bF_{\star});\label{eq:V-bound}\\
\left\Vert (\bW_{t+1}\bQ_{t,3}-\bW_{\star})\bSigma_{\star,3}\right\Vert _{\fro}^{2} & \le(1-\eta)^{2}\|\bDelta_{W}\bSigma_{\star,3}\|_{\fro}^{2}\nonumber \\
 & \qquad-2\eta(1-\eta)\left\langle \bcT_{W},\bcT_{U}+\bcT_{V}+\bcT_{W}\right\rangle +\eta^{2}\left\Vert \bcT_{U}+\bcT_{V}+\bcT_{W}\right\Vert _{\fro}^{2}\nonumber \\
 & \qquad+2\eta(1-\eta)C_{1}\epsilon\dist^{2}(\bF_{t},\bF_{\star})+\eta^{2}C_{2}\epsilon\dist^{2}(\bF_{t},\bF_{\star});\label{eq:W-bound}\\
\left\Vert (\bQ_{t,1}^{-1},\bQ_{t,2}^{-1},\bQ{}_{t,3}^{-1})\bcdot\bcS_{t+1}-\bcS_{\star}\right\Vert _{\fro}^{2} & \le(1-\eta)^{2}\|\bDelta_{\cS}\|_{\fro}^{2}-\eta(2-5\eta)\left(\left\Vert \bD_{U}\right\Vert _{\fro}^{2}+\left\Vert \bD_{V}\right\Vert _{\fro}^{2}+\left\Vert \bD_{W}\right\Vert _{\fro}^{2}\right)\nonumber \\
 & \qquad+2\eta(1-\eta)C_{1}\epsilon\dist^{2}(\bF_{t},\bF_{\star})+\eta^{2}C_{2}\epsilon\dist^{2}(\bF_{t},\bF_{\star}),\label{eq:S-bound}
\end{align}
\end{subequations} where $C_{1},C_{2}>1$ are two universal constants.
Suppose for the moment that the four bounds~\eqref{eq:final-bounds}
hold. We can then combine them all to deduce 
\begin{align}
\dist^{2}(\bF_{t+1},\bF_{\star}) & \le(1-\eta)^{2}\Big(\left\Vert \bDelta_{U}\bSigma_{\star,1}\right\Vert _{\fro}^{2}+\left\Vert \bDelta_{V}\bSigma_{\star,2}\right\Vert _{\fro}^{2}+\left\Vert \bDelta_{W}\bSigma_{\star,3}\right\Vert _{\fro}^{2}+\|\bDelta_{\cS}\|_{\fro}^{2}\Big)\nonumber \\
 & \qquad-\eta(2-5\eta)\left\Vert \bcT_{U}+\bcT_{V}+\bcT_{W}\right\Vert _{\fro}^{2}-\eta(2-5\eta)\Big(\left\Vert \bD_{U}\right\Vert _{\fro}^{2}+\left\Vert \bD_{V}\right\Vert _{\fro}^{2}+\left\Vert \bD_{W}\right\Vert _{\fro}^{2}\Big)\nonumber \\
 & \qquad+2\eta(1-\eta)C\epsilon\dist^{2}(\bF_{t},\bF_{\star})+\eta^{2}C\epsilon\dist^{2}(\bF_{t},\bF_{\star}).\label{eq:TF_bound}
\end{align}
Here $C\coloneqq4(C_{1}\vee C_{2})$. As long as $\eta\le2/5$ and
$\epsilon\le0.2/C$, one has 
\begin{align*}
\dist^{2}(\bF_{t+1},\bF_{\star})\le\left((1-\eta)^{2}+2\eta(1-\eta)C\epsilon+\eta^{2}C\epsilon\right)\dist^{2}(\bF_{t},\bF_{\star})\le(1-0.7\eta)^{2}\dist^{2}(\bF_{t},\bF_{\star}),
\end{align*}
and therefore we arrive at the conclusion that $\dist(\bF_{t+1},\bF_{\star})\le(1-0.7\eta)\dist(\bF_{t},\bF_{\star})$. In addition, the relation \eqref{eq:tensor2factors} in Lemma~\ref{lemma:perturb_bounds} guarantees that $\|(\bU_{t},\bV_{t},\bW_{t})\bcdot\bcS_{t}-\bcX_{\star}\|_{\fro}\le 3\dist(\bF_{t},\bF_{\star})$.

It then boils down to demonstrating the four bounds~\eqref{eq:final-bounds}.
Due to the symmetry among $\bU,\bV$ and $\bW$, we will focus on
proving the bounds~\eqref{eq:U-bound} and~\eqref{eq:S-bound},
omitting the proofs for the other two.

\paragraph{Proof of \eqref{eq:U-bound}.}

Utilize the ScaledGD update rule~\eqref{eq:iterates_TF} to write
\begin{align}
(\bU_{t+1}\bQ_{t,1}-\bU_{\star})\bSigma_{\star,1} & =\left(\bU-\eta\cM_{1}\left((\bU,\bV,\bW)\bcdot\bcS-\bcX_{\star}\right)\breve{\bU}(\breve{\bU}^{\top}\breve{\bU})^{-1}-\bU_{\star}\right)\bSigma_{\star,1}\nonumber \\
 & =(1-\eta)\bDelta_{U}\bSigma_{\star,1}-\eta\bU_{\star}(\breve{\bU}-\breve{\bU}_{\star})^{\top}\breve{\bU}(\breve{\bU}^{\top}\breve{\bU})^{-1}\bSigma_{\star,1},\label{eq:first_square}
\end{align}
where we use the decomposition of the mode-$1$ matricization 
\begin{align*}
\cM_{1}\left((\bU,\bV,\bW)\bcdot\bcS-\bcX_{\star}\right) & =\bU\cM_{1}(\bcS)(\bW\otimes\bV)^{\top}-\bU_{\star}\cM_{1}(\bcS_{\star})(\bW_{\star}\otimes\bV_{\star})^{\top}\\
 & =\bDelta_{U}\cM_{1}(\bcS)(\bW\otimes\bV)^{\top}+\bU_{\star}\left(\cM_{1}(\bcS)(\bW\otimes\bV)^{\top}-\cM_{1}(\bcS_{\star})(\bW_{\star}\otimes\bV_{\star})^{\top}\right)\\
 & =\bDelta_{U}\breve{\bU}^{\top}+\bU_{\star}(\breve{\bU}-\breve{\bU}_{\star})^{\top}.
\end{align*}
Take the squared norm of both sides of the identity~\eqref{eq:first_square}
to obtain 
\begin{align*}
\left\Vert (\bU_{t+1}\bQ_{t,1}-\bU_{\star})\bSigma_{\star,1}\right\Vert _{\fro}^{2} & =(1-\eta)^{2}\|\bDelta_{U}\bSigma_{\star,1}\|_{\fro}^{2}-2\eta(1-\eta)\underbrace{\big\langle\bDelta_{U}\bSigma_{\star,1},\bU_{\star}(\breve{\bU}-\breve{\bU}_{\star})^{\top}\breve{\bU}(\breve{\bU}^{\top}\breve{\bU})^{-1}\bSigma_{\star,1}\big\rangle}_{\eqqcolon\mfk{U}_{1}}\\
 & \qquad+\eta^{2}\underbrace{\big\|\bU_{\star}(\breve{\bU}-\breve{\bU}_{\star})^{\top}\breve{\bU}(\breve{\bU}^{\top}\breve{\bU})^{-1}\bSigma_{\star,1}\big\|_{\fro}^{2}}_{\eqqcolon\mfk{U}_{2}}.
\end{align*}

The following two claims bound the two terms $\mfk{U}_{1}$ and $\mfk{U}_{2}$,
whose proofs can be found in Appendix~\ref{proof:claim_U1} and Appendix~\ref{proof:claim_U2},
respectively.

\begin{claim}\label{claim:U1} $\mfk{U}_{1}\geq \left\langle \bcT_{U},\bcT_{U}+\bcT_{V}+\bcT_{W}\right\rangle  - C_{1}\epsilon\dist^{2}(\bF_{t},\bF_{\star})$. \end{claim}

\begin{claim}\label{claim:U2} $\mfk{U}_{2}\le\left\Vert \bcT_{U}+\bcT_{V}+\bcT_{W}\right\Vert _{\fro}^{2}+C_{2}\epsilon\dist^{2}(\bF_{t},\bF_{\star})$. \end{claim}

We can combine the above two claims to obtain that  
\begin{multline*}
\left\Vert (\bU_{t+1}\bQ_{t,1}-\bU_{\star})\bSigma_{\star,1}\right\Vert _{\fro}^{2}\le(1-\eta)^{2}\|\bDelta_{U}\bSigma_{\star,1}\|_{\fro}^{2}-2\eta(1-\eta)\left\langle \bcT_{U},\bcT_{U}+\bcT_{V}+\bcT_{W}\right\rangle \\
+\eta^{2}\left\Vert \bcT_{U}+\bcT_{V}+\bcT_{W}\right\Vert _{\fro}^{2}+2\eta(1-\eta)C_{1}\epsilon\dist^{2}(\bF_{t},\bF_{\star})+\eta^{2}C_{2}\epsilon\dist^{2}(\bF_{t},\bF_{\star}),
\end{multline*}
as long as $\eta <1$. This proves the bound~\eqref{eq:U-bound}.

\paragraph{Proof of \eqref{eq:S-bound}.}

Again, we use the ScaledGD update rule~\eqref{eq:iterates_TF}
and the decomposition $\bcS=\bDelta_{\cS}+\bcS_{\star}$ to obtain
\begin{align}
 & (\bQ_{t,1}^{-1},\bQ_{t,2}^{-1},\bQ_{t,3}^{-1})\bcdot\bcS_{t+1}-\bcS_{\star} \nonumber \\
 & \quad=\bcS-\eta\left((\bU^{\top}\bU)^{-1}\bU^{\top},(\bV^{\top}\bV)^{-1}\bV^{\top},(\bW^{\top}\bW)^{-1}\bW^{\top}\right)\bcdot \big((\bU,\bV,\bW)\bcdot\bcS-\bcX_{\star} \big)-\bcS_{\star} \nonumber\\
 & \quad=(1-\eta)\bDelta_{\cS}-\eta\left((\bU^{\top}\bU)^{-1}\bU^{\top},(\bV^{\top}\bV)^{-1}\bV^{\top},(\bW^{\top}\bW)^{-1}\bW^{\top}\right)\bcdot\big((\bU,\bV,\bW)\bcdot\bcS_{\star}-\bcX_{\star}\big), \label{eq:last_square}
\end{align}
where we used \eqref{eq:tensor_properties_c} in the last line. Expand the squared norm of both sides to reach 
\begin{align*}
 & \left\Vert (\bQ_{t,1}^{-1},\bQ_{t,2}^{-1},\bQ_{t,3}^{-1})\bcdot\bcS_{t+1}-\bcS_{\star}\right\Vert _{\fro}^{2}=(1-\eta)^{2}\|\bDelta_{\cS}\|_{\fro}^{2}\\
 & \qquad-2\eta(1-\eta)\underbrace{\left\langle \bDelta_{\cS},\left((\bU^{\top}\bU)^{-1}\bU^{\top},(\bV^{\top}\bV)^{-1}\bV^{\top},(\bW^{\top}\bW)^{-1}\bW^{\top}\right)\bcdot\left((\bU,\bV,\bW)\bcdot\bcS_{\star}-\bcX_{\star}\right)\right\rangle }_{\eqqcolon\mfk{S}_{1}}\\
 & \qquad+\eta^{2}\underbrace{\left\Vert \left((\bU^{\top}\bU)^{-1}\bU^{\top},(\bV^{\top}\bV)^{-1}\bV^{\top},(\bW^{\top}\bW)^{-1}\bW^{\top}\right)\bcdot\left((\bU,\bV,\bW)\bcdot\bcS_{\star}-\bcX_{\star}\right)\right\Vert _{\fro}^{2}}_{\eqqcolon\mfk{S}_{2}}.
\end{align*}

We collect the bounds of the two relevant terms $\mfk{S}_{1}$ and
$\mfk{S}_{2}$ in the following two claims, whose proofs can be found in Appendix~\ref{proof:claim_S1} and Appendix~\ref{proof:claim_S2},
respectively.

\begin{claim}\label{claim:S1} $\mfk{S}_{1} \geq \|\bD_{U}\|_{\fro}^{2}+\|\bD_{V}\|_{\fro}^{2}+\|\bD_{W}\|_{\fro}^{2} - C_{1}\epsilon\dist^{2}(\bF_{t},\bF_{\star})$. \end{claim}

\begin{claim}\label{claim:S2} $\mfk{S}_{2}\le3\left(\|\bD_{U}\|_{\fro}^{2}+\|\bD_{V}\|_{\fro}^{2}+\|\bD_{W}\|_{\fro}^{2}\right)+C_{2}\epsilon\dist^{2}(\bF_{t},\bF_{\star}).$\end{claim}

Take the bounds on $\mfk{S}_{1}$ and $\mfk{S}_{2}$ collectively
to reach 
\begin{align*}
 & \left\Vert (\bQ_{t,1}^{-1},\bQ_{t,2}^{-1},\bQ_{t,3}^{-1})\bcdot\bcS_{t+1}-\bcS_{\star}\right\Vert _{\fro}^{2}\le(1-\eta)^{2}\|\bDelta_{\cS}\|_{\fro}^{2}-\eta(2-5\eta)\left(\|\bD_{U}\|_{\fro}^{2}+\|\bD_{V}\|_{\fro}^{2}+\|\bD_{W}\|_{\fro}^{2}\right)\nonumber \\
 & \qquad+2\eta(1-\eta)C_{1}\epsilon\dist^{2}(\bF_{t},\bF_{\star})+\eta^{2}C_{2}\epsilon\dist^{2}(\bF_{t},\bF_{\star})
\end{align*}
as long as $\eta<1$. This recovers the bound~\eqref{eq:S-bound}. 


\subsection{Proof of Claim~\ref{claim:U1}}

\label{proof:claim_U1}

Use the relation \eqref{eq:decomp_R} to decompose $\mfk{U}_{1}$
as 
\begin{align*}
\mfk{U}_{1} & = \big\langle\bU_{\star}^{\top}\bDelta_{U}\bSigma_{\star,1},(\breve{\bU}-\breve{\bU}_{\star})^{\top}\breve{\bU}(\breve{\bU}^{\top}\breve{\bU})^{-1}\bSigma_{\star,1}\big\rangle \\
 & =\underbrace{\big\langle \bU_{\star}^{\top}\bDelta_{U}\bSigma_{\star,1},\cM_{1}(\bcS_{\star})(\bW\otimes\bDelta_{V}+\bDelta_{W}\otimes\bV_{\star})^{\top}\breve{\bU}(\breve{\bU}^{\top}\breve{\bU})^{-1}\bSigma_{\star,1}\big\rangle }_{ \eqqcolon   \mfk{U}_{1,1}}\\
 & \qquad+\underbrace{\big\langle \bU_{\star}^{\top}\bDelta_{U}\bSigma_{\star,1},\cM_{1}(\bDelta_{\cS})(\bW\otimes\bV)^{\top}\breve{\bU}(\breve{\bU}^{\top}\breve{\bU})^{-1}\bSigma_{\star,1}\big\rangle }_{ \eqqcolon  \mfk{U}_{1,2}}.
\end{align*}
In what follows, we bound $\mfk{U}_{1,1}$ and $\mfk{U}_{1,2}$ separately.

\paragraph{Step 1: tackling $\mfk{U}_{1,1}$.}
We can further decompose $\mfk{U}_{1,1}$ into the following four
terms 
\begin{align*}
\mfk{U}_{1,1} & =\underbrace{\left\langle \bU_{\star}^{\top}\bDelta_{U}\bSigma_{\star,1},\cM_{1}(\bcS_{\star})(\bW_{\star}\otimes\bDelta_{V}+\bDelta_{W}\otimes\bV_{\star})^{\top}\breve{\bU}_{\star}\bSigma_{\star,1}^{-1} \right\rangle }_{\eqqcolon  \mfk{U}_{1,1}^{\main}}\\
 & \qquad+\underbrace{\left\langle \bU_{\star}^{\top}\bDelta_{U}\bSigma_{\star,1},\cM_{1}(\bcS_{\star})(\bW_{\star}\otimes\bDelta_{V})^{\top}\left(\breve{\bU}(\breve{\bU}^{\top}\breve{\bU})^{-1}\bSigma_{\star,1}-\breve{\bU}_{\star}\bSigma_{\star,1}^{-1}\right)\right\rangle }_{\eqqcolon \mfk{U}_{1,1}^{\ptb, 1}}\\
 & \qquad+\underbrace{\left\langle \bU_{\star}^{\top}\bDelta_{U}\bSigma_{\star,1},\cM_{1}(\bcS_{\star})(\bDelta_{W}\otimes\bV_{\star})^{\top}\left(\breve{\bU}(\breve{\bU}^{\top}\breve{\bU})^{-1}\bSigma_{\star,1}-\breve{\bU}_{\star}\bSigma_{\star,1}^{-1}\right)\right\rangle }_{\eqqcolon \mfk{U}_{1,1}^{\ptb,2}}\\
 & \qquad+\underbrace{\left\langle \bU_{\star}^{\top}\bDelta_{U}\bSigma_{\star,1},\cM_{1}(\bcS_{\star})(\bDelta_{W}\otimes\bDelta_{V})^{\top}\breve{\bU}(\breve{\bU}^{\top}\breve{\bU})^{-1}\bSigma_{\star,1}\right\rangle }_{\eqqcolon\mfk{U}_{1,1}^{\ptb,3}},
\end{align*}
where $\mfk{U}_{1,1}^{\main}$ denotes the main term and the remaining
ones are perturbation terms. 

Utilizing the definition of $\breve{\bU}_{\star}$ in \eqref{eq:short_notations} and the relation \eqref{eq:matricization}, the main term $\mfk{U}_{1,1}^{\main}$ can be rewritten as an inner product
in the tensor space: 
\begin{align*}
\mfk{U}_{1,1}^{\main} & =\left\langle \bU_{\star}^{\top}\bDelta_{U}\cM_{1}(\bcS_{\star}),\cM_{1}(\bcS_{\star})(\bI_{r_{3}}\otimes\bDelta_{V}^{\top}\bV_{\star}+\bDelta_{W}^{\top}\bW_{\star}\otimes\bI_{r_{2}})\right\rangle \\
 & =\left\langle \bcT_{U},\bcT_{V}+\bcT_{W}\right\rangle .
\end{align*}
To control the other three perturbation terms, Lemma~\ref{lemma:perturb_bounds}
turns out to be extremely useful. For instance, the perturbation term $\mfk{U}_{1,1}^{\ptb,1}$
is bounded by 
\begin{align*}
|\mfk{U}_{1,1}^{\ptb,1}| & \le\left\Vert \bU_{\star}^{\top}\bDelta_{U}\bSigma_{\star,1}\right\Vert _{\fro}\left\Vert \cM_{1}(\bcS_{\star})(\bW_{\star}\otimes\bDelta_{V})^{\top}\right\Vert _{\fro}\left\Vert \breve{\bU}(\breve{\bU}^{\top}\breve{\bU})^{-1}\bSigma_{\star,1}-\breve{\bU}_{\star}\bSigma_{\star,1}^{-1} \right\Vert _{\op}\\
 & \le\frac{\sqrt{2}(3\epsilon+3\epsilon^{2}+\epsilon^{3})}{(1-\epsilon)^{3}}\|\bDelta_{U}\bSigma_{\star,1}\|_{\fro}\|\bDelta_{V}\bSigma_{\star,2}\|_{\fro}.
\end{align*}
Here in the last inequality, we used the upper bound~\eqref{eq:perturb_Rinv_d}
and changed the matricization mode to obtain 
\begin{align*}
\left\Vert \cM_{1}(\bcS_{\star})(\bW_{\star}\otimes\bDelta_{V})^{\top}\right\Vert _{\fro} & =\|(\bI_{r_{1}},\bDelta_{V},\bW_{\star})\bcdot\bcS_{\star}\|_{\fro}=\left\Vert \bDelta_{V}\cM_{2}(\bcS_{\star})(\bW_{\star}\otimes\bI_{r_{1}})^{\top}\right\Vert _{\fro}\le\|\bDelta_{V}\bSigma_{\star,2}\|_{\fro}.
\end{align*}
Similarly, the remaining two perturbation terms $\mfk{U}_{1,1}^{\ptb,2}$
and $\mfk{U}_{1,1}^{\ptb,3}$ obey 
\begin{align*}
|\mfk{U}_{1,1}^{\ptb,2}| & \le\frac{\sqrt{2}(3\epsilon+3\epsilon^{2}+\epsilon^{3})}{(1-\epsilon)^{3}}\|\bDelta_{U}\bSigma_{\star,1}\|_{\fro}\|\bDelta_{W}\bSigma_{\star,3}\|_{\fro},\\
|\mfk{U}_{1,1}^{\ptb,3}| & \le\frac{\epsilon}{(1-\epsilon)^{3}}\|\bDelta_{U}\bSigma_{\star,1}\|_{\fro}\|\bDelta_{V}\bSigma_{\star,2}\|_{\fro}.
\end{align*}

\paragraph{Step 2: tackling $\mfk{U}_{1,2}$.}
Now we move on to $\mfk{U}_{1,2}$, which can be decomposed as 
\begin{align*}
\mfk{U}_{1,2} & =\left\langle \bU_{\star}^{\top}\bDelta_{U}\bSigma_{\star,1},\cM_{1}(\bDelta_{\cS})\cM_{1}(\bcS_{\star})^{\top}\bSigma_{\star,1}^{-1}\right\rangle \\
 & \qquad+\underbrace{\left\langle \bU_{\star}^{\top}\bDelta_{U}\bSigma_{\star,1},\cM_{1}(\bDelta_{\cS})(\bW_{\star}\otimes\bV_{\star})^{\top}\left(\breve{\bU}(\breve{\bU}^{\top}\breve{\bU})^{-1}\bSigma_{\star,1}-\breve{\bU}_{\star}\bSigma_{\star,1}^{-1}\right)\right\rangle }_{ \eqqcolon \mfk{U}_{1,2}^{\ptb,1}}\\
 & \qquad+\underbrace{\left\langle \bU_{\star}^{\top}\bDelta_{U}\bSigma_{\star,1},\cM_{1}(\bDelta_{\cS})(\bW\otimes\bV-\bW_{\star}\otimes\bV_{\star})^{\top}\breve{\bU}(\breve{\bU}^{\top}\breve{\bU})^{-1}\bSigma_{\star,1}\right\rangle }_{\eqqcolon \mfk{U}_{1,2}^{\ptb,2}}\\
 & =\left\langle \bU_{\star}^{\top}\bDelta_{U}\bSigma_{\star,1},\cM_{1}(\bDelta_{\cS})\cM_{1}(\bcS)^{\top}\bSigma_{\star,1}^{-1}\right\rangle \underbrace{ -\left\langle \bU_{\star}^{\top}\bDelta_{U}\bSigma_{\star,1},\cM_{1}(\bDelta_{\cS})\cM_{1}(\bDelta_{\cS})^{\top}\bSigma_{\star,1}^{-1}\right\rangle }_{\eqqcolon\mfk{U}_{1,2}^{\ptb,3}}+\mfk{U}_{1,2}^{\ptb,1}+\mfk{U}_{1,2}^{\ptb,2}\\
 & =\left\langle \bU_{\star}^{\top}\bDelta_{U}\bSigma_{\star,1},\bU^{\top}\bDelta_{U}\bSigma_{\star,1}\right\rangle +\mfk{U}_{1,2}^{\ptb,1}+\mfk{U}_{1,2}^{\ptb,2} + \mfk{U}_{1,2}^{\ptb,3}\\
 & =\left\Vert \bcT_{U}\right\Vert _{\fro}^{2}+\mfk{U}_{1,2}^{\ptb,1}+\mfk{U}_{1,2}^{\ptb,2} + \mfk{U}_{1,2}^{\ptb,3}+\underbrace{\left\langle \bU_{\star}^{\top}\bDelta_{U}\bSigma_{\star,1},\bDelta_{U}^{\top}\bDelta_{U}\bSigma_{\star,1}\right\rangle }_{\eqqcolon \mfk{U}_{1,2}^{\ptb,4}},
\end{align*}
where in the penultimate identity we have applied the identity~\eqref{eq:alignment}
to replace $\cM_{1}(\bDelta_{\cS})\cM_{1}(\bcS)^{\top}$. Again, by
Lemma~\ref{lemma:perturb_bounds}, the perturbation term $\mfk{U}_{1,2}^{\ptb,1}$
is bounded by 
\begin{align*}
|\mfk{U}_{1,2}^{\ptb,1}| & \le\left\Vert \bU_{\star}^{\top}\bDelta_{U}\bSigma_{\star,1}\right\Vert _{\fro}\left\Vert \cM_{1}(\bDelta_{\cS})(\bW_{\star}\otimes\bV_{\star})^{\top}\right\Vert _{\fro}\big\Vert \breve{\bU}(\breve{\bU}^{\top}\breve{\bU})^{-1}\bSigma_{\star,1}-\breve{\bU}_{\star}\bSigma_{\star,1}^{-1} \big\Vert _{\op}\\
 & \le\frac{\sqrt{2}(3\epsilon+3\epsilon^{2}+\epsilon^{3})}{(1-\epsilon)^{3}}\|\bDelta_{U}\bSigma_{\star,1}\|_{\fro}\|\bDelta_{\cS}\|_{\fro}.
\end{align*}
In addition, $\mfk{U}_{1,2}^{\ptb,2}$ is bounded by 
\begin{align*}
|\mfk{U}_{1,2}^{\ptb,2}| & \le\left\Vert \bU_{\star}^{\top}\bDelta_{U}\bSigma_{\star,1}\right\Vert _{\fro}\|\cM_{1}(\bDelta_{\cS})\|_{\fro}\left\Vert \bW\otimes\bV-\bW_{\star}\otimes\bV_{\star}\right\Vert _{\op}\big\Vert \breve{\bU}(\breve{\bU}^{\top}\breve{\bU})^{-1}\bSigma_{\star,1} \big\Vert _{\op}\\
 & \le\frac{2\epsilon+\epsilon^{2}}{(1-\epsilon)^{3}}\|\bDelta_{U}\bSigma_{\star,1}\|_{\fro}\|\bDelta_{\cS}\|_{\fro},
\end{align*}
where we have used 
\begin{align*}
\left\Vert \bW\otimes\bV-\bW_{\star}\otimes\bV_{\star}\right\Vert _{\op} & \le\|\bDelta_{W}\otimes\bV_{\star}\|_{\op}+\|\bW_{\star}\otimes\bDelta_{V}\|_{\op}+\|\bDelta_{W}\otimes\bDelta_{V}\|_{\op}\\
 & \le\|\bDelta_{W}\|_{\op}+\|\bDelta_{V}\|_{\op}+\|\bDelta_{V}\|_{\op}\|\bDelta_{W}\|_{\op}\le2\epsilon+\epsilon^{2}.
\end{align*}
Following similar arguments (i.e.~repeatedly using Lemma~\ref{lemma:perturb_bounds}),
we can bound $\mfk{U}_{1,2}^{\ptb,3}$ and $\mfk{U}_{1,2}^{\ptb,4}$ as 
\begin{align*}
|\mfk{U}_{1,2}^{\ptb,3}| & \le\left\Vert \bU_{\star}^{\top}\bDelta_{U}\bSigma_{\star,1}\right\Vert _{\fro}\|\cM_{1}(\bDelta_{\cS})\|_{\fro}\left\Vert \cM_{1}(\bDelta_{\cS})^{\top}\bSigma_{\star,1}^{-1}\right\Vert _{\op}\le\epsilon\|\bDelta_{U}\bSigma_{\star,1}\|_{\fro}\|\bDelta_{\cS}\|_{\fro};\\
|\mfk{U}_{1,2}^{\ptb,4}| & \le\left\Vert \bU_{\star}^{\top}\bDelta_{U}\bSigma_{\star,1}\right\Vert _{\fro}\|\bDelta_{U}\|_{\op}\|\bDelta_{U}\bSigma_{\star,1}\|_{\fro}\le\epsilon\|\bDelta_{U}\bSigma_{\star,1}\|_{\fro}^{2}.
\end{align*}

\paragraph{Step 3: putting the bound together.} Combine these results on $\mfk{U}_{1,1}$ and $\mfk{U}_{1,2}$ to
see 
\begin{align*}
\mfk{U}_{1} & =\left\langle \bcT_{U},\bcT_{U}+\bcT_{V}+\bcT_{W}\right\rangle +\mfk{U}_{1}^{\ptb},
\end{align*}
where the perturbation term $\mfk{U}_{1}^{\ptb}:= \sum_{i=1}^{3} \mfk{U}_{1,1}^{\ptb,i} + \sum_{i=1}^4\mfk{U}_{1,2}^{\ptb,i}$ obeys 
\begin{align*}
|\mfk{U}_{1}^{\ptb}|\le\epsilon\|\bDelta_{U}\bSigma_{\star,1}\|_{\fro}\Big(\|\bDelta_{U}\bSigma_{\star,1}\|_{\fro} & +\frac{1+\sqrt{2}(3+3\epsilon+\epsilon^{2})}{(1-\epsilon)^{3}}\|\bDelta_{V}\bSigma_{\star,2}\|_{\fro}+\frac{\sqrt{2}(3+3\epsilon+\epsilon^{2})}{(1-\epsilon)^{3}}\|\bDelta_{W}\bSigma_{\star,3}\|_{\fro} \\
 & + (1+\frac{2+\epsilon+\sqrt{2}(3+3\epsilon+\epsilon^{2})}{(1-\epsilon)^{3}})\|\bDelta_{\cS}\|_{\fro}\Big).
\end{align*}
Using the Cauchy--Schwarz inequality, we can further simplify it
as $|\mfk{U}_{1}^{p}|\le C_{1}\epsilon\dist^{2}(\bF_{t},\bF_{\star})$
for some universal constant $C_{1}>1$. 

\subsection{Proof of Claim~\ref{claim:U2}}

\label{proof:claim_U2} 

Note that
\begin{align}
\mfk{U}_{2} & =\big\Vert (\breve{\bU}-\breve{\bU}_{\star})^{\top}\breve{\bU}(\breve{\bU}^{\top}\breve{\bU})^{-1}\bSigma_{\star,1}\big\Vert _{\fro}^{2}\nonumber\\
 & \le\big\Vert (\breve{\bU}-\breve{\bU}_{\star})^{\top}\breve{\bU}\bSigma_{\star,1}^{-1}\big\Vert _{\fro}^{2}\big\Vert \bSigma_{\star,1}(\breve{\bU}^{\top}\breve{\bU})^{-1}\bSigma_{\star,1}\big\Vert _{\op}^{2}\nonumber\\
 & \le\big\Vert (\breve{\bU}-\breve{\bU}_{\star})^{\top}\breve{\bU}\bSigma_{\star,1}^{-1}\big\Vert _{\fro}^{2}(1-\epsilon)^{-12}, \label{eq:U2_bound}
\end{align}
where the last relation arises from the bound~\eqref{eq:perturb_Rinv_2}
in Lemma~\ref{lemma:perturb_bounds}. We can then use the decomposition~\eqref{eq:decomp_R}
to obtain 
\begin{align*}
\big\Vert (\breve{\bU}-\breve{\bU}_{\star})^{\top}\breve{\bU}\bSigma_{\star,1}^{-1}\big\Vert _{\fro} & =\left\Vert \Big(\cM_{1}(\bcS_{\star})(\bW\otimes\bDelta_{V}+\bDelta_{W}\otimes\bV_{\star})^{\top}+\cM_{1}(\bDelta_{\cS})(\bW\otimes\bV)^{\top}\Big)(\bW\otimes\bV)\cM_{1}(\bcS)^{\top}\bSigma_{\star,1}^{-1}\right\Vert _{\fro}\\
 & \le\underbrace{\left\Vert \cM_{1}(\bcS_{\star})\left(\bI_{r_{3}}\otimes\bDelta_{V}^{\top}\bV_{\star}+\bDelta_{W}^{\top}\bW_{\star}\otimes\bI_{r_{2}}\right)\cM_{1}(\bcS_{\star})^{\top}\bSigma_{\star,1}^{-1}+\cM_{1}(\bDelta_{\cS})\cM_{1}(\bcS)^{\top}\bSigma_{\star,1}^{-1}\right\Vert _{\fro}}_{\eqqcolon\mfk{U}_{2}^{\main}}\\
 & \qquad+\underbrace{\left\Vert \cM_{1}(\bcS_{\star})\left(\bW^{\top}\bW\otimes\bDelta_{V}^{\top}\bV-\bI_{r_{3}}\otimes\bDelta_{V}^{\top}\bV_{\star}\right)\cM_{1}(\bcS_{\star})^{\top}\bSigma_{\star,1}^{-1}\right\Vert _{\fro}}_{\eqqcolon\mfk{U}_{2}^{\ptb,1}}\\
 & \qquad+\underbrace{\left\Vert \cM_{1}(\bcS_{\star})\left(\bDelta_{W}^{\top}\bW\otimes\bV_{\star}^{\top}\bV-\bDelta_{W}^{\top}\bW_{\star}\otimes\bI_{r_{2}}\right)\cM_{1}(\bcS_{\star})^{\top}\bSigma_{\star,1}^{-1}\right\Vert _{\fro}}_{\eqqcolon\mfk{U}_{2}^{\ptb,2}}\\
 & \qquad+\underbrace{\left\Vert \cM_{1}(\bcS_{\star})\left(\bW^{\top}\bW\otimes\bDelta_{V}^{\top}\bV+\bDelta_{W}^{\top}\bW\otimes\bV_{\star}^{\top}\bV\right)\cM_{1}(\bDelta_{\cS})^{\top}\bSigma_{\star,1}^{-1}\right\Vert _{\fro}}_{\eqqcolon\mfk{U}_{2}^{\ptb,3}}\\
 & \qquad+\underbrace{\left\Vert \cM_{1}(\bDelta_{\cS})\left(\bW^{\top}\bW\otimes\bV^{\top}\bV-\bI_{r_{3}}\otimes\bI_{r_{2}}\right)\cM_{1}(\bcS)^{\top}\bSigma_{\star,1}^{-1}\right\Vert _{\fro}}_{\eqqcolon\mfk{U}_{2}^{\ptb,4}}.
\end{align*}
Here, $\mfk{U}_{2}^{\main}$ is the main term while the remaining four
are perturbation terms. Use the relation~\eqref{eq:alignment} again
to replace $\cM_{1}(\bDelta_{\cS})\cM_{1}(\bcS)^{\top}$ in the main
term $\mfk{U}_{2}^{\main}$ and see 
\begin{align*}
\mfk{U}_{2}^{\main} & =\left\Vert \Big(\cM_{1}(\bcS_{\star})(\bI_{r_{3}}\otimes\bDelta_{V}^{\top}\bV_{\star}+\bDelta_{W}^{\top}\bW_{\star}\otimes\bI_{r_{2}})+\bU_{\star}^{\top}\bDelta_{U}\cM_{1}(\bcS_{\star})\Big)\cM_{1}(\bcS_{\star})^{\top}\bSigma_{\star,1}^{-1}\right\Vert _{\fro}\\
 & \le\left\Vert \cM_{1}(\bcS_{\star})(\bI_{r_{3}}\otimes\bDelta_{V}^{\top}\bV_{\star}+\bDelta_{W}^{\top}\bW_{\star}\otimes\bI_{r_{2}})+\bU_{\star}^{\top}\bDelta_{U}\cM_{1}(\bcS_{\star})\right\Vert _{\fro}\|\cM_{1}(\bcS_{\star})^{\top}\bSigma_{\star,1}^{-1}\|_{\op}\\
 & =\left\Vert \bcT_{U}+\bcT_{V}+\bcT_{W}\right\Vert _{\fro},
\end{align*}
where the last equality uses $\|\cM_{1}(\bcS_{\star})^{\top}\bSigma_{\star,1}^{-1}\|_{\op}=1$.
The perturbation terms are bounded by 
\begin{align*}
\mfk{U}_{2}^{\ptb,1} & \le((1+\epsilon)^{3}-1)\|\bDelta_{V}\bSigma_{\star,2}\|_{\fro};\\
\mfk{U}_{2}^{\ptb,2} & \le((1+\epsilon)^{2}-1)\|\bDelta_{W}\bSigma_{\star,3}\|_{\fro};\\
\mfk{U}_{2}^{\ptb,3} & \le\epsilon(1+\epsilon)^{3}\|\bDelta_{V}\bSigma_{\star,2}\|_{\fro}+\epsilon(1+\epsilon)^{2}\|\bDelta_{W}\bSigma_{\star,3}\|_{\fro};\\
\mfk{U}_{2}^{\ptb,4} & \le((1+\epsilon)^{4}-1)(1+\epsilon)\|\bDelta_{\cS}\|_{\fro}.
\end{align*}
They follow from similar calculations as those in bounding $\mfk{U}_{1}$ with the aid of Lemma~\ref{lemma:perturb_bounds};
hence we omit the details for brevity. Combine these results to see 
\begin{align*}
\big\Vert (\breve{\bU}-\breve{\bU}_{\star})^{\top}\breve{\bU}\bSigma_{\star,1}^{-1} \big\Vert _{\fro} & \le\left\Vert \bcT_{U}+\bcT_{V}+\bcT_{W}\right\Vert _{\fro}+\mfk{U}_{2}^{\ptb},
\end{align*}
with $\mfk{U}_{2}^{\ptb}: = \sum_{i=1}^4 \mfk{U}_{2}^{\ptb,i}$ obeying 
\begin{align*}
\mfk{U}_{2}^{\ptb} & \le((1+\epsilon)^{4}-1)\|\bDelta_{V}\bSigma_{\star,2}\|_{\fro}+((1+\epsilon)^{3}-1)\|\bDelta_{W}\bSigma_{\star,3}\|_{\fro}+((1+\epsilon)^{4}-1)(1+\epsilon)\|\bDelta_{\cS}\|_{\fro} \\
&  \lesssim \epsilon \big(\|\bDelta_{V}\bSigma_{\star,2}\|_{\fro}+ \|\bDelta_{W} \bSigma_{\star,3}\|_{\fro} + \|\bDelta_{\cS}\|_{\fro} \big)\lesssim \epsilon \dist(\bF_{t},\bF_{\star}).
\end{align*}
Next take the square to obtain 
\begin{align*}
\big\|(\breve{\bU}-\breve{\bU}_{\star})^{\top}\breve{\bU}\bSigma_{\star,1}^{-1}\big\|_{\fro}^{2} & \le\left\Vert \bcT_{U}+\bcT_{V}+\bcT_{W}\right\Vert _{\fro}^{2}+2\mfk{U}_{2}^{\ptb}\left\Vert \bcT_{U}+\bcT_{V}+\bcT_{W}\right\Vert _{\fro}+(\mfk{U}_{2}^{\ptb})^2.
\end{align*}
Finally plug this back into \eqref{eq:U2_bound} to conclude 
\begin{align*}
\mfk{U}_{2} & \le (1-\epsilon)^{-12}\left\Vert \bcT_{U}+\bcT_{V}+\bcT_{W}\right\Vert _{\fro}^{2}+2(1-\epsilon)^{-12}\mfk{U}_{2}^{\ptb}\left\Vert \bcT_{U}+\bcT_{V}+\bcT_{W}\right\Vert _{\fro}+(1-\epsilon)^{-12}(\mfk{U}_{2}^{\ptb})^{2} \\
&\le \left\Vert \bcT_{U}+\bcT_{V}+\bcT_{W}\right\Vert _{\fro}^{2}+\left((1-\epsilon)^{-12}-1\right)\left(\|\bDelta_{U}\bSigma_{\star,1}\|_{\fro}+\|\bDelta_{V}\bSigma_{\star,2}\|_{\fro}+\|\bDelta_{W}\bSigma_{\star,3}\|_{\fro}\right)^{2}\\
& \qquad+2(1-\epsilon)^{-12}\mfk{U}_{2}^{\ptb}\left(\|\bDelta_{U}\bSigma_{\star,1}\|_{\fro}+\|\bDelta_{V}\bSigma_{\star,2}\|_{\fro}+\|\bDelta_{W}\bSigma_{\star,3}\|_{\fro}\right)+(1-\epsilon)^{-12}(\mfk{U}_{2}^{\ptb})^{2} \\
&\le \left\Vert \bcT_{U}+\bcT_{V}+\bcT_{W}\right\Vert _{\fro}^{2} + C_{2}\epsilon \dist^{2}(\bF_{t},\bF_{\star}),
\end{align*}
for some universal constant $C_2>1$. Here in the second inequality, we use the fact that $\left\Vert \bcT_{U}\right\Vert_{\fro}\leq \|\bDelta_{U}\bSigma_{\star,1}\|_{\fro}$, $\left\Vert \bcT_{V}\right\Vert_{\fro}\leq \|\bDelta_{V}\bSigma_{\star,2}\|_{\fro}$, and $\left\Vert \bcT_{W}\right\Vert_{\fro}\leq \|\bDelta_{W}\bSigma_{\star,3}\|_{\fro}$. This finishes the proof of the claim. 

\subsection{Proof of Claim~\ref{claim:S1}}
\label{proof:claim_S1}

Use the decomposition 
\begin{align}
(\bU,\bV,\bW)\bcdot\bcS_{\star}-\bcX_{\star}=(\bDelta_{U},\bV,\bW)\bcdot\bcS_{\star}+(\bU_{\star},\bDelta_{V},\bW)\bcdot\bcS_{\star}+(\bU_{\star},\bV_{\star},\bDelta_{W})\bcdot\bcS_{\star}\label{eq:TF_decomp_S}
\end{align}
to rewrite $\mfk{S}_{1}$ as 
\begin{align*}
\mfk{S}_{1} & =\underbrace{\left\langle \bDelta_{\cS},((\bU^{\top}\bU)^{-1}\bU^{\top}\bDelta_{U},\bI_{r_{2}},\bI_{r_{3}})\bcdot\bcS_{\star}\right\rangle }_{ \eqqcolon \mfk{S}_{1,1}}+\underbrace{\left\langle \bDelta_{\cS},((\bU^{\top}\bU)^{-1}\bU^{\top}\bU_{\star},(\bV^{\top}\bV)^{-1}\bV^{\top}\bDelta_{V},\bI_{r_{3}})\bcdot\bcS_{\star}\right\rangle }_{\eqqcolon \mfk{S}_{1,2}}\\
 & \qquad+\underbrace{\left\langle \bDelta_{\cS},((\bU^{\top}\bU)^{-1}\bU^{\top}\bU_{\star},(\bV^{\top}\bV)^{-1}\bV^{\top}\bV_{\star},(\bW^{\top}\bW)^{-1}\bW^{\top}\bDelta_{W})\bcdot\bcS_{\star}\right\rangle }_{\eqqcolon \mfk{S}_{1,3}}.
\end{align*}

\paragraph{Step 1: tackling $\mfk{S}_{1,1}$.}
Translating the inner product from the tensor space to the matrix
space via the mode-1 matricization yields 
\begin{align*}
\mfk{S}_{1,1} & =\left\langle \cM_{1}(\bDelta_{\cS}),(\bU^{\top}\bU)^{-1}\bU^{\top}\bDelta_{U}\cM_{1}(\bcS_{\star})\right\rangle \\
 & =\underbrace{\left\langle \cM_{1}(\bDelta_{\cS}),(\bU^{\top}\bU)^{-1}\bU^{\top}\bDelta_{U}\cM_{1}(\bcS)\right\rangle }_{\eqqcolon \mfk{S}_{1,1}^{\main}} - \underbrace{\left\langle \cM_{1}(\bDelta_{\cS}),(\bU^{\top}\bU)^{-1}\bU^{\top}\bDelta_{U}\cM_{1}(\bDelta_{\cS})\right\rangle }_{\eqqcolon\mfk{S}_{1,1}^{\ptb}}.
\end{align*}
Again, the identity~\eqref{eq:alignment} is helpful in characterizing
the main term $\mfk{S}_{1,1}^{\main}$:
\begin{align*}
\mfk{S}_{1,1}^{\main} & =\left\langle \bU^{\top}\bDelta_{U}\bSigma_{\star,1}^{2},(\bU^{\top}\bU)^{-1}\bU^{\top}\bDelta_{U}\right\rangle =\big\Vert (\bU^{\top}\bU)^{-1/2}\bU^{\top}\bDelta_{U}\bSigma_{\star,1}\big\Vert _{\fro}^{2}.
\end{align*}
The perturbation term $\mfk{S}_{1,1}^{\ptb}$ is bounded by 
\begin{align*}
|\mfk{S}_{1,1}^{\ptb}|\le\|\cM_{1}(\bDelta_{\cS})\|_{\fro}\left\Vert \bU(\bU^{\top}\bU)^{-1}\right\Vert _{\op}\|\bDelta_{U}\|_{\op}\|\cM_{1}(\bDelta_{\cS})\|_{\fro}\le\epsilon(1-\epsilon)^{-1}\|\bDelta_{\cS}\|_{\fro}^{2},
\end{align*}
which follows directly from Lemma~\ref{lemma:perturb_bounds}. 

\paragraph{Step 2: tackling $\mfk{S}_{1,2}$.}
Following the same recipe as above, we can apply the mode-$2$ matricization
to $\mfk{S}_{1,2}$ to see 
\begin{align*}
\mfk{S}_{1,2} & =\left\langle \cM_{2}(\bDelta_{\cS}),(\bV^{\top}\bV)^{-1}\bV^{\top}\bDelta_{V}\cM_{2}(\bcS_{\star})\left(\bI_{r_{3}}\otimes\bU_{\star}^{\top}\bU(\bU^{\top}\bU)^{-1}\right)\right\rangle \\
 & =\underbrace{\left\langle \cM_{2}(\bDelta_{\cS}),(\bV^{\top}\bV)^{-1}\bV^{\top}\bDelta_{V}\cM_{2}(\bcS)\right\rangle }_{ \eqqcolon\mfk{S}_{1,2}^{\main}} - \underbrace{\left\langle \cM_{2}(\bDelta_{\cS}),(\bV^{\top}\bV)^{-1}\bV^{\top}\bDelta_{V}\cM_{2}(\bDelta_{\cS})\right\rangle }_{ \eqqcolon \mfk{S}_{1,2}^{\ptb,1}}\\
 & \quad+\underbrace{\left\langle \cM_{2}(\bDelta_{\cS}),(\bV^{\top}\bV)^{-1}\bV^{\top}\bDelta_{V}\cM_{2}(\bcS_{\star})\left(\bI_{r_{3}}\otimes(\bU_{\star}^{\top}\bU(\bU^{\top}\bU)^{-1}-\bI_{r_{1}})\right)\right\rangle }_{ \eqqcolon\mfk{S}_{1,2}^{\ptb,2}}.
\end{align*}
In view of the relation~\eqref{eq:alignment}, we can rewrite the
main term $\mfk{S}_{1,2}^{\main}$ as
\begin{align*}
\mfk{S}_{1,2}^{\main} & =\left\Vert (\bV^{\top}\bV)^{-1/2}\bV^{\top}\bDelta_{V}\bSigma_{\star,2}\right\Vert _{\fro}^{2}.
\end{align*}
In addition, for the perturbation terms, Lemma~\ref{lemma:perturb_bounds}
allows us to obtain 
\begin{align*}
|\mfk{S}_{1,2}^{\ptb,1}|\le\|\cM_{2}(\bDelta_{\cS})\|_{\fro}\left\Vert \bV(\bV^{\top}\bV)^{-1}\right\Vert _{\op}\|\bDelta_{V}\|_{\op}\|\cM_{2}(\bDelta_{\cS})\|_{\fro}\le\epsilon(1-\epsilon)^{-1}\|\bDelta_{\cS}\|_{\fro}^{2}.
\end{align*}
Moreover, we can write $\bU_{\star}^{\top}\bU(\bU^{\top}\bU)^{-1}-\bI_{r_{1}}=-\bDelta_{U}^{\top}\bU(\bU^{\top}\bU)^{-1}$, and bound $\mfk{S}_{1,2}^{\ptb,2}$ as 
\begin{align*}
|\mfk{S}_{1,2}^{\ptb,2}| & \le\|\cM_{2}(\bDelta_{\cS})\|_{\fro}\|\bV(\bV^{\top}\bV)^{-1}\|_{\op}\|\bDelta_{V}\cM_{2}(\bcS_{\star})\|_{\fro}\|\bDelta_{U}\|_{\op}\|\bU(\bU^{\top}\bU)^{-1}\|_{\op}\\
 & \le\epsilon(1-\epsilon)^{-2}\|\bDelta_{\cS}\|_{\fro}\|\bDelta_{V}\bSigma_{\star,2}\|_{\fro}.
\end{align*}

\paragraph{Step 3: tackling $\mfk{S}_{1,3}$.}

Similar to before, we rewrite $\mfk{S}_{1,3}$ by applying the mode-$3$ matricization as
\begin{align*}
\mfk{S}_{1,3} & =\left\langle \cM_{3}(\bDelta_{\cS}),(\bW^{\top}\bW)^{-1}\bW^{\top}\bDelta_{W}\cM_{3}(\bcS_{\star})\left(\bV_{\star}^{\top}\bV(\bV^{\top}\bV)^{-1}\otimes\bU_{\star}^{\top}\bU(\bU^{\top}\bU)^{-1}\right)\right\rangle \\
 & =\underbrace{\left\langle \cM_{3}(\bDelta_{\cS}),(\bW^{\top}\bW)^{-1}\bW^{\top}\bDelta_{W}\cM_{3}(\bcS)\right\rangle }_{\eqqcolon\mfk{S}_{1,3}^{\main}}-\underbrace{\left\langle \cM_{3}(\bDelta_{\cS}),(\bW^{\top}\bW)^{-1}\bW^{\top}\bDelta_{W}\cM_{3}(\bDelta_{\cS})\right\rangle }_{ \eqqcolon\mfk{S}_{1,3}^{\ptb,1}}\\
 & \qquad+\underbrace{\left\langle \cM_{3}(\bDelta_{\cS}),(\bW^{\top}\bW)^{-1}\bW^{\top}\bDelta_{W}\cM_{3}(\bcS_{\star})\left(\bV_{\star}^{\top}\bV(\bV^{\top}\bV)^{-1}\otimes\bU_{\star}^{\top}\bU(\bU^{\top}\bU)^{-1}-\bI_{r_{2}}\otimes\bI_{r_{1}}\right)\right\rangle }_{\eqqcolon \mfk{S}_{1,3}^{\ptb,2}}.
\end{align*}
The main term obeys (thanks again to the identity~\eqref{eq:alignment})
\begin{align*}
\mfk{S}_{1,3}^{\main} & =\big\Vert (\bW^{\top}\bW)^{-1/2}\bW^{\top}\bDelta_{W}\bSigma_{\star,3} \big\Vert _{\fro}^{2}.
\end{align*}
As the same time, the perturbation term $\mfk{S}_{1,3}^{\ptb,1}$ can be
bounded by 
\begin{align*}
|\mfk{S}_{1,3}^{\ptb,1}|\le\|\cM_{3}(\bDelta_{\cS})\|_{\fro}\left\Vert \bW(\bW^{\top}\bW)^{-1}\right\Vert _{\op}\|\bDelta_{W}\|_{\op}\|\cM_{3}(\bDelta_{\cS})\|_{\fro}\le\epsilon(1-\epsilon)^{-1}\|\bDelta_{\cS}\|_{\fro}^{2}.
\end{align*}
Similarly, we have 
\begin{align*}
|\mfk{S}_{1,3}^{\ptb,2}| & \le\|\cM_{3}(\bDelta_{\cS})\|_{\fro}\|\bW(\bW^{\top}\bW)^{-1}\|_{\op}\|\bDelta_{W}\cM_{3}(\bcS_{\star})\|_{\fro}\left\Vert \bV_{\star}^{\top}\bV(\bV^{\top}\bV)^{-1}\otimes\bU_{\star}^{\top}\bU(\bU^{\top}\bU)^{-1}-\bI_{r_{2}}\otimes\bI_{r_{1}}\right\Vert _{\op}\\
 & \le\frac{2\epsilon+\epsilon^2}{(1-\epsilon)^{3}}\|\bDelta_{\cS}\|_{\fro}\|\bDelta_{W}\bSigma_{\star,3}\|_{\fro},
\end{align*}
where we use the decomposition
\begin{align*}
\bV_{\star}^{\top}\bV(\bV^{\top}\bV)^{-1}\otimes\bU_{\star}^{\top}\bU(\bU^{\top}\bU)^{-1}-\bI_{r_{2}}\otimes\bI_{r_{1}} = \left(\bV_{\star}\otimes\bU_{\star}-\bV\otimes\bU\right)^{\top}\left(\bV(\bV^{\top}\bV)^{-1}\otimes\bU(\bU^{\top}\bU)^{-1}\right)
\end{align*}
and its immediate consequence 
\begin{align*}
\left\Vert \bV_{\star}^{\top}\bV(\bV^{\top}\bV)^{-1}\otimes\bU_{\star}^{\top}\bU(\bU^{\top}\bU)^{-1}-\bI_{r_{2}}\otimes\bI_{r_{1}}\right\Vert _{\op} & \le\left\Vert \bV_{\star}\otimes\bU_{\star}-\bV\otimes\bU\right\Vert _{\op}\left\Vert \bV(\bV^{\top}\bV)^{-1}\right\Vert _{\op}\left\Vert \bU(\bU^{\top}\bU)^{-1}\right\Vert _{\op} \\
 & \le \frac{2\epsilon+\epsilon^2}{(1-\epsilon)^{2}}.
\end{align*}

\paragraph{Step 4: putting all pieces together.}

Combine results of $\mfk{S}_{1,1},\mfk{S}_{1,2},\mfk{S}_{1,3}$ to
see 
\begin{align*}
\mfk{S}_{1} & =\big\Vert (\bU^{\top}\bU)^{-1/2}\bU^{\top}\bDelta_{U}\bSigma_{\star,1}\big\Vert _{\fro}^{2}+\big\Vert (\bV^{\top}\bV)^{-1/2}\bV^{\top}\bDelta_{V}\bSigma_{\star,2}\big\Vert _{\fro}^{2}+\big\Vert (\bW^{\top}\bW)^{-1/2}\bW^{\top}\bDelta_{W}\bSigma_{\star,3}\big\Vert _{\fro}^{2}+\mfk{S}_{1,p},
\end{align*}
where the aggregated perturbation term $\mfk{S}_{1}^{\ptb}$ obeys 
\begin{align*}
|\mfk{S}_{1}^{\ptb} |\le\epsilon\|\bDelta_{\cS}\|_{\fro}\left((1-\epsilon)^{-2}\|\bDelta_{V}\bSigma_{\star,2}\|_{\fro}+(2+\epsilon)(1-\epsilon)^{-3}\|\bDelta_{W}\bSigma_{\star,3}\|_{\fro}+3(1-\epsilon)^{-1}\|\bDelta_{\cS}\|_{\fro}\right).
\end{align*}
It is straightforward to check that $|\mfk{S}_{1}^{\ptb}|\le C_{1}\epsilon\dist^{2}(\bF_{t},\bF_{\star})$
for some absolute constant $C_{1}>1$.

\subsection{Proof of Claim~\ref{claim:S2}}
\label{proof:claim_S2}

Reuse the decomposition~\eqref{eq:TF_decomp_S} and the elementary
inequality $(a+b+c)^{2}\le3(a^{2}+b^{2}+c^{2})$ to obtain 
\begin{align*}
\mfk{S}_{2} & \le3\underbrace{\left\Vert ((\bU^{\top}\bU)^{-1}\bU^{\top}\bDelta_{U},\bI_{r_{2}},\bI_{r_{3}})\bcdot\bcS_{\star}\right\Vert _{\fro}^{2}}_{\eqqcolon \mfk{S}_{2,1}}+3\underbrace{\left\Vert ((\bU^{\top}\bU)^{-1}\bU^{\top}\bU_{\star},(\bV^{\top}\bV)^{-1}\bV^{\top}\bDelta_{V},\bI_{r_{3}})\bcdot\bcS_{\star}\right\Vert _{\fro}^{2}}_{\eqqcolon \mfk{S}_{2,2}}\\
 & \quad+3\underbrace{\left\Vert ((\bU^{\top}\bU)^{-1}\bU^{\top}\bU_{\star},(\bV^{\top}\bV)^{-1}\bV^{\top}\bV_{\star},(\bW^{\top}\bW)^{-1}\bW^{\top}\bDelta_{W})\bcdot\bcS_{\star}\right\Vert _{\fro}^{2}}_{ \eqqcolon \mfk{S}_{2,3}}.
\end{align*}
Apply the mode-$1$ matricization and Lemma~\ref{lemma:perturb_bounds}
to $\mfk{S}_{2,1}$ to see 
\begin{align*}
\mfk{S}_{2,1} & =\left\Vert (\bU^{\top}\bU)^{-1}\bU^{\top}\bDelta_{U}\cM_{1}(\bcS_{\star})\right\Vert _{\fro}^{2}\\
 & \le\|(\bU^{\top}\bU)^{-1}\|_{\op} \big\Vert (\bU^{\top}\bU)^{-1/2}\bU^{\top}\bDelta_{U}\cM_{1}(\bcS_{\star})\big\Vert _{\fro}^{2}\\
 & \le(1-\epsilon)^{-2}\big\Vert (\bU^{\top}\bU)^{-1/2}\bU^{\top}\bDelta_{U}\bSigma_{\star,1}\big\Vert _{\fro}^{2}.
\end{align*}
Similarly, apply the mode-$2$ (resp.~mode-3) matricization to $\mfk{S}_{2,2}$ (resp.~$\mfk{S}_{2,3}$) to see 
\begin{align*}
\mfk{S}_{2,2} & =\left\Vert (\bV^{\top}\bV)^{-1}\bV^{\top}\bDelta_{V}\cM_{2}(\bcS_{\star})\left(\bI_{r_{3}}\otimes\bU_{\star}^{\top}\bU(\bU^{\top}\bU)^{-1}\right)\right\Vert _{\fro}^{2}\\
 & \le\|(\bV^{\top}\bV)^{-1}\|_{\op} \big\Vert (\bV^{\top}\bV)^{-1/2}\bV^{\top}\bDelta_{V}\cM_{2}(\bcS_{\star})\big\Vert _{\fro}^{2}\|\bU(\bU^{\top}\bU)^{-1}\|_{\op}^{2}\\
 & \le(1-\epsilon)^{-4}\big\Vert (\bV^{\top}\bV)^{-1/2}\bV^{\top}\bDelta_{V}\bSigma_{\star,2}\big\Vert _{\fro}^{2},
\end{align*}
and
\begin{align*}
\mfk{S}_{2,3} & =\left\Vert (\bW^{\top}\bW)^{-1}\bW^{\top}\bDelta_{W}\cM_{3}(\bcS_{\star})\left(\bV_{\star}^{\top}\bV(\bV^{\top}\bV)^{-1}\otimes\bU_{\star}^{\top}\bU(\bU^{\top}\bU)^{-1}\right)\right\Vert _{\fro}^{2}\\
 & \le\|(\bW^{\top}\bW)^{-1}\|_{\op} \big\Vert (\bW^{\top}\bW)^{-1/2}\bW^{\top}\bDelta_{W}\cM_{3}(\bcS_{\star}) \big\Vert _{\fro}^{2}\|\bU(\bU^{\top}\bU)^{-1}\|_{\op}^{2}\|\bV(\bV^{\top}\bV)^{-1}\|_{\op}^{2}\\
 & \le(1-\epsilon)^{-6} \big\Vert (\bW^{\top}\bW)^{-1/2}\bW^{\top}\bDelta_{W}\bSigma_{\star,3}\big\Vert _{\fro}^{2}.
\end{align*}
Combine the bounds on $\mfk{S}_{2,1},\mfk{S}_{2,2},\mfk{S}_{2,3}$
to write $\mfk{S}_{2}$ as 
\begin{align*}
\mfk{S}_{2} & \le3(1-\epsilon)^{-2} \big\Vert (\bU^{\top}\bU)^{-1/2}\bU^{\top}\bDelta_{U}\bSigma_{\star,1}\big\Vert _{\fro}^{2}+3(1-\epsilon)^{-4}\big\Vert (\bV^{\top}\bV)^{-1/2}\bV^{\top}\bDelta_{V}\bSigma_{\star,2}\big\Vert _{\fro}^{2}\\
 & \qquad+3(1-\epsilon)^{-6}\big\Vert (\bW^{\top}\bW)^{-1/2}\bW^{\top}\bDelta_{W}\bSigma_{\star,3}\big\Vert _{\fro}^{2}.
\end{align*}
By symmetry, one can permute $\bDelta_{U},\bDelta_{V},\bDelta_{W}$,
and take the average to balance their coefficients and reach the conclusion
that  
\begin{align*}
\mfk{S}_{2} & \le3\left(\big\Vert (\bU^{\top}\bU)^{-1/2}\bU^{\top}\bDelta_{U}\bSigma_{\star,1}\big\Vert _{\fro}^{2}+\big\Vert (\bV^{\top}\bV)^{-1/2}\bV^{\top}\bDelta_{V}\bSigma_{\star,2}\big\Vert _{\fro}^{2}+\big\Vert (\bW^{\top}\bW)^{-1/2}\bW^{\top}\bDelta_{W}\bSigma_{\star,3}\big\Vert _{\fro}^{2}\right)+\mfk{S}_{2}^{\ptb},
\end{align*}
where the perturbation term $\mfk{S}_{2}^{\ptb}$ obeys 
\begin{align*}
\mfk{S}_{2}^{\ptb} & \le\left((1-\epsilon)^{-2}+(1-\epsilon)^{-4}+(1-\epsilon)^{-6}-3\right)\left(\|\bDelta_{U}\bSigma_{\star,1}\|_{\fro}^{2}+\|\bDelta_{V}\bSigma_{\star,2}\|_{\fro}^{2}+\|\bDelta_{W}\bSigma_{\star,3}\|_{\fro}^{2}\right).
\end{align*}
A bit simplification yields $\mfk{S}_{2}^{\ptb}\le C_{2}\epsilon\dist^{2}(\bF_{t},\bF_{\star})$.


\section{Proof for Tensor Completion}\label{sec:proof_TC}
This section is devoted to the proofs of claims related to tensor completion. 
To begin with, we state several bounds regarding the $\ell_{2,\infty}$ norm that will be repeatedly used throughout this section. 

\begin{lemma}\label{lemma:TC_perturb_bounds} Suppose that $\bcX_{\star}$ is $\mu$-incoherent, and that $\bF=(\bU,\bV,\bW,\bcS)$ satisfies $\dist(\bF,\bF_{\star})\le \epsilon\sigma_{\min}(\bcX_{\star})$ for $\epsilon <1$ and the incoherence condition~\eqref{eq:TC_cond_2inf}. Then one has the following bounds regarding the $\ell_{2,\infty}$ norm:
\begin{subequations}
\begin{align}
\sqrt{n_1}\|\bU\cM_{1}(\bcS)\|_{2,\infty} &\le (1-\epsilon)^{-2} C_B\sqrt{\mu r}\sigma_{\max}(\bcX_{\star}); \label{eq:TC_US_2inf} \\
\sqrt{n_1}\|\bU\cM_{1}(\bcS_{\star})\|_{2,\infty} = \sqrt{n_1}\|\bU\bSigma_{\star,1}\|_{2,\infty} &\le (1-\epsilon)^{-3}C_B\sqrt{\mu r}\sigma_{\max}(\bcX_{\star}); \label{eq:TC_USstar_2inf} \\
\sqrt{n_1}\|\bU\|_{2,\infty} &\le (1-\epsilon)^{-3}C_B\kappa\sqrt{\mu r}. \label{eq:TC_U_2inf}
\end{align}
\end{subequations}
By symmetry, a corresponding set of bounds hold for $\bV,\breve{\bV}$ and $\bW,\breve{\bW}$. 
\end{lemma}

\begin{proof} 
For \eqref{eq:TC_US_2inf}, we have  
\begin{align*}
\|\bU\cM_{1}(\bcS)\|_{2,\infty} &= \big\|\bU\breve{\bU}^{\top} \big(\bW(\bW^{\top}\bW)^{-1}\otimes\bV(\bV^{\top}\bV)^{-1}\big)\big\|_{2,\infty} \\ 
&\le \|\bU\breve{\bU}^{\top}\|_{2,\infty} \left\|\bW(\bW^{\top}\bW)^{-1}\right\|_{\op} \left\|\bV(\bV^{\top}\bV)^{-1}\right\|_{\op} \\
&\le \|\bU\breve{\bU}^{\top}\|_{2,\infty}(1-\epsilon)^{-2},
\end{align*}
where the first line uses \eqref{eq:breve_U_top}, the second line follows from $\|\bA\bB\|_{2,\infty} \le \|\bA\|_{2,\infty}\|\bB\|_{\op}$, and the last inequality uses \eqref{eq:perturb_Uinv_d}. This combined with condition~\eqref{eq:TC_cond_2inf} leads to the declared bound. 

Similarly for \eqref{eq:TC_USstar_2inf}, we have
\begin{align*}
\|\bU\bSigma_{\star,1}\|_{2,\infty} &= \big\|\bU\breve{\bU}^{\top} \breve{\bU}(\breve{\bU}^{\top}\breve{\bU})^{-1}\bSigma_{\star,1}\big\|_{2,\infty} \\ 
&\le \|\bU\breve{\bU}^{\top}\|_{2,\infty} \left\|\breve{\bU}(\breve{\bU}^{\top}\breve{\bU})^{-1}\bSigma_{\star,1}\right\|_{\op} \\
&\le \|\bU\breve{\bU}^{\top}\|_{2,\infty}(1-\epsilon)^{-3},
\end{align*}
where the last line follows from \eqref{eq:perturb_Rinv}.

Finally, observe that
\begin{align*}
\|\bU\bSigma_{\star,1}\|_{2,\infty} \ge \|\bU\|_{2,\infty}\sigma_{\min}(\bSigma_{\star,1}) \ge \|\bU\|_{2,\infty}\sigma_{\min}(\bcX_{\star}).
\end{align*}
Combining the above inequality with \eqref{eq:TC_USstar_2inf}, we reach the bound \eqref{eq:TC_U_2inf}. 
\end{proof}

\subsection{Proof of Lemma~\ref{lemma:scaled_proj}}

A crucial operation, which aims to preserve the desirable incoherence property with respect to the scaled distance, is the scaled projection $\bF=\cP_{B}(\bF_{+})$ defined in \eqref{eq:scaled_proj}. For the purpose of understanding, it is instructive to view $\bF$ as the solution to the following optimization problems:
\begin{align}
\begin{split}
\bU &= \argmin_{\bU}\; \big\|(\bU-{\bU}_{+})\breve{{\bU}}_{+}^{\top}\big\|_{\fro}^{2} \qquad\mbox{s.t.}\quad \sqrt{n_1}\|\bU\breve{{\bU}}_{+}^{\top}\|_{2,\infty} \le B, \\ 
\bV &=  \argmin_{\bV}\; \big\|(\bV-{\bV}_{+})\breve{{\bV}}_{+}^{\top}\big\|_{\fro}^{2} \qquad\mbox{s.t.}\quad \sqrt{n_2}\|\bV\breve{{\bV}}_{+}^{\top}\|_{2,\infty} \le B, \\ 
\bW &= \argmin_{\bW}\; \big\|(\bW-{\bW}_{+})\breve{{\bW}}_{+}^{\top}\big\|_{\fro}^{2} \qquad\mbox{s.t.}\quad \sqrt{n_3}\|\bW\breve{{\bW}}_{+}^{\top}\|_{2,\infty} \le B.
\end{split}\label{eq:scaled_proj_opt}
\end{align} 

The remaining proof follows similar arguments as \cite{tong2021accelerating}. To begin, we collect a useful claim as follows.
\begin{claim}[{\cite[Claim~5]{tong2021accelerating}}]\label{claim:nonexpansive} For vectors $\bu,\bu_{\star}\in\RR^{n}$ and $\lambda\ge\|\bu_{\star}\|_{2}/\|\bu\|_{2}$, it holds that 
\begin{align*}
\|(1\wedge\lambda)\bu-\bu_{\star}\|_{2}\le\|\bu-\bu_{\star}\|_{2}.
\end{align*}
\end{claim}

\paragraph{Proof of the non-expansive property.}
We begin with proving the non-expansive property. Denote the optimal alignment matrices between $\bF_{+}$ and $\bF_{\star}$ as $\{\bQ_{+,k}\}_{k=1,2,3}$, whose existence is guaranteed by Lemma~\ref{lemma:Q_existence}. 
Assume for now (which shall be established at the end of the proof) that for any $1 \le i_1 \le n_{1}$, we have
\begin{align} \label{eq:claim5_prereq}
\frac{B}{\sqrt{n_1}\big\|\bU_{+}(i_1,:)\breve{\bU}_{+}^{\top}\big\|_{2}} \ge \frac{ \big\|\bU_{\star}(i_1,:)\bSigma_{\star,1}\big\|_{2}}{\big\|\bU_{+}(i_1,:)\bQ_{+,1}\bSigma_{\star,1}\big\|_{2}}.
\end{align}
This taken together with Claim~\ref{claim:nonexpansive} immediately implies 
\begin{align*}
\big\|\bU(i_1,:)\bQ_{+,1}\bSigma_{\star,1} -  \bU_{\star}(i_1,:)\bSigma_{\star,1} \big\|_{2}  & \le \big\|\bU_{+}(i_1,:)\bQ_{+,1}\bSigma_{\star,1}-\bU_{\star}(i_1,:)\bSigma_{\star,1}\big\|_{2}, \qquad 1\le i_1\le n_1, \\
 \Longrightarrow \qquad \big\|(\bU\bQ_{+,1}-\bU_{\star})\bSigma_{\star,1} \big\|_{\fro} & \le \big\|(\bU_{+}\bQ_{+,1}-\bU_{\star})\bSigma_{\star,1} \big\|_{\fro}.
\end{align*}
Repeating similar arguments for the other two factors, we obtain
\begin{align*}
\big\|(\bV\bQ_{+,2}-\bV_{\star})\bSigma_{\star,2} \big\|_{\fro} &\le \big\|(\bV_{+}\bQ_{+,2}-\bV_{\star})\bSigma_{\star,2} \big\|_{\fro}, \quad
\big\|(\bW\bQ_{+,3}-\bW_{\star})\bSigma_{\star,3} \big\|_{\fro} \le \big\|(\bW_{+}\bQ_{+,3}-\bW_{\star})\bSigma_{\star,3} \big\|_{\fro}.
\end{align*}
Combining the above bounds, we have
\begin{align*}
\dist^{2}(\bF,\bF_{\star}) &\le  \left\|(\bU\bQ_{+,1}-\bU_{\star})\bSigma_{\star,1}\right\|_{\fro}^{2} + \left\|(\bV\bQ_{+,2}-\bV_{\star})\bSigma_{\star,2}\right\|_{\fro}^{2} \\
&\quad + \left\|(\bW\bQ_{+,3}-\bW_{\star})\bSigma_{\star,3}\right\|_{\fro}^{2} + \left\|(\bQ_{+,1}^{-1},\bQ_{+,2}^{-1},\bQ_{+,3}^{-1})\bcdot\bcS-\bcS_{\star}\right\|_{\fro}^2 =  \dist^{2}(\bF_{+},\bF_{\star}).
\end{align*}

\paragraph{Proof of the incoherence condition.}
Turning to the incoherence condition, it follows that for any $1 \le i_1 \le n_{1}$,  
\begin{align*}
& \big\|\bU(i_1,:)\breve{\bU}^{\top} \big\|_{2}^2 = \sum_{i_2=1}^{n_2}\sum_{i_3=1}^{n_3} \big\langle\bU(i_1,:)\cM_{1}(\bcS), \bW(i_3,:)\otimes\bV(i_2,:) \big\rangle^{2} \\
&\overset{\text{(i)}}{=} \sum_{i_2=1}^{n_2}\sum_{i_3=1}^{n_3} \big\langle\bU(i_1,:)\cM_{1}(\bcS), \bW_{+}(i_3,:)\otimes\bV_{+}(i_2,:)\big\rangle^{2} \left(1 \wedge \frac{B}{\sqrt{n_3}\|\bW_{+}(i_3,:)\breve{\bW}_{+}^{\top}\|_{2}}\right)^{2}\left(1 \wedge \frac{B}{\sqrt{n_2}\|\bV_{+}(i_2,:)\breve{\bV}_{+}^{\top}\|_{2}}\right)^{2} \\
&\overset{\text{(ii)}}{\le} \sum_{i_2=1}^{n_2}\sum_{i_3=1}^{n_3} \big\langle\bU(i_1,:)\cM_{1}(\bcS), \bW_{+}(i_3,:)\otimes\bV_{+}(i_2,:)\big\rangle^{2} \\
&\overset{\text{(iii)}}{=} \sum_{i_2=1}^{n_2}\sum_{i_3=1}^{n_3} \left(1 \wedge \frac{B}{\sqrt{n_1}\|\bU_{+}(i_1,:)\breve{\bU}_{+}^{\top}\|_{2}}\right)^{2} \big\langle\bU_{+}(i_1,:)\cM_{1}(\bcS_{+}), \bW_{+}(i_3,:)\otimes\bV_{+}(i_2,:) \big\rangle^{2} \\
&= \left(1 \wedge \frac{B}{\sqrt{n_1}\|\bU_{+}(i_1,:)\breve{\bU}_{+}^{\top}\|_{2}}\right)^{2} \big\|\bU_{+}(i_1,:)\breve{\bU}_{+}^{\top} \big\|_{2}^2 \overset{\text{(iv)}}{\le} \frac{B^2}{n_1}.
\end{align*}
Here, (i) and (iii) follow from the definition of the scaled projection \eqref{eq:scaled_proj}, (ii) and (iv) follow from the basic relations $a\wedge b \le a$ and $a\wedge b \le b$. By symmetry, one has
\begin{align*}
\sqrt{n_1}\|\bU\breve{\bU}^{\top}\|_{2,\infty} \vee \sqrt{n_2}\|\bV\breve{\bV}^{\top}\|_{2,\infty} \vee \sqrt{n_3}\|\bW\breve{\bW}^{\top}\|_{2,\infty} \le B.
\end{align*}
The proof is then finished once we prove inequality~\eqref{eq:claim5_prereq}.
 
\paragraph{Proof of \eqref{eq:claim5_prereq}.}
Under the condition $\dist(\bF_{+},\bF_{\star})\le \epsilon\sigma_{\min}(\bcX_{\star})$, invoke \eqref{eq:perturb} in Lemma~\ref{lemma:perturb_bounds} on the factor quadruple $\left(\bU_{+}\bQ_{+,1},\bV_{+}\bQ_{+,2},\bW_{+}\bQ_{+,3}, (\bQ_{+,1}^{-1},\bQ_{+,2}^{-1},\bQ_{+,3}^{-1})\bcdot \bcS_{+}\right)$ to see
\begin{align*}
\|\bV_{+}\bQ_{+,2}\|_{\op} \vee \|\bW_{+}\bQ_{+,3}\|_{\op} \vee \left\Vert \cM_{1}\left((\bQ_{+,1}^{-1},\bQ_{+,2}^{-1},\bQ_{+,3}^{-1})\bcdot \bcS_{+}\right)^{\top}\bSigma_{\star,1}^{-1}\right\Vert _{\op} \le 1+\epsilon,
\end{align*}
which further implies that
\begin{align} \label{eq:Uplus_spectral}
\big\|\breve{\bU}_{+}\bQ_{+,1}^{-\top}\bSigma_{\star,1}^{-1}\big\|_{\op} \le \|\bV_{+}\bQ_{+,2}\|_{\op}\|\bW_{+}\bQ_{+,3}\|_{\op}\left\Vert \cM_{1}\left((\bQ_{+,1}^{-1},\bQ_{+,2}^{-1},\bQ_{+,3}^{-1})\bcdot \bcS_{+}\right)^{\top}\bSigma_{\star,1}^{-1}\right\Vert _{\op} \le (1+\epsilon)^{3}.
\end{align}
For any $1 \le i_1 \le n_{1}$, one has
\begin{align*}
\big\|\bU_{+}(i_1,:)\breve{\bU}_{+}^{\top}\big\|_{2} &\le \left\|\bU_{+}(i_1,:)\bQ_{+,1}\bSigma_{\star,1}\right\|_{2} \big\|\breve{\bU}_{+}\bQ_{+,1}^{-\top}\bSigma_{\star,1}^{-1}\big\|_{\op} \\
&\le \left\|\bU_{+}(i_1,:)\bQ_{+,1}\bSigma_{\star,1}\right\|_{2} (1+\epsilon)^{3},
\end{align*}
where the second line follows from the bound~\eqref{eq:Uplus_spectral}.
In addition, the incoherence assumption of $\bcX_{\star}$ \eqref{eq:mu} implies that
\begin{align*}
\sqrt{n_1} \big\|\bU_{\star}(i_1,:)\bSigma_{\star,1} \big\|_{2} \le \sqrt{n_1} \big\|\bU_{\star}(i_1,:) \big\|_{2} \big\|\bSigma_{\star,1}\big\|_{\op} \le \sqrt{\mu r}\sigma_{\max}(\bcX_{\star}) \le B(1+\epsilon)^{-3},
\end{align*}
where the last inequality follows from the choice of $B$. Take the above two relations collectively to reach the advertised bound \eqref{eq:claim5_prereq}.

\subsection{Concentration inequalities}

We gather several useful concentration inequalities regarding the partial observation operator $\cP_{\Omega}(\cdot)$ for the Bernoulli observation model \eqref{eq:bernoulli_model}.


\begin{lemma}\label{lemma:P_Omega_tangent} Suppose that $\bcX_{\star}$ is $\mu$-incoherent, and that $pn_1n_2n_3 \gtrsim n\mu^2 r^2 \log n$. With overwhelming probability, one has
\begin{align*}
\left| \left\langle (p^{-1}\cP_{\Omega}-\cI)(\bcX_{A}),\bcX_{B} \right\rangle \right| \le C_T \sqrt{\frac{n\mu^2 r^2 \log n}{p n_1n_2n_3}}\|\bcX_{A}\|_{\fro}\|\bcX_{B}\|_{\fro}
\end{align*}
simultaneously for all tensors $\bcX_{A},\bcX_{B}\in \RR^{n_1\times n_2\times n_3}$ in the form of
\begin{align*}
\bcX_{A} &= (\bU_{A},\bV_{\star},\bW_{\star})\bcdot\bcS_{A,1} + (\bU_{\star},\bV_{A},\bW_{\star})\bcdot\bcS_{A,2} + (\bU_{\star},\bV_{\star},\bW_{A})\bcdot\bcS_{A,3}, \\
\bcX_{B} &= (\bU_{B},\bV_{\star},\bW_{\star})\bcdot\bcS_{B,1} + (\bU_{\star},\bV_{B},\bW_{\star})\bcdot\bcS_{B,2} + (\bU_{\star},\bV_{\star},\bW_{B})\bcdot\bcS_{B,3},
\end{align*}
where $\bU_{A},\bU_{B}\in\RR^{n_1\times r_1}$, $\bV_{A},\bV_{B}\in\RR^{n_2\times r_2}$, $\bW_{A},\bW_{B}\in\RR^{n_3\times r_3}$, and $\bcS_{A,k},\bcS_{B,k}\in\RR^{r_1\times r_2\times r_3}$ are arbitrary factors, and $C_T > 0$ is some universal constant.
\end{lemma}

\begin{lemma}[{\cite[Lemma~D.2]{cai2019nonconvex}}] \label{lemma:P_Omega_fixed} For any fixed $\bcX\in\RR^{n_1\times n_2 \times n_3}$, with overwhelming probability, one has
\begin{align*}
\left\|(p^{-1}\cP_{\Omega}-\cI)(\bcX)\right\|_{\op} \le C_Y  \left(p^{-1} \log^{3} n \|\bcX\|_{\infty} + \sqrt{p^{-1} \log^{5} n} \max_{k=1,2,3}\|\cM_{k}(\bcX)^{\top}\|_{2,\infty}\right),
\end{align*}
where $C_Y > 0$ is some universal constant.
\end{lemma}

\begin{lemma} \label{lemma:P_Omega_2inf} With overwhelming probability, one has
\begin{align*}
\left|\left\langle(p^{-1}\cP_{\Omega}-\cI)((\bU_{A},\bV_{A},\bW_{A})\bcdot\bcS_{A}), (\bU_{B},\bV_{B},\bW_{B})\bcdot\bcS_{B}\right\rangle\right|\le C_Y \left(p^{-1}\log^3 n + \sqrt{p^{-1}n \log^5 n}\right) \mfk{N},
\end{align*}
simultaneously for all tensors $(\bU_{A},\bV_{A},\bW_{A})\bcdot\bcS_{A}$ and $(\bU_{B},\bV_{B},\bW_{B})\bcdot\bcS_{B}$, where the quantity $\mfk{N}$ obeys
\begin{align*}
 \mfk{N} \le & \big(\|\bU_{A}\cM_{1}(\bcS_{A})\|_{2,\infty}\|\bU_{B}\cM_{1}(\bcS_{B})\|_{\fro} \wedge \|\bU_{A}\cM_{1}(\bcS_{A})\|_{\fro}\|\bU_{B}\cM_{1}(\bcS_{B})\|_{2,\infty}\big) \\
 & \qquad \big(\|\bV_{A}\|_{2,\infty}\|\bV_{B}\|_{\fro} \wedge \|\bV_{A}\|_{\fro}\|\bV_{B}\|_{2,\infty}\big) \big(\|\bW_{A}\|_{2,\infty}\|\bW_{B}\|_{\fro} \wedge \|\bW_{A}\|_{\fro}\|\bW_{B}\|_{2,\infty}\big).
\end{align*}
By symmetry, the above bound continues to hold if permuting the occurrences of $\bU$, $\bV$, and $\bW$.
\end{lemma}

\begin{lemma}[{\cite[Lemma~3.24]{chen2021spectral},\cite[Lemma~1]{cai2019subspace}}] \label{lemma:P_Omega_quad} For any fixed $\bcX\in\RR^{n_1\times n_2\times n_3}$, $k=1,2,3$, with overwhelming probability, one has
\begin{align*}
 & \left\Vert \Poffdiag\left(p^{-2}\cM_{k}(\cP_{\Omega}(\bcX))\cM_{k}(\cP_{\Omega}(\bcX))^{\top}\right) - \cM_{k}(\bcX)\cM_{k}(\bcX)^{\top}\right\Vert _{\op} \\
 &\qquad \le C_M\left(p^{-1}\sqrt{\log n}\|\cM_{k}(\bcX)\|_{2,\infty}\|\cM_{k}(\bcX)^{\top}\|_{2,\infty} + \sqrt{p^{-1}\log n}\;\sigma_{\max}(\cM_{k}(\bcX))\|\cM_{k}(\bcX)^{\top}\|_{2,\infty} \right)\\
 &\qquad\qquad + C_M\left(p^{-1}\log n \|\bcX\|_{\infty} + \sqrt{p^{-1}\log n}\|\cM_{k}(\bcX)^{\top}\|_{2,\infty}\right)^{2}\log n + \|\cM_{k}(\bcX)\|_{2,\infty}^2, 
\end{align*}
where $C_M > 0$ is some universal constant.
\end{lemma}

\subsubsection{Proof of Lemma~\ref{lemma:P_Omega_tangent} }

This lemma is essentially \cite[Lemma~5]{yuan2016tensor} under the Bernoulli observation model. Here, we provide a simpler proof based on the matrix Bernstein inequality. Let $\bcE_{i_1,i_2,i_3}$ be the tensor with only the $(i_1,i_2,i_3)$-th entry as $1$ and all the other entries as $0$, and let $\delta_{i_1,i_2,i_3}\sim\text{Bernoulli}(p)$ be an i.i.d.~Bernoulli random variable for $1\le i_k \le n_k$, $k=1,2,3$.  Define an operator $\cP_{T}: \RR^{n_1\times n_2 \times n_3} \mapsto \RR^{n_1\times n_2 \times n_3}$ as
\begin{align*}
\cP_{T}(\bcX) = (\bI_{n_1},\bV_{\star}\bV_{\star}^{\top},\bW_{\star}\bW_{\star}^{\top})\bcdot\bcX + (\bU_{\star}\bU_{\star}^{\top},\bV_{\star\perp}\bV_{\star\perp}^{\top},\bW_{\star}\bW_{\star}^{\top})\bcdot\bcX + (\bU_{\star}\bU_{\star}^{\top},\bV_{\star}\bV_{\star}^{\top},\bW_{\star\perp}\bW_{\star\perp}^{\top})\bcdot\bcX,
\end{align*}
where $\bV_{\star\perp},\bW_{\star\perp}$ denote the orthogonal complements of $\bV_{\star},\bW_{\star}$. It is straightforward to verify that $\cP_{T}(\cdot)$ defines a projection, and that
\begin{align*}
\bcX_{A}  & = (\bU_{A},\bV_{\star},\bW_{\star})\bcdot\bcS_{A,1}  +   (\bU_{\star},\bV_{A},\bW_{\star})\bcdot\bcS_{A,2}  +  (\bU_{\star},\bV_{\star},\bW_{A})\bcdot\bcS_{A,3}    \\
&=\cP_{T}( (\bU_{A},\bV_{\star},\bW_{\star})\bcdot\bcS_{A,1} ) + \cP_{T}( (\bU_{\star},\bV_{A},\bW_{\star})\bcdot\bcS_{A,2} ) + \cP_{T}( (\bU_{\star},\bV_{\star},\bW_{A})\bcdot\bcS_{A,3}) \\
& =\cP_{T}(\bcX_{A}) = \sum_{i_1,i_2,i_3} \langle\cP_{T}(\bcX_{A}), \bcE_{i_1,i_2,i_3}\rangle \bcE_{i_1,i_2,i_3} = \sum_{i_1,i_2,i_3} \langle \bcX_{A}, \cP_{T}(\bcE_{i_1,i_2,i_3})\rangle \bcE_{i_1,i_2,i_3}.
\end{align*}
A similar expression holds for $\bcX_{B}$. Hence, we have
\begin{align*}
& \left| \left\langle (p^{-1}\cP_{\Omega}-\cI)(\bcX_{A}), \bcX_{B} \right\rangle \right| = \left| \sum_{i_1,i_2,i_3} \left(p^{-1}\delta_{i_1,i_2,i_3}-1 \right) \left\langle \bcX_{A}, \cP_{T}(\bcE_{i_1,i_2,i_3}) \right\rangle \left\langle \bcX_{B}, \cP_{T}(\bcE_{i_1,i_2,i_3}) \right\rangle \right |\\
&\qquad\qquad= \left| \Big\langle \vc(\bcX_{A}), \sum_{i_1,i_2,i_3} (p^{-1}\delta_{i_1,i_2,i_3}-1) \vc\left(\cP_T(\bcE_{i_1,i_2,i_3})\right)\vc\left(\cP_T(\bcE_{i_1,i_2,i_3})\right)^{\top} \vc(\bcX_{B})\Big\rangle \right | \\
&\qquad\qquad \le \|\bcX_{A}\|_{\fro}\|\bcX_{B}\|_{\fro}\left\|\sum_{i_1,i_2,i_3} (p^{-1}\delta_{i_1,i_2,i_3}-1) \vc\left(\cP_T(\bcE_{i_1,i_2,i_3})\right)\vc\left(\cP_T(\bcE_{i_1,i_2,i_3})\right)^{\top}\right\|_{\op}.
\end{align*}
Therefore it suffices to bound the last term in the above inequality, which we resort to the matrix Bernstein inequality: with overwhelming probability, one has
\begin{align} \label{eq:matrix_bernstein_tbd}
\left\|\sum_{i_1,i_2,i_3} (p^{-1}\delta_{i_1,i_2,i_3}-1) \vc\left(\cP_T(\bcE_{i_1,i_2,i_3})\right)\vc\left(\cP_T(\bcE_{i_1,i_2,i_3})\right)^{\top}\right\|_{\op} & \lesssim \left(\frac{n \mu^2 r^2 \log n}{p n_1n_2n_3} + \sqrt{\frac{n \mu^2 r^2 \log n}{p n_1n_2n_3}}\right)  \\
& \lesssim \sqrt{\frac{n\mu^2 r^2 \log n}{p n_1n_2n_3}}, \nonumber
\end{align}
where the second line holds as long as $pn_1n_2n_3\gtrsim n\mu^2r^2 \log n$. Plugging the above bound (which will be proved at the end) in the previous one, we immediately arrive at the desired result:
\begin{align*}
\left|\left\langle(p^{-1}\cP_{\Omega}-\cI)(\bcX_{A}),\bcX_{B}\right\rangle\right| \lesssim \sqrt{\frac{n\mu^2 r^2 \log n}{p n_1n_2n_3}}\|\bcX_{A}\|_{\fro}\|\bcX_{B}\|_{\fro}.
\end{align*}
 
\paragraph{Proof of \eqref{eq:matrix_bernstein_tbd}.} By standard matrix Bernstein inequality, we have
\begin{align*}
\left\|\sum_{i_1,i_2,i_3} (p^{-1}\delta_{i_1,i_2,i_3}-1) \vc\left(\cP_T(\bcE_{i_1,i_2,i_3})\right)\vc\left(\cP_T(\bcE_{i_1,i_2,i_3})\right)^{\top}\right\|_{\op} \lesssim L \log n + \sigma\sqrt{\log n},
\end{align*}
where 
\begin{align*}
L &= \max_{i_1,i_2,i_3} \left\|(p^{-1}\delta_{i_1,i_2,i_3}-1) \vc\left(\cP_T(\bcE_{i_1,i_2,i_3})\right)\vc\left(\cP_T(\bcE_{i_1,i_2,i_3})\right)^{\top}\right\|_{\op},\\
\sigma^2 &= \left\|\sum_{i_1,i_2,i_3}\EE(p^{-1}\delta_{i_1,i_2,i_3}-1)^{2} \vc\left(\cP_T(\bcE_{i_1,i_2,i_3})\right)\vc\left(\cP_T(\bcE_{i_1,i_2,i_3})\right)^{\top}\vc\left(\cP_T(\bcE_{i_1,i_2,i_3})\right)\vc\left(\cP_T(\bcE_{i_1,i_2,i_3})\right)^{\top}\right\|_{\op} .
\end{align*} 
\begin{itemize}
\item Here, $L$ obeys
\begin{align*}
L &= \max_{i_1,i_2,i_3} \left\|(p^{-1}\delta_{i_1,i_2,i_3}-1) \vc\left(\cP_T(\bcE_{i_1,i_2,i_3})\right)\vc\left(\cP_T(\bcE_{i_1,i_2,i_3})\right)^{\top}\right\|_{\op} \le p^{-1} \max_{i_1,i_2,i_3} \left\|\cP_T(\bcE_{i_1,i_2,i_3})\right\|_{\fro}^2,
\end{align*}
where the last inequality uses $|(p^{-1}\delta_{i_1,i_2,i_3}-1)| \le p^{-1}$. To proceed, first notice that the three terms in $\cP_T(\bcE_{i_1,i_2,i_3})$ are mutually orthogonal, which allows 
\begin{align*}
\left\|\cP_T(\bcE_{i_1,i_2,i_3})\right\|_{\fro}^2 &= \left\|(\bI_{n_1},\bV_{\star}\bV_{\star}^{\top},\bW_{\star}\bW_{\star}^{\top})\bcdot\bcE_{i_1,i_2,i_3}\right\|_{\fro}^{2} + \left\|(\bU_{\star}\bU_{\star}^{\top},\bV_{\star\perp}\bV_{\star\perp}^{\top},\bW_{\star}\bW_{\star}^{\top})\bcdot\bcE_{i_1,i_2,i_3}\right\|_{\fro}^{2} \\
&\quad + \left\|(\bU_{\star}\bU_{\star}^{\top},\bV_{\star}\bV_{\star}^{\top},\bW_{\star\perp}\bW_{\star\perp}^{\top})\bcdot\bcE_{i_1,i_2,i_3}\right\|_{\fro}^{2}.
\end{align*}
Since $\bU_{\star},\bV_{\star},\bW_{\star}$ have orthonormal columns, it is straightforward to see
\begin{align*}
\left\|(\bI_{n_1},\bV_{\star}\bV_{\star}^{\top},\bW_{\star}\bW_{\star}^{\top})\bcdot\bcE_{i_1,i_2,i_3}\right\|_{\fro}^{2} &= \left\|\bI_{n_1}(i_1,:)\right\|_{2}^{2} \left\|\bV_{\star}(i_2,:)\bV_{\star}^{\top}\right\|_{2}^{2}\left\|\bW_{\star}(i_3,:)\bW_{\star}^{\top}\right\|_{2}^{2}  \\
 & \le \|\bV_{\star}\|_{2,\infty}^{2}\|\bW_{\star}\|_{2,\infty}^{2}; \\ 
\left\|(\bU_{\star}\bU_{\star}^{\top},\bV_{\star\perp}\bV_{\star\perp}^{\top},\bW_{\star}\bW_{\star}^{\top})\bcdot\bcE_{i_1,i_2,i_3}\right\|_{\fro}^{2} &= \left\|\bU_{\star}(i_1,:)\bU_{\star}^{\top}\right\|_{2}^{2} \left\|\bV_{\star\perp}(i_2,:)\bV_{\star\perp}^{\top}\right\|_{2}^{2}\left\|\bW_{\star}(i_3,:)\bW_{\star}^{\top}\right\|_{2}^{2} \\
 & \le \|\bU_{\star}\|_{2,\infty}^{2}\|\bW_{\star}\|_{2,\infty}^{2}; \\   
\left\|(\bU_{\star}\bU_{\star}^{\top},\bV_{\star}\bV_{\star}^{\top},\bW_{\star\perp}\bW_{\star\perp}^{\top})\bcdot\bcE_{i_1,i_2,i_3}\right\|_{\fro}^{2} &= \left\|\bU_{\star}(i_1,:)\bU_{\star}^{\top}\right\|_{2}^{2} \left\|\bV_{\star}(i_2,:)\bV_{\star}^{\top}\right\|_{2}^{2}\left\|\bW_{\star\perp}(i_3,:)\bW_{\star\perp}^{\top}\right\|_{2}^{2} \\
 & \le \|\bU_{\star}\|_{2,\infty}^{2}\|\bV_{\star}\|_{2,\infty}^{2}.
\end{align*}
Finally use the definition of incoherence (cf.~Definition~\ref{def:mu}) to conclude
\begin{align*}
L&\le p^{-1}\left(\|\bV_{\star}\|_{2,\infty}^{2}\|\bW_{\star}\|_{2,\infty}^{2} + \|\bU_{\star}\|_{2,\infty}^{2}\|\bW_{\star}\|_{2,\infty}^{2} + \|\bU_{\star}\|_{2,\infty}^{2}\|\bV_{\star}\|_{2,\infty}^{2}\right) \le \frac{3 n\mu^2 r^2}{p n_1n_2n_3}.
\end{align*}
\item In addition, $\sigma^2$ obeys
\begin{align*}
\sigma^2 &\le p^{-1}\max_{i_1,i_2,i_3} \left\|\cP_T(\bcE_{i_1,i_2,i_3})\right\|_{\fro}^2 \left\|\sum_{i_1,i_2,i_3} \vc\left(\cP_T(\bcE_{i_1,i_2,i_3})\right)\vc\left(\cP_T(\bcE_{i_1,i_2,i_3})\right)^{\top} \right\|_{\op} \le \frac{3n\mu^2 r^2}{p n_1n_2n_3},
\end{align*}
where we have used the variational representation to conclude
\begin{align*}
\left\|\sum_{i_1,i_2,i_3} \vc\left(\cP_T(\bcE_{i_1,i_2,i_3})\right)\vc\left(\cP_T(\bcE_{i_1,i_2,i_3})\right)^{\top} \right\|_{\op} &= \sup_{\widetilde{\bcX}: \|\widetilde{\bcX}\|_{\fro} \le 1} \sum_{i_1,i_2,i_3}\langle\widetilde{\bcX}, \cP_{T}(\bcE_{i_1,i_2,i_3})\rangle^2 \\
&= \sup_{\widetilde{\bcX}: \|\widetilde{\bcX}\|_{\fro} \le 1} \|\cP_{T}(\widetilde{\bcX})\|_{\fro}^{2} \le 1.
\end{align*}
\end{itemize}
Plugging the expressions of $L$ and $\sigma$ leads to the advertised bound \eqref{eq:matrix_bernstein_tbd}. 

\subsubsection{Proof of Lemma~\ref{lemma:P_Omega_2inf} }

This lemma generalizes \cite[Lemma~8]{chen2019model} to the tensor setting, which is a powerful tool in the analysis of matrix completion \cite{chen2019nonconvex,tong2021accelerating}. 
We begin by decomposing $(\bU_{A},\bV_{A},\bW_{A})\bcdot\bcS_{A}$ into a sum of $r_2r_3$ rank-$\one$ tensors:
\begin{align*}
(\bU_{A},\bV_{A},\bW_{A})\bcdot\bcS_{A} &= \sum_{a_{2}=1}^{r_2}\sum_{a_{3}=1}^{r_3} (\bu_{a_{2},a_{3}}, \bv_{a_{2}}, \bw_{a_{3}})\bcdot 1,
\end{align*}
where we denote the column vectors $\bu_{a_2,a_3}\coloneqq [\bU_{A}\cM_{1}(\bcS_{A})](:,(r_3-1)a_2 + a_3)$, $\bv_{a_2} \coloneqq \bV_{A}(:,a_2)$, and $\bw_{a_3}\coloneqq\bW_{A}(:,a_3)$ for notational convenience. Similarly, we can decompose $(\bU_{B},\bV_{B},\bW_{B})\bcdot\bcS_{B}$ as
\begin{align*}
(\bU_{B},\bV_{B},\bW_{B})\bcdot\bcS_{B} = \sum_{b_{2}=1}^{r_2}\sum_{b_{3}=1}^{r_3} (\bu_{b_{2},b_{3}}, \bv_{b_{2}}, \bw_{b_{3}})\bcdot 1,
\end{align*}
with $\bu_{b_2,b_3}$, $\bv_{b_2}$ and $\bw_{b_3}$ defined analogously. We further denote $\bcJ \in \RR^{n_1\times n_2\times n_3}$ as the tensor with all-one entries, i.e.~$\bcJ(i_1,i_2,i_3)=1$ for all $1\le i_k \le n_k$, $k=1,2,3$.  With these preparation in hand, by the triangle inequality we have
\begin{align*}
& \left|\left\langle(p^{-1}\cP_{\Omega}-\cI)((\bU_{A},\bV_{A}, \bW_{A})\bcdot\bcS_{A}), (\bU_{B},\bV_{B},\bW_{B})\bcdot\bcS_{B}\right\rangle\right| \\
&\qquad \le  \sum_{a_{2},b_{2}=1}^{r_2}\sum_{a_{3},b_{3}=1}^{r_3}\left|\left\langle(p^{-1}\cP_{\Omega}-\cI)((\bu_{a_2,a_3},\bv_{a_2}, \bw_{a_3})\bcdot 1), (\bu_{b_2,b_3},\bv_{b_2},\bw_{b_3})\bcdot 1\right\rangle\right|  \\
& \qquad = \sum_{a_{2},b_{2}=1}^{r_2}\sum_{a_{3},b_{3}=1}^{r_3} \left|\left\langle (p^{-1}\cP_{\Omega}-\cI)(\bcJ), (\bu_{a_2,a_3}\odot\bu_{b_2,b_3}, \bv_{a_2}\odot\bv_{b_2},\bw_{a_3}\odot\bw_{b_3}) \bcdot 1 \right\rangle\right|  \\
& \qquad \le  \sum_{a_{2},b_{2}=1}^{r_2}\sum_{a_{3},b_{3}=1}^{r_3} \|(p^{-1}\cP_{\Omega}-\cI)(\bcJ)\|_{\op} \|\bu_{a_2,a_3}\odot\bu_{b_2,b_3}\|_{2} \|\bv_{a_2}\odot\bv_{b_2}\|_{2} \|\bw_{a_3}\odot\bw_{b_3}\|_{2} \\
& \qquad = \|(p^{-1}\cP_{\Omega}-\cI)(\bcJ)\|_{\op} \mfk{N} ,
\end{align*}
where $\odot$ denotes the Hadamard (entrywise) product, and
$$ \mfk{N} \coloneqq \sum_{a_{2},b_{2}=1}^{r_2}\sum_{a_{3},b_{3}=1}^{r_3}\|\bu_{a_2,a_3}\odot\bu_{b_2,b_3}\|_{2} \|\bv_{a_2}\odot\bv_{b_2}\|_{2} \|\bw_{a_3}\odot\bw_{b_3}\|_{2}.$$
Therefore, it boils down to controlling $\|(p^{-1}\cP_{\Omega}-\cI)(\bcJ)\|_{\op} $ and $\mfk{N}$.
\begin{itemize}
\item Regarding $\|(p^{-1}\cP_{\Omega}-\cI)(\bcJ)\|_{\op} $, Lemma~\ref{lemma:P_Omega_fixed} tells that, with overwhelming probability, it is bounded by 
\begin{align*}
\|(p^{-1}\cP_{\Omega}-\cI)(\bcJ)\|_{\op} \le C_Y \left(p^{-1}\log^3 n + \sqrt{p^{-1}n \log^5 n}\right),
\end{align*}
where we use the fact $\|\bcJ\|_{\infty}=1$ and $\max_{k=1,2,3}\| \mathcal{M}_k(\bcJ)^{\top}\|_{2,\infty} \le \sqrt{n}$. 

\item Turning to $\mfk{N}$, applying the Cauchy-Schwarz inequality  we have
 \begin{align*}
\mfk{N} &\le \sqrt{\sum_{a_{2},b_{2}=1}^{r_2}\sum_{a_{3},b_{3}=1}^{r_3} \|\bu_{a_2,a_3}\odot\bu_{b_2,b_3}\|_{2}^2} \sqrt{\sum_{a_2,b_2=1}^{r_2} \|\bv_{a_2}\odot\bv_{b_2}\|_{2}^2\sum_{a_3,b_3=1}^{r_3}\|\bw_{a_3}\odot\bw_{b_3}\|_{2}^2}\\
&= \sqrt{\sum_{i_1=1}^{n_1} \|\bU_{A}(i_1,:)\cM_{1}(\bcS_{A})\|_{2}^2 \|\bU_{B}(i_1,:)\cM_{1}(\bcS_{B})\|_{2}^2}  \\
& \qquad\qquad\qquad \sqrt{\sum_{i_2=1}^{n_2} \|\bV_{A}(i_2,:)\|_{2}^2 \|\bV_{B}(i_2,:)\|_{2}^2} \sqrt{\sum_{i_3=1}^{n_3} \|\bW_{A}(i_3,:)\|_{2}^2 \|\bW_{B}(i_3,:)\|_{2}^2} \\
&\le \big(\|\bU_{A}\cM_{1}(\bcS_{A})\|_{2,\infty}\|\bU_{B}\cM_{1}(\bcS_{B})\|_{\fro} \wedge \|\bU_{A}\cM_{1}(\bcS_{A})\|_{\fro}\|\bU_{B}\cM_{1}(\bcS_{B})\|_{2,\infty} \big) \\
&\qquad\qquad \big(\|\bV_{A}\|_{2,\infty}\|\bV_{B}\|_{\fro} \wedge \|\bV_{A}\|_{\fro}\|\bV_{B}\|_{2,\infty} \big) \big(\|\bW_{A}\|_{2,\infty}\|\bW_{B}\|_{\fro} \wedge \|\bW_{A}\|_{\fro}\|\bW_{B}\|_{2,\infty} \big).
\end{align*}
\end{itemize}
The proof is complete by combining the above two bounds.

\subsection{Proof of spectral initialization (Lemma~\ref{lemma:init_TC})}
In view of Lemma~\ref{lemma:Procrustes}, we start by relating $\dist(\bF_{+},\bF_{\star})$ to $\|(\bU_{+},\bV_{+},\bW_{+})\bcdot\bcS_{+}-\bcX_{\star}\|_{\fro}$ as
\begin{align*}
\dist(\bF_{+},\bF_{\star}) \le (\sqrt{2}+1)^{3/2}\left\|(\bU_{+},\bV_{+},\bW_{+})\bcdot\bcS_{+}-\bcX_{\star}\right\|_{\fro}.
\end{align*}
With this bound in mind, it suffices to control $\left\Vert (\bU_{+},\bV_{+},\bW_{+})\bcdot\bcS_{+}-\bcX_{\star}\right\Vert _{\fro}$. 
To proceed, define $\bP_{U}\coloneqq\bU_{+}\bU_{+}^{\top}$ as the projection matrix onto the column space of $\bU_{+}$, $\bP_{U_\perp}\coloneqq\bI_{n_{1}}-\bP_{U}$ as its orthogonal complement, and define $\bP_{V},\bP_{V_\perp}, \bP_{W},\bP_{W_\perp}$ analogously. We have the decomposition
\begin{align*}
\bcX_{\star} &= (\bP_{U},\bP_{V},\bP_{W})\bcdot\bcX_{\star} + (\bP_{U_\perp},\bP_{V},\bP_{W})\bcdot\bcX_{\star} + (\bI_{n_1}, \bP_{V_\perp},\bP_{W})\bcdot\bcX_{\star} + (\bI_{n_1}, \bI_{n_2},\bP_{W_\perp})\bcdot\bcX_{\star}.
\end{align*}
Expand the following squared norm and use that the four terms are mutually orthogonal to see
\begin{align}
& \left\|(\bU_{+},\bV_{+},\bW_{+})\bcdot\bcS_{+}-\bcX_{\star}\right\|_{\fro}^{2} = \left\|(\bP_{U},\bP_{V},\bP_{W})\bcdot(p^{-1}\bcY)-\bcX_{\star}\right\|_{\fro}^{2} \nonumber\\
&\quad = \big\|(\bP_{U},\bP_{V},\bP_{W})\bcdot(p^{-1}\bcY-\bcX_{\star}) - (\bP_{U_\perp},\bP_{V},\bP_{W})\bcdot\bcX_{\star} - (\bI_{n_1},\bP_{V_\perp},\bP_{W})\bcdot\bcX_{\star} - (\bI_{n_1},\bI_{n_2},\bP_{W_\perp})\bcdot\bcX_{\star}\big\|_{\fro}^{2} \nonumber\\
&\quad = \left\|(\bP_{U},\bP_{V},\bP_{W})\bcdot(p^{-1}\bcY-\bcX_{\star})\right\|_{\fro}^{2} + \left\|(\bP_{U_\perp},\bP_{V},\bP_{W})\bcdot\bcX_{\star}\right\|_{\fro}^{2} + \left\|(\bI_{n_1}, \bP_{V_\perp},\bP_{W})\bcdot\bcX_{\star}\right\|_{\fro}^{2} \nonumber \\
&\quad\qquad + \left\|(\bI_{n_1},\bI_{n_2},\bP_{W_\perp})\bcdot\bcX_{\star}\right\|_{\fro}^{2} \nonumber\\
&\quad \le \left\|(\bP_{U},\bP_{V},\bP_{W})\bcdot(p^{-1}\bcY-\bcX_{\star})\right\|_{\fro}^{2} + \left\|\bP_{U_\perp}\cM_{1}(\bcX_{\star})\right\|_{\fro}^{2} + \left\|\bP_{V_\perp}\cM_{2}(\bcX_{\star})\right\|_{\fro}^{2} + \left\|\bP_{W_\perp}\cM_{3}(\bcX_{\star})\right\|_{\fro}^{2}.\label{eq:TC_init_expand}
\end{align}
We next control the terms in \eqref{eq:TC_init_expand} one by one.

\paragraph{Bounding $\left\|(\bP_{U},\bP_{V},\bP_{W})\bcdot(\bcY-\bcX_{\star})\right\|_{\fro}$.} For the first term in \eqref{eq:TC_init_expand}, since $(\bP_{U},\bP_{V},\bP_{W})\bcdot(p^{-1}\bcY-\bcX_{\star})$ has a multilinear rank of at most $\br$, applying the relation \eqref{eq:tensor_spec_frob} leads to
\begin{align*}
\left\|(\bP_{U},\bP_{V},\bP_{W})\bcdot(p^{-1}\bcY-\bcX_{\star})\right\|_{\fro} \le r \left\|(\bP_{U},\bP_{V},\bP_{W})\bcdot(p^{-1}\bcY-\bcX_{\star})\right\|_{\op} \le r \left\|(p^{-1}\cP_{\Omega}-\cI)(\bcX_{\star})\right\|_{\op}.
\end{align*}
Therefore, it comes down to control $\left\|(p^{-1}\cP_{\Omega}-\cI)(\bcX_{\star})\right\|_{\op}$. Lemma~\ref{lemma:P_Omega_fixed} tells with overwhelming probability that
\begin{align*}
\left\|(p^{-1}\cP_{\Omega}-\cI)(\bcX_{\star})\right\|_{\op}  & \lesssim \left(p^{-1} \log^{3} n \|\bcX_{\star}\|_{\infty} + \sqrt{p^{-1} \log^{5} n} \max_{k=1,2,3}\|\cM_{k}(\bcX_{\star})^{\top}\|_{2,\infty}\right) \\
& \lesssim \left(\frac{\mu^{3/2} r^{3/2} \log^{3} n }{p\sqrt{n_1n_2n_3}} + \sqrt{\frac{n\mu^2 r^2\log^{5} n}{p n_1n_2n_3}}\right)\sigma_{\max}(\bX_{\star}),
\end{align*}
where the second line follows from the following relations in view of the incoherence property of $\bcX_{\star}$ (cf.~Definition~\ref{def:mu}):
\begin{align}
\begin{split}
\|\bcX_{\star}\|_{\infty} &\le \sigma_{\max}(\bcX_{\star})\|\bU_{\star}\|_{2,\infty}\|\bV_{\star}\|_{2,\infty}\|\bW_{\star}\|_{2,\infty}  \le \sigma_{\max}(\bcX_{\star})  \sqrt{\frac{\mu^3 r^{3}}{n_1n_2n_3}}; \\ 
\|\cM_{1}(\bcX_{\star})^{\top}\|_{2,\infty} & \le \|\bU_{\star}\cM_{1}(\bcS_{\star})\|_{\op}\|\bW_{\star}\|_{2,\infty}\|\bV_{\star}\|_{2,\infty} \le \sigma_{\max}(\bcX_{\star}) \sqrt{\frac{\mu^2 r^2}{n_2n_3}} ; \\
\|\cM_{2}(\bcX_{\star})^{\top}\|_{2,\infty} & \le \|\bV_{\star}\cM_{2}(\bcS_{\star})\|_{\op}\|\bW_{\star}\|_{2,\infty}\|\bU_{\star}\|_{2,\infty} \le \sigma_{\max}(\bcX_{\star}) \sqrt{\frac{\mu^2 r^2}{n_1n_3}}; \\
\|\cM_{3}(\bcX_{\star})^{\top}\|_{2,\infty} & \le \|\bW_{\star}\cM_{3}(\bcS_{\star})\|_{\op}\|\bV_{\star}\|_{2,\infty}\|\bU_{\star}\|_{2,\infty} \le \sigma_{\max}(\bcX_{\star}) \sqrt{\frac{\mu^2 r^2}{n_1n_2}}.
\end{split}\label{eq:X_star_incoherence}
\end{align}
In total, the first term in \eqref{eq:TC_init_expand} is bounded by
\begin{align*}
\left\|(\bP_{U},\bP_{V},\bP_{W})\bcdot(p^{-1}\bcY-\bcX_{\star})\right\|_{\fro} &\lesssim \left(\frac{\mu^{3/2} r^{3/2} \log^{3} n}{p\sqrt{n_1n_2n_3}} + \sqrt{\frac{n\mu^2 r^2\log^{5} n}{p n_1n_2n_3}}\right) r\kappa\sigma_{\min}(\bX_{\star}).
\end{align*}

\paragraph{Bounding $\left\|\bP_{U_\perp}\cM_{1}(\bcX_{\star})\right\|_{\fro}$. } For the second term in \eqref{eq:TC_init_expand}, first bound it by 
\begin{align*}
\left\|\bP_{U_\perp}\cM_{1}(\bcX_{\star})\right\|_{\fro} \le \frac{\sqrt{r_1}}{\sigma_{\min}(\bcX_{\star})} \left\|\bP_{U_\perp}\cM_{1}(\bcX_{\star})\cM_{1}(\bcX_{\star})^{\top}\right\|_{\op},
\end{align*}
where we use the facts that $\bP_{U_\perp}\cM_{1}(\bcX_{\star})$ has rank at most $r_1$ and $\|\bA\bB\|_{\op}\ge\|\bA\|_{\op}\sigma_{\min}(\bB)$. For notation simplicity, we abbreviate
\begin{align*}
\bG\coloneqq \Poffdiag(p^{-2}\cM_{1}(\bcY)\cM_{1}(\bcY)^{\top}),\quad\mbox{and}\quad\bG_{\star}\coloneqq \cM_{1}(\bcX_{\star})\cM_{1}(\bcX_{\star})^{\top}.
\end{align*}
Invoke Lemma~\ref{lemma:P_Omega_quad} together with incoherence conditions \eqref{eq:X_star_incoherence} as well as 
\begin{align*}
\|\cM_{1}(\bcX_{\star})\|_{2,\infty} \le \|\bU_{\star}\|_{2,\infty}\left\|\cM_{1}(\bcS_{\star})(\bW_{\star}\otimes\bV_{\star})^{\top}\right\|_{\op} \le \sigma_{\max}(\bcX_{\star})\sqrt{\frac{\mu r_1}{n_1}}
\end{align*}
to conclude with overwhelming probability that
\begin{align*}
\left\|\bG - \bG_{\star}\right\|_{\op} \lesssim \left(\frac{\mu^{3/2} r^{3/2} \sqrt{\log n}}{p\sqrt{n_1n_2n_3}} + \sqrt{\frac{n\mu^2 r^{2}\log n}{p n_1n_2n_3}} + \frac{\mu^{3} r^{3} \log^{3} n}{p^2 n_1n_2n_3} + \frac{n\mu^2 r^{2}\log^{2} n}{p n_1n_2n_3} + \frac{\mu r_1}{n_1} \right) \sigma_{\max}^{2}(\bcX_{\star}).
\end{align*}
Under the conditions $n_1\gtrsim\epsilon_{0}^{-1}\mu r_1^{3/2}\kappa^2$ and 
\begin{align*}
pn_1n_2n_3 \gtrsim \epsilon_{0}^{-1}\sqrt{n_1n_2n_3}\mu^{3/2}r^{5/2}\kappa^{2}\log^{3} n + \epsilon_{0}^{-2}n\mu^{2} r^{4}\kappa^{4}\log^{5} n
\end{align*}
for some small constant $\epsilon_0 > 0$, we have $\|\bG-\bG_{\star}\|_{\op}\le \epsilon_{0}\sigma_{\min}^{2}(\bcX_{\star})$, which implies that $\bG$ is positive semi-definite, and therefore $\|\bP_{U_\perp}\bG\|_{\op}=\sigma_{r_1+1}(\bG)$. By the triangle inequality, we obtain
\begin{align*}
\left\|\bP_{U_\perp}\bG_{\star}\right\|_{\op} &\le \left\|\bP_{U_\perp}\left(\bG - \bG_{\star}\right)\right\|_{\op} + \left\|\bP_{U_\perp}\bG\right\|_{\op}  \le \left\|\bG - \bG_{\star}\right\|_{\op} + \sigma_{r_1+1}\left(\bG\right) \\
&\le \left\|\bG - \bG_{\star}\right\|_{\op} + \sigma_{r_1+1}\left(\bG_{\star}\right) + \left\|\bG - \bG_{\star}\right\|_{\op} = 2\left\|\bG - \bG_{\star}\right\|_{\op},
\end{align*}
where the second line follows from Weyl's inequality and that $\bG_{\star}$ has rank $r_1$. In total, the second term of~\eqref{eq:TC_init_expand} is bounded by
\begin{align*}
\left\|\bP_{U_\perp}\cM_{1}(\bcX_{\star})\right\|_{\fro} &\le \frac{2\sqrt{r_1}}{\sigma_{\min}(\bcX_{\star})}\|\bG-\bG_{\star}\|_{\op} \\ 
 & \lesssim \left(\frac{\mu^{3/2} r^{2} \sqrt{\log n}}{p\sqrt{n_1n_2n_3}} + \sqrt{\frac{n\mu^2 r^{3}\log n}{p n_1n_2n_3}} + \frac{\mu^{3} r^{7/2} \log^{3} n}{p^2 n_1n_2n_3} + \frac{n\mu^2 r^{5/2} \log^{2} n}{p n_1n_2n_3} + \frac{\mu r_1^{3/2}}{n_1} \right)\kappa^{2}\sigma_{\min}(\bcX_{\star}).
\end{align*}

\paragraph{Completing the proof.} The third and fourth terms in \eqref{eq:TC_init_expand} can be bounded similarly. In all, we conclude that
\begin{align*}
\dist(\bF_{+},\bF_{\star}) \le (\sqrt{2}+1)^{3/2}\left\|(\bU_{+},\bV_{+},\bW_{+})\bcdot\bcS_{+}-\bcX_{\star}\right\|_{\fro} \le \epsilon_{0}\sigma_{\min}(\bcX_{\star}).
\end{align*}

\subsection{Proof of local convergence (Lemma~\ref{lemma:contraction_TC})}

Define the event $\cE$ as the intersection of the events that Lemmas~\ref{lemma:P_Omega_tangent} and \ref{lemma:P_Omega_2inf} hold, which happens with overwhelming probability. The rest of the proof is then performed under the event that $\cE$ holds. 

Given that $\dist(\bF_{t},\bF_{\star})\le\epsilon\sigma_{\min}(\bcX_{\star})$, the conclusion $\|(\bU_{t},\bV_{t},\bW_{t})\bcdot\bcS_{t}-\bcX_{\star}\|_{\fro}\le3\dist(\bF_{t},\bF_{\star})$ follows from the relation~\eqref{eq:tensor2factors} in Lemma~\ref{lemma:perturb_bounds}. 
As in the proof of Theorem~\ref{thm:TF}, we reuse the notations in \eqref{eq:short_notations} and \eqref{eq:additional_notation_aligned}. By the definition of $\dist(\bF_{t+},\bF_{\star})$, where $\bF_{t+}$ is the update before projection, one has
\begin{align}
\dist^2(\bF_{t+},\bF_{\star}) &\le \left\|(\bU_{t+}\bQ_{t,1}-\bU_{\star})\bSigma_{\star,1}\right\|_{\fro}^{2}+\left\|(\bV_{t+}\bQ_{t,2}-\bV_{\star})\bSigma_{\star,2}\right\|_{\fro}^{2}+\left\|(\bW_{t+}\bQ_{t,3}-\bW_{\star})\bSigma_{\star,3}\right\|_{\fro}^{2} \nonumber\\
&\qquad + \left\|(\bQ_{t,1}^{-1},\bQ_{t,2}^{-1},\bQ_{t,3}^{-1})\bcdot\bcS_{t+}-\bcS_{\star}\right\|_{\fro}^2. \label{eq:TC_expand}
\end{align}
In the sequel, we shall bound each square on the right hand side of equation~\eqref{eq:TC_expand} separately. After a long journey of computation, the final result is
\begin{align}
\dist^{2}(\bF_{t+},\bF_{\star}) &\le (1-\eta)^2\left(\left\|\bDelta_{U}\bSigma_{\star,1}\right\|_{\fro}^{2}+\left\|\bDelta_{V}\bSigma_{\star,2}\right\|_{\fro}^{2}+\left\|\bDelta_{W}\bSigma_{\star,3}\right\|_{\fro}^{2}+\|\bDelta_{\cS}\|_{\fro}^2\right) \nonumber\\
 &\qquad - \eta(2-5\eta)\left\|\bcT_{U} + \bcT_{V} + \bcT_{W}\right\|_{\fro}^2 - \eta(2-5\eta)\left(\left\|\bD_{U}\right\|_{\fro}^{2} + \left\|\bD_{V}\right\|_{\fro}^{2} + \left\|\bD_{W}\right\|_{\fro}^{2}\right) \nonumber\\
 &\qquad + 2\eta(1-\eta) C(\epsilon+\delta + \delta^{2}) \dist^2(\bF_{t},\bF_{\star}) + \eta^2 C (\epsilon+\delta+\delta^{2})\dist^2(\bF_{t},\bF_{\star}), \label{eq:TC_bound}
\end{align}
where $C>1$ is some universal constant, and $\delta$ is defined as
\begin{align} \label{eq:delta_Def}
\delta \coloneqq C_T \sqrt{\frac{n\mu^2 r^2 \log n}{p n_1 n_2 n_3}} + C_Y\left(p^{-1}\log^3 n + \sqrt{p^{-1} n \log^5 n}\right) \sqrt{\frac{\mu^3 r^4}{n_1n_2n_3}} C_B^{3}\kappa^{3}.
\end{align}
Under the condition 
\begin{align*}
pn_1n_2n_3 \gtrsim \sqrt{n_1n_2n_3}\mu^{3/2} r^{2} \kappa^{3} \log^{3} n + n\mu^{3} r^{4} \kappa^{6} \log^{5} n,
\end{align*}
$\delta$ is a sufficiently small constant. As long as $\eta \le 2/5$ and $\epsilon$ is small, one has 
$\dist(\bF_{t+},\bF_{\star}) \le (1-0.6\eta)\dist(\bF_{t},\bF_{\star})$. Finally Lemma~\ref{lemma:scaled_proj} implies $\dist(\bF_{t+1},\bF_{\star})\le \dist(\bF_{t+},\bF_{\star}) \le (1-0.6\eta)\dist(\bF_{t},\bF_{\star})$ and the incoherence condition.

It then boils down to expanding and bounding the four terms in \eqref{eq:TC_expand}. As before, we omit the control
of the terms pertaining to $\bm{V}$ and $\bm{W}$.

\subsubsection{Bounding the term related to $\bU$}
The first term in \eqref{eq:TC_expand} is related to
\begin{align*}
 (\bU_{t+}\bQ_{t,1}-\bU_{\star})\bSigma_{\star,1} &= \Big(\bU-\eta\cM_{1}\left(p^{-1}\cP_{\Omega}((\bU,\bV,\bW)\bcdot\bcS-\bcX_{\star})\right)\breve{\bU}(\breve{\bU}^{\top}\breve{\bU})^{-1}-\bU_{\star}\Big)\bSigma_{\star,1} \\
 &= (1-\eta)\bDelta_{U}\bSigma_{\star,1} - \eta\bU_{\star}(\breve{\bU}-\breve{\bU}_{\star})^{\top}\breve{\bU}(\breve{\bU}^{\top}\breve{\bU})^{-1}\bSigma_{\star,1} \\
 &\qquad - \eta\cM_{1}\left((p^{-1}\cP_{\Omega}-\cI)((\bU,\bV,\bW)\bcdot\bcS-\bcX_{\star})\right)\breve{\bU}(\breve{\bU}^{\top}\breve{\bU})^{-1}\bSigma_{\star,1}.
\end{align*}
Take the squared norm of both sides to reach
\begin{align*}
 & \left\|(\bU_{t+}\bQ_{t,1}-\bU_{\star})\bSigma_{\star,1}\right\|_{\fro}^{2} = \underbrace{\left\|(1-\eta)\bDelta_{U}\bSigma_{\star,1} - \eta\bU_{\star}(\breve{\bU}-\breve{\bU}_{\star})^{\top}\breve{\bU}(\breve{\bU}^{\top}\breve{\bU})^{-1}\bSigma_{\star,1}\right\|_{\fro}^{2}}_{\eqqcolon\mfk{P}_{U}^{\main}} \\
 &\qquad - 2\eta(1-\eta)\underbrace{\left\langle\bDelta_{U}\bSigma_{\star,1}, \cM_{1}\left((p^{-1}\cP_{\Omega}-\cI)((\bU,\bV,\bW)\bcdot\bcS-\bcX_{\star})\right)\breve{\bU}(\breve{\bU}^{\top}\breve{\bU})^{-1}\bSigma_{\star,1}\right\rangle}_{\eqqcolon\mfk{P}_{U}^{\ptb,1}} \\
 &\qquad + 2\eta^{2}\underbrace{\left\langle\bU_{\star}(\breve{\bU}-\breve{\bU}_{\star})^{\top}\breve{\bU}(\breve{\bU}^{\top}\breve{\bU})^{-1}\bSigma_{\star,1}, \cM_{1}\left((p^{-1}\cP_{\Omega}-\cI)((\bU,\bV,\bW)\bcdot\bcS-\bcX_{\star})\right)\breve{\bU}(\breve{\bU}^{\top}\breve{\bU})^{-1}\bSigma_{\star,1}\right\rangle}_{\eqqcolon\mfk{P}_{U}^{\ptb,2}} \\
 &\qquad + \eta^{2}\underbrace{\left\|\cM_{1}\left((p^{-1}\cP_{\Omega}-\cI)((\bU,\bV,\bW)\bcdot\bcS-\bcX_{\star})\right)\breve{\bU}(\breve{\bU}^{\top}\breve{\bU})^{-1}\bSigma_{\star,1}\right\|_{\fro}^{2}}_{\eqqcolon\mfk{P}_{U}^{\ptb,3}}.
\end{align*}
As before, the main term $\mfk{P}_{U}^{\main}$ has been handled in the tensor factorization problem in Section~\ref{sec:proof_TF}; see \eqref{eq:first_square} and the bound~\eqref{eq:U-bound}. 
Hence we shall focus on the perturbation terms.
\paragraph{Step 1: bounding $\mfk{P}_{U}^{\ptb,1}$.}
First, rewrite $\mfk{P}_{U}^{\ptb,1}$ as the inner product in the tensor space:
\begin{align*}
\mfk{P}_{U}^{\ptb,1} = \left\langle\left(\bDelta_{U}\bSigma_{\star,1}^2(\breve{\bU}^{\top}\breve{\bU})^{-1}, \bV,\bW\right)\bcdot\bcS, (p^{-1}\cP_{\Omega}-\cI)((\bU,\bV,\bW)\bcdot\bcS-\bcX_{\star})\right\rangle.
\end{align*}
Apply the decomposition 
\begin{align}
(\bU,\bV,\bW)\bcdot\bcS-\bcX_{\star} &= (\bU,\bDelta_{V},\bW)\bcdot\bcS + (\bU,\bV_{\star},\bDelta_{W})\bcdot\bcS + (\bU,\bV_{\star},\bW_{\star})\bcdot\bcS - (\bU_{\star},\bV_{\star},\bW_{\star})\bcdot\bcS_{\star} \nonumber \\
&= (\bU,\bDelta_{V},\bW)\bcdot\bcS + (\bU,\bV_{\star},\bDelta_{W})\bcdot\bcS + (\bU,\bV_{\star},\bW_{\star})\bcdot\bDelta_{\cS} + (\bDelta_{U},\bV_{\star},\bW_{\star})\bcdot\bcS_{\star} \label{eq:decomp_T_tangent}
\end{align} 
to further expand $\mfk{P}_{U}^{\ptb,1}$ as
\begin{align*}
\mfk{P}_{U}^{\ptb,1} &= \underbrace{\left\langle\left(\bDelta_{U}\bSigma_{\star,1}^2(\breve{\bU}^{\top}\breve{\bU})^{-1}, \bV_{\star},\bW_{\star}\right)\bcdot\bcS, (p^{-1}\cP_{\Omega}-\cI)\left((\bU,\bV_{\star},\bW_{\star})\bcdot\bDelta_{\cS} + (\bDelta_{U},\bV_{\star},\bW_{\star})\bcdot\bcS_{\star}\right)\right\rangle}_{\eqqcolon \mfk{P}_{U}^{\ptb,1,1}} \\
&\quad + \underbrace{\begin{aligned}\Big\langle\left(\bDelta_{U}\bSigma_{\star,1}^2(\breve{\bU}^{\top}\breve{\bU})^{-1}, \bDelta_{V},\bW\right)\bcdot\bcS  + \left(\bDelta_{U}\bSigma_{\star,1}^2(\breve{\bU}^{\top}\breve{\bU})^{-1}, \bV_{\star},\bDelta_{W}\right)\bcdot\bcS, \\ 
(p^{-1}\cP_{\Omega}-\cI)\left((\bU,\bV_{\star},\bW_{\star})\bcdot\bcS - (\bU_{\star},\bV_{\star},\bW_{\star})\bcdot\bcS_{\star}\right)\Big\rangle\end{aligned}}_{\eqqcolon\mfk{P}_{U}^{\ptb,1,2}} \\
&\quad + \underbrace{\left\langle\left(\bDelta_{U}\bSigma_{\star,1}^2(\breve{\bU}^{\top}\breve{\bU})^{-1}, \bV,\bW\right)\bcdot\bcS, (p^{-1}\cP_{\Omega}-\cI) \left((\bU,\bDelta_{V},\bW)\bcdot\bcS + (\bU,\bV_{\star},\bDelta_{W})\bcdot\bcS\right)\right\rangle}_{\eqqcolon\mfk{P}_{U}^{\ptb,1,3}}.
\end{align*}
We shall bound each term in the sequel.
\begin{itemize}
\item For the first term $\mfk{P}_{U}^{\ptb,1,1}$, we resort to Lemma~\ref{lemma:P_Omega_tangent}, which leads to
\begin{align*}
|\mfk{P}_{U}^{\ptb,1,1}| &\le C_T\sqrt{\frac{n\mu^2 r^2 \log n}{p n_1 n_2 n_3}}\left\|\left(\bDelta_{U}\bSigma_{\star,1}^2(\breve{\bU}^{\top}\breve{\bU})^{-1}, \bV_{\star},\bW_{\star}\right)\bcdot\bcS\right\|_{\fro} \left\|(\bU,\bV_{\star},\bW_{\star})\bcdot\bDelta_{\cS} + (\bDelta_{U},\bV_{\star},\bW_{\star})\bcdot\bcS_{\star}\right\|_{\fro}. 
\end{align*}
Further use \eqref{eq:TC_SRinv} to bound that
\begin{align*}
\left\|\left(\bDelta_{U}\bSigma_{\star,1}^2(\breve{\bU}^{\top}\breve{\bU})^{-1}, \bV_{\star},\bW_{\star}\right)\bcdot\bcS\right\|_{\fro} &= \left\|\bDelta_{U}\bSigma_{\star,1}^2(\breve{\bU}^{\top}\breve{\bU})^{-1}\cM_{1}(\bcS)(\bW_{\star}\otimes\bV_{\star})^{\top}\right\|_{\fro} \\
&\le \|\bDelta_{U}\bSigma_{\star,1}\|_{\fro}\left\|\bSigma_{\star,1}(\breve{\bU}^{\top}\breve{\bU})^{-1}\cM_{1}(\bcS)\right\|_{\op}  \\
&\le \|\bDelta_{U}\bSigma_{\star,1}\|_{\fro}(1-\epsilon)^{-5},
\end{align*}
and that
\begin{align*}
\|(\bU,\bV_{\star},\bW_{\star})\bcdot\bDelta_{\cS}\|_{\fro} &\le \|\bU\cM_{1}(\bDelta_{\cS})\|_{\fro} \le (1+\epsilon)\|\bDelta_{\cS}\|_{\fro}; \\
\|(\bDelta_{U},\bV_{\star},\bW_{\star})\bcdot\bcS_{\star}\|_{\fro} &\le \|\bDelta_{U}\bSigma_{\star,1}\|_{\fro}. 
\end{align*}
Combine the preceding bounds to see
\begin{align*}
|\mfk{P}_{U}^{\ptb,1,1}| &\le C_T\sqrt{\frac{n\mu^2 r^2 \log n}{p n_1 n_2 n_3}} \frac{\|\bDelta_{U}\bSigma_{\star,1}\|_{\fro}}{(1-\epsilon)^{5}}\left(\|\bDelta_{U}\bSigma_{\star,1}\|_{\fro} + (1+\epsilon)\|\bDelta_{\cS}\|_{\fro}\right).
\end{align*}

\item For the second term $\mfk{P}_{U}^{\ptb,1,2}$, our main hammer is Lemma~\ref{lemma:P_Omega_2inf}, which implies
\begin{align*}
|\mfk{P}_{U}^{\ptb,1,2}| &\le C_Y \left(p^{-1}\log^3 n + \sqrt{p^{-1} n \log^5 n}\right)\left\|\bDelta_{U}\bSigma_{\star,1}^{2}(\breve{\bU}^{\top}\breve{\bU})^{-1}\cM_{1}(\bcS)\right\|_{\fro} \\
&\qquad \left(\left\|\bU\cM_{1}(\bcS)\right\|_{2,\infty} + \left\|\bU_{\star}\cM_{1}(\bcS_{\star})\right\|_{2,\infty}\right)\left(\|\bDelta_{V}\|_{\fro}\|\bW\|_{\fro} + \|\bV_{\star}\|_{\fro}\|\bDelta_{W}\|_{\fro}\right)\|\bV_{\star}\|_{2,\infty}\|\bW_{\star}\|_{2,\infty}.
\end{align*}
Use results in Lemma~\ref{lemma:TC_perturb_bounds}, together with the bounds
\begin{align*}
\|\bDelta_{V}\|_{\fro} &\le \frac{\|\bDelta_{V}\bSigma_{\star,2}\|_{\fro}}{\sigma_{\min}(\bSigma_{\star,2})} \le \frac{\|\bDelta_{V}\bSigma_{\star,2}\|_{\fro}}{\sigma_{\min}(\bcX_{\star})}; &\qquad \|\bDelta_{W}\|_{\fro} &\le \frac{\|\bDelta_{W}\bSigma_{\star,3}\|_{\fro}}{\sigma_{\min}(\bcX_{\star})}; \\
\|\bW\|_{\fro} &\le \sqrt{r_3}\|\bW\|_{\op} \le \sqrt{r_3}(1+\epsilon); \qquad& \|\bV_{\star}\|_{\fro} &= \sqrt{r_2}; \\
\|\bU_{\star}\cM_{1}(\bS_{\star})\|_{2,\infty} &\le \|\bU_{\star}\|_{2,\infty}\|\cM_{1}(\bS_{\star})\|_{\op} \le \sqrt{\frac{\mu r}{n_1}}\sigma_{\max}(\bcX_{\star}); & \|\bV_{\star}\|_{2,\infty} &\le \sqrt{\frac{\mu r}{n_2}}; \qquad \|\bW_{\star}\|_{2,\infty} \le \sqrt{\frac{\mu r}{n_3}},
\end{align*}
to arrive at the conclusion that 
\begin{align*}
|\mfk{P}_{U}^{\ptb,1,2}| &\le C_Y \left(p^{-1}\log^3 n + \sqrt{p^{-1} n \log^5 n}\right) \frac{\|\bDelta_{U}\bSigma_{\star,1}\|_{\fro}}{(1-\epsilon)^{5}} \left((1-\epsilon)^{-2}C_B+1\right)\sqrt{\frac{\mu r}{n_1}}\sigma_{\max}(\bcX_{\star}) \\
&\qquad \left(\frac{\|\bDelta_{V}\bSigma_{\star,2}\|_{\fro}}{\sigma_{\min}(\bcX_{\star})}\sqrt{r}(1+\epsilon) + \sqrt{r}\frac{\|\bDelta_{W}\bSigma_{\star,3}\|_{\fro}}{\sigma_{\min}(\bcX_{\star})}\right)\sqrt{\frac{\mu r}{n_2}} \sqrt{\frac{\mu r}{n_3}} \\
&= C_Y \left(p^{-1}\log^3 n + \sqrt{p^{-1} n \log^5 n}\right)\sqrt{\frac{\mu^3 r^4}{n_1n_2n_3}}\frac{(1-\epsilon)^{-2}C_B+1}{(1-\epsilon)^{5}}\kappa \\
&\qquad \|\bDelta_{U}\bSigma_{\star,1}\|_{\fro}\left((1+\epsilon)\|\bDelta_{V}\bSigma_{\star,2}\|_{\fro} + \|\bDelta_{W}\bSigma_{\star,3}\|_{\fro}\right).
\end{align*}

\item Repeat similar arguments, we can obtain the bound on $\mfk{P}_{U}^{\ptb,1,3}$:
\begin{align*}
|\mfk{P}_{U}^{\ptb,1,3}| &\le C_Y \left(p^{-1}\log^3 n + \sqrt{p^{-1} n \log^5 n}\right)\left\|\bDelta_{U}\bSigma_{\star,1}^{2}(\breve{\bU}^{\top}\breve{\bU})^{-1}\cM_{1}(\bcS)\right\|_{\fro} \|\bU\cM_{1}(\bcS)\|_{2,\infty} \\
&\qquad \|\bV\|_{2,\infty}\|\bW\|_{2,\infty} (\|\bDelta_{V}\|_{\fro}\|\bW\|_{\fro} + \|\bV_{\star}\|_{\fro}\|\bDelta_{W}\|_{\fro}) \\
&\le C_Y \left(p^{-1}\log^3 n + \sqrt{p^{-1} n \log^5 n}\right) \frac{\|\bDelta_{U}\bSigma_{\star,1}\|_{\fro}}{(1-\epsilon)^{5}} \frac{C_B}{(1-\epsilon)^{2}}\sqrt{\frac{\mu r}{n_1}}\sigma_{\max}(\bcX_{\star}) \\
&\qquad \frac{C_B\kappa}{(1-\epsilon)^3}\sqrt{\frac{\mu r}{n_2}} \frac{C_B\kappa}{(1-\epsilon)^3}\sqrt{\frac{\mu r}{n_3}} \left(\frac{\|\bDelta_{V}\bSigma_{\star,2}\|_{\fro}}{\sigma_{\min}(\bcX_{\star})}\sqrt{r}(1+\epsilon) + \sqrt{r}\frac{\|\bDelta_{W}\bSigma_{\star,3}\|_{\fro}}{\sigma_{\min}(\bcX_{\star})}\right) \\
&\le C_Y \left(p^{-1}\log^3 n + \sqrt{p^{-1} n \log^5 n}\right)\sqrt{\frac{\mu^3 r^4}{n_1n_2n_3}}\frac{C_B^{3}\kappa^{3}}{(1-\epsilon)^{13}} \\
&\qquad \|\bDelta_{U}\bSigma_{\star,1}\|_{\fro}\left((1+\epsilon)\|\bDelta_{V}\bSigma_{\star,2}\|_{\fro} + \|\bDelta_{W}\bSigma_{\star,3}\|_{\fro}\right).
\end{align*}
\end{itemize}
In total, we have
\begin{align*}
|\mfk{P}_{U}^{\ptb,1}| \le |\mfk{P}_{U}^{\ptb,1,1}| + |\mfk{P}_{U}^{\ptb,1,2}| + |\mfk{P}_{U}^{\ptb,1,3}| \lesssim \delta \dist^2(\bF_{t},\bF_{\star}),
\end{align*}
where we recall the definition of $\delta$ in \eqref{eq:delta_Def}.

\paragraph{Step 2: bounding $\mfk{P}_{U}^{\ptb,2}$.}
We begin by rewriting $\mfk{P}_{U}^{\ptb,2}$  as
\begin{align*}
\mfk{P}_{U}^{\ptb,2} &= \left\langle\left(\bU_{\star}(\breve{\bU}-\breve{\bU}_{\star})^{\top}\breve{\bU}(\breve{\bU}^{\top}\breve{\bU})^{-1}\bSigma_{\star,1}^{2}(\breve{\bU}^{\top}\breve{\bU})^{-1},\bV,\bW\right)\bcdot\bcS, (p^{-1}\cP_{\Omega}-\cI)((\bU,\bV,\bW)\bcdot\bcS-\bcX_{\star})\right\rangle.
\end{align*}
Compared to $\mfk{P}_{U}^{\ptb,1}$, the only difference is that the leading term $\bDelta_{U}\bSigma_{\star,1}$ in the first argument of the inner product is replaced by $\bU_{\star}(\breve{\bU}-\breve{\bU}_{\star})^{\top}\breve{\bU}(\breve{\bU}^{\top}\breve{\bU})^{-1}\bSigma_{\star,1}$. 
Note that 
\begin{align*}
\left\|\bU_{\star}(\breve{\bU}-\breve{\bU}_{\star})^{\top}\breve{\bU}(\breve{\bU}^{\top}\breve{\bU})^{-1}\bSigma_{\star,1}\right\|_{\fro} & \le \left\|\breve{\bU}-\breve{\bU}_{\star}\right\|_{\fro}\left\|\breve{\bU}(\breve{\bU}^{\top}\breve{\bU})^{-1}\bSigma_{\star,1}\right\|_{\fro} \\
 & \le \frac{1+\epsilon+\frac{1}{3}\epsilon^2}{(1-\epsilon)^3}\left(\|\bDelta_{V}\bSigma_{\star,2}\|_{\fro}+\|\bDelta_{W}\bSigma_{\star,3}\|_{\fro}+\|\bDelta_{\cS}\|_{\fro}\right).
\end{align*}
Omitting the somewhat tedious details, we can go through the same argument as bounding $\mfk{P}_{U}^{\ptb,1}$ and arrive at
\begin{align*}
|\mfk{P}_{U}^{\ptb,2}| &\le C_T\sqrt{\frac{n\mu^2 r^2 \log n}{p n_1 n_2 n_3}} \frac{1+\epsilon+\frac{1}{3}\epsilon^2}{(1-\epsilon)^{8}} \left(\|\bDelta_{V}\bSigma_{\star,2}\|_{\fro}+\|\bDelta_{W}\bSigma_{\star,3}\|_{\fro}+\|\bDelta_{\cS}\|_{\fro}\right)\left(\|\bDelta_{U}\bSigma_{\star,1}\|_{\fro} + (1+\epsilon)\|\bDelta_{\cS}\|_{\fro}\right) \\
&\qquad + C_Y \left(p^{-1}\log^3 n + \sqrt{p^{-1} n \log^5 n}\right)\sqrt{\frac{\mu^3 r^4}{n_1n_2n_3}}\frac{(1+\epsilon+\frac{1}{3}\epsilon^2)((1-\epsilon)^{-2}C_B+1)}{(1-\epsilon)^{8}}\kappa \\
&\qquad\qquad \left(\|\bDelta_{V}\bSigma_{\star,2}\|_{\fro}+\|\bDelta_{W}\bSigma_{\star,3}\|_{\fro}+\|\bDelta_{\cS}\|_{\fro}\right)\left((1+\epsilon)\|\bDelta_{V}\bSigma_{\star,2}\|_{\fro} + \|\bDelta_{W}\bSigma_{\star,3}\|_{\fro}\right) \\
&\qquad + C_Y \left(p^{-1}\log^3 n + \sqrt{p^{-1} n \log^5 n}\right)\sqrt{\frac{\mu^3 r^4}{n_1n_2n_3}}\frac{(1+\epsilon+\frac{1}{3}\epsilon^2)C_B^{3}\kappa^{3}}{(1-\epsilon)^{16}} \\
&\qquad\qquad  \left(\|\bDelta_{V}\bSigma_{\star,2}\|_{\fro}+\|\bDelta_{W}\bSigma_{\star,3}\|_{\fro}+\|\bDelta_{\cS}\|_{\fro}\right)\left((1+\epsilon)\|\bDelta_{V}\bSigma_{\star,2}\|_{\fro} + \|\bDelta_{W}\bSigma_{\star,3}\|_{\fro}\right) \\
&\lesssim \delta \dist^2(\bF_{t},\bF_{\star}).
\end{align*}

\paragraph{Step 3: bounding $\mfk{P}_{U}^{\ptb,3}$.}
Use the variational representation of the Frobenius norm to write
\begin{align*}
\sqrt{\mfk{P}_{U}^{\ptb,3}} = \left\langle\left(\widetilde{\bU}\bSigma_{\star,1}(\breve{\bU}^{\top}\breve{\bU})^{-1},\bV,\bW\right)\bcdot\bcS, (p^{-1}\cP_{\Omega}-\cI)((\bU,\bV,\bW)\bcdot\bcS-\bcX_{\star})\right\rangle
\end{align*}
for some $\widetilde{\bU}\in\RR^{n_1\times r_1}$ obeying $\|\widetilde{\bU}\|_{\fro}=1$. Repeat the same argument as bounding $\mfk{P}_{U}^{\ptb,1}$ with proper modifications to yield 
\begin{align*}
\sqrt{\mfk{P}_{U}^{\ptb,3}} &\le C_T\sqrt{\frac{n\mu^2 r^2 \log n}{p n_1 n_2 n_3}} (1-\epsilon)^{-5} \left(\|\bDelta_{U}\bSigma_{\star,1}\|_{\fro} + (1+\epsilon)\|\bDelta_{\cS}\|_{\fro}\right) \\
&\quad + C_Y \left(p^{-1}\log^3 n + \sqrt{p^{-1} n \log^5 n}\right)\sqrt{\frac{\mu^3 r^4}{n_1n_2n_3}}\frac{(1-\epsilon)^{-2}C_B+1}{(1-\epsilon)^{5}}\kappa \left((1+\epsilon)\|\bDelta_{V}\bSigma_{\star,2}\|_{\fro} + \|\bDelta_{W}\bSigma_{\star,3}\|_{\fro}\right) \\
&\quad + C_Y \left(p^{-1}\log^3 n + \sqrt{p^{-1} n \log^5 n}\right)\sqrt{\frac{\mu^3 r^4}{n_1n_2n_3}}\frac{C_B^{3}\kappa^{3}}{(1-\epsilon)^{13}} \left((1+\epsilon)\|\bDelta_{V}\bSigma_{\star,2}\|_{\fro} + \|\bDelta_{W}\bSigma_{\star,3}\|_{\fro}\right) \\
& \lesssim \delta \dist(\bF_{t},\bF_{\star}).
\end{align*}
Then take the square of both sides to see
\begin{align*}
\mfk{P}_{U}^{\ptb,3} \lesssim \delta^{2} \dist^{2}(\bF_{t},\bF_{\star}).
\end{align*}

\subsubsection{Bounding the term related to $\bcS$}
The last term of \eqref{eq:TC_expand} is related to
\begin{align*}
 &(\bQ_{t,1}^{-1},\bQ_{t,2}^{-1},\bQ_{t,3}^{-1})\bcdot\bcS_{t+}-\bcS_{\star}\\
 &\quad = \bcS - \eta\left((\bU^{\top}\bU)^{-1}\bU^{\top}, (\bV^{\top}\bV)^{-1}\bV^{\top}, (\bW^{\top}\bW)^{-1}\bW^{\top}\right)\bcdot p^{-1}\cP_{\Omega}\left((\bU,\bV,\bW)\bcdot\bcS-\bcX_{\star}\right) - \bcS_{\star} \\
 &\quad = (1-\eta)\bDelta_{\cS} - \eta \left((\bU^{\top}\bU)^{-1}\bU^{\top}, (\bV^{\top}\bV)^{-1}\bV^{\top}, (\bW^{\top}\bW)^{-1}\bW^{\top}\right)\bcdot\left((\bU,\bV,\bW)\bcdot\bcS_{\star}-\bcX_{\star}\right) \\
 &\qquad - \eta\left((\bU^{\top}\bU)^{-1}\bU^{\top}, (\bV^{\top}\bV)^{-1}\bV^{\top}, (\bW^{\top}\bW)^{-1}\bW^{\top}\right)\bcdot(p^{-1}\cP_{\Omega}-\cI)((\bU,\bV,\bW)\bcdot\bcS-\bcX_{\star}).
\end{align*}
Expand its squared norm to obtain
\begin{align*}
 & \left\|(\bQ_{t,1}^{-1},\bQ_{t,2}^{-1},\bQ_{t,3}^{-1})\bcdot\bcS_{t+}-\bcS_{\star}\right\|_{\fro}^2 \\
 &\quad = \underbrace{\left\|(1-\eta)\bDelta_{\cS} - \eta \left((\bU^{\top}\bU)^{-1}\bU^{\top}, (\bV^{\top}\bV)^{-1}\bV^{\top}, (\bW^{\top}\bW)^{-1}\bW^{\top}\right)\bcdot\left((\bU,\bV,\bW)\bcdot\bcS_{\star}-\bcX_{\star}\right)\right\|_{\fro}^2}_{ \eqqcolon\mfk{P}_{\cS}^{\main}} \\
 &\quad\quad - 2\eta(1-\eta)\underbrace{\left\langle\bDelta_{\cS}, \left((\bU^{\top}\bU)^{-1}\bU^{\top}, (\bV^{\top}\bV)^{-1}\bV^{\top}, (\bW^{\top}\bW)^{-1}\bW^{\top}\right)\bcdot(p^{-1}\cP_{\Omega}-\cI)((\bU,\bV,\bW)\bcdot\bcS-\bcX_{\star})\right\rangle}_{ \eqqcolon \mfk{P}_{\cS}^{\ptb,1}} \\
 &\quad\quad + 2\eta^2\underbrace{\begin{aligned} &\Big\langle\left((\bU^{\top}\bU)^{-1}\bU^{\top}, (\bV^{\top}\bV)^{-1}\bV^{\top}, (\bW^{\top}\bW)^{-1}\bW^{\top}\right)\bcdot\left((\bU,\bV,\bW)\bcdot\bcS_{\star}-\bcX_{\star}\right), \\
 &\quad\quad \left((\bU^{\top}\bU)^{-1}\bU^{\top}, (\bV^{\top}\bV)^{-1}\bV^{\top}, (\bW^{\top}\bW)^{-1}\bW^{\top}\right)\bcdot(p^{-1}\cP_{\Omega}-\cI)((\bU,\bV,\bW)\bcdot\bcS-\bcX_{\star}\Big\rangle \end{aligned}}_{ \eqqcolon \mfk{P}_{\cS}^{\ptb,2}} \\
 &\quad\quad + \eta^{2}\underbrace{\left\|\left((\bU^{\top}\bU)^{-1}\bU^{\top}, (\bV^{\top}\bV)^{-1}\bV^{\top}, (\bW^{\top}\bW)^{-1}\bW^{\top}\right)\bcdot(p^{-1}\cP_{\Omega}-\cI)((\bU,\bV,\bW)\bcdot\bcS-\bcX_{\star})\right\|_{\fro}^{2}}_{\eqqcolon \mfk{P}_{\cS}^{\ptb,3}}.
\end{align*}
Recall that the main term $\mfk{P}_{\cS}^{\main}$ has been controlled in Section~\ref{sec:proof_TF}; see \eqref{eq:last_square} and the bound~\eqref{eq:S-bound}. We therefore concentrate on the remaining perturbation terms.

\paragraph{Step 1: bounding $\mfk{P}_{\cS}^{\ptb,1}$.}
Write $\mfk{P}_{\cS}^{\ptb,1}$ as
\begin{align*}
\mfk{P}_{\cS}^{\ptb,1} = \left\langle\left(\bU(\bU^{\top}\bU)^{-1}, \bV(\bV^{\top}\bV)^{-1}, \bW(\bW^{\top}\bW)^{-1}\right)\bcdot\bDelta_{\cS}, (p^{-1}\cP_{\Omega}-\cI)((\bU,\bV,\bW)\bcdot\bcS-\bcX_{\star})\right\rangle.
\end{align*}
Use the decomposition \eqref{eq:decomp_T_tangent} to further obtain
\begin{align*}
\mfk{P}_{\cS}^{\ptb,1} &= \underbrace{\left\langle\left(\bU(\bU^{\top}\bU)^{-1}, \bV_{\star}(\bV^{\top}\bV)^{-1},\bW_{\star}(\bW^{\top}\bW)^{-1}\right)\bcdot\bDelta_{\cS}, (p^{-1}\cP_{\Omega}-\cI)\left((\bU,\bV_{\star},\bW_{\star})\bcdot\bDelta_{\cS} + (\bDelta_{U},\bV_{\star},\bW_{\star})\bcdot\bcS_{\star}\right)\right\rangle}_{ \eqqcolon\mfk{P}_{\cS}^{\ptb,1,1}} \\
&\quad + \underbrace{
\begin{aligned}\Big\langle\left(\bU(\bU^{\top}\bU)^{-1}, \bDelta_{V}(\bV^{\top}\bV)^{-1},\bW(\bW^{\top}\bW)^{-1}\right)\bcdot\bDelta_{\cS} + \left(\bU(\bU^{\top}\bU)^{-1}, \bV_{\star}(\bV^{\top}\bV)^{-1},\bDelta_{W}(\bW^{\top}\bW)^{-1}\right)\bcdot\bDelta_{\cS}, \\
(p^{-1}\cP_{\Omega}-\cI)\left((\bU,\bV_{\star},\bW_{\star})\bcdot\bcS - (\bU_{\star},\bV_{\star},\bW_{\star})\bcdot\bcS_{\star}\right)\Big\rangle\end{aligned}}_{\eqqcolon\mfk{P}_{\cS}^{\ptb,1,2}} \\
&\quad + \underbrace{\left\langle\left(\bU(\bU^{\top}\bU)^{-1}, \bV(\bV^{\top}\bV)^{-1}, \bW(\bW^{\top}\bW)^{-1}\right)\bcdot\bDelta_{\cS}, (p^{-1}\cP_{\Omega}-\cI) \left((\bU,\bDelta_{V},\bW)\bcdot\bcS + (\bU,\bV_{\star},\bDelta_{W})\bcdot\bcS\right)\right\rangle}_{\eqqcolon \mfk{P}_{\cS}^{\ptb,1,3}}.
\end{align*}
We then bound each term in sequel.
\begin{itemize}
\item Regarding the first term $\mfk{P}_{\cS}^{\ptb,1,1}$, we can apply Lemma~\ref{lemma:P_Omega_tangent} to see
\begin{align*}
|\mfk{P}_{\cS}^{\ptb,1,1}| &\le C_T\sqrt{\frac{n\mu^2 r^2 \log n}{p n_1 n_2 n_3}}\left\|\left(\bU(\bU^{\top}\bU)^{-1}, \bV_{\star}(\bV^{\top}\bV)^{-1},\bW_{\star}(\bW^{\top}\bW)^{-1}\right)\bcdot\bDelta_{\cS}\right\|_{\fro} \\
&\quad \left\|(\bU,\bV_{\star},\bW_{\star})\bcdot\bDelta_{\cS} + (\bDelta_{U},\bV_{\star},\bW_{\star})\bcdot\bcS_{\star}\right\|_{\fro}. 
\end{align*}
In addition, notice that
\begin{align*}
\left\|\left(\bU(\bU^{\top}\bU)^{-1}, \bV_{\star}(\bV^{\top}\bV)^{-1},\bW_{\star}(\bW^{\top}\bW)^{-1}\right)\bcdot\bDelta_{\cS}\right\|_{\fro} 
&\le \left\|\bU(\bU^{\top}\bU)^{-1}\right\|_{\op}\left\|(\bV^{\top}\bV)^{-1}\right\|_{\op}\left\|(\bW^{\top}\bW)^{-1}\right\|_{\op}\|\bDelta_{\cS}\|_{\fro} \\
&\le (1-\epsilon)^{-5}\|\bDelta_{\cS}\|_{\fro},
\end{align*}
which further implies 
\begin{align*}
|\mfk{P}_{\cS}^{\ptb,1,1}| &\le C_T\sqrt{\frac{n\mu^2 r^2 \log n}{p n_1 n_2 n_3}}(1-\epsilon)^{-5}\|\bDelta_{\cS}\|_{\fro}\left(\|\bDelta_{U}\bSigma_{\star,1}\|_{\fro} + (1+\epsilon)\|\bDelta_{\cS}\|_{\fro}\right).
\end{align*}
\item Now we turn to the second term $\mfk{P}_{\cS}^{\ptb,1,2}$, for which  Lemma~\ref{lemma:P_Omega_2inf} yields
\begin{align*}
|\mfk{P}_{\cS}^{\ptb,1,2}| &\le C_Y \left(p^{-1}\log^3 n + \sqrt{p^{-1} n \log^5 n}\right)\left\|\bU(\bU^{\top}\bU)^{-1}\cM_{1}(\bDelta_{\cS})\right\|_{\fro}\left(\left\|\bU\cM_{1}(\bcS)\right\|_{2,\infty} + \left\|\bU_{\star}\cM_{1}(\bcS_{\star})\right\|_{2,\infty}\right) \\
&\qquad \left(\left\|\bDelta_{V}(\bV^{\top}\bV)^{-1}\right\|_{\fro}\left\|\bW(\bW^{\top}\bW)^{-1}\right\|_{\fro} + \left\|\bV_{\star}(\bV^{\top}\bV)^{-1}\right\|_{\fro}\left\|\bDelta_{W}(\bW^{\top}\bW)^{-1}\right\|_{\fro}\right)\|\bV_{\star}\|_{2,\infty}\|\bW_{\star}\|_{2,\infty}.
\end{align*}
The results in Lemma~\ref{lemma:TC_perturb_bounds} together with the bounds
\begin{align*}
\left\|\bDelta_{V}(\bV^{\top}\bV)^{-1}\right\|_{\fro} &\le \|\bDelta_{V}\|_{\fro}\left\|(\bV^{\top}\bV)^{-1}\right\|_{\op} \le (1-\epsilon)^{-2}\|\bDelta_{V}\|_{\fro} \le \frac{\|\bDelta_{V}\bSigma_{\star,2}\|_{\fro}}{(1-\epsilon)^{2}\sigma_{\min}(\bcX_{\star})}; \\
\left\|\bW(\bW^{\top}\bW)^{-1}\right\|_{\fro} &\le \sqrt{r_3}\left\|\bW(\bW^{\top}\bW)^{-1}\right\|_{\op} \le \sqrt{r_3}(1-\epsilon)^{-1}; \\
\left\|\bV_{\star}(\bV^{\top}\bV)^{-1}\right\|_{\fro} &\le \|\bV_{\star}\|_{\fro}\left\|(\bV^{\top}\bV)^{-1}\right\|_{\op} \le \sqrt{r_2}(1-\epsilon)^{-2}; \\
\left\|\bDelta_{W}(\bW^{\top}\bW)^{-1}\right\|_{\fro} &\le \|\bDelta_{W}\|_{\fro}\left\|(\bW^{\top}\bW)^{-1}\right\|_{\op} \le \|\bDelta_{W}\|_{\fro}(1-\epsilon)^{-2} \le \frac{\|\bDelta_{W}\bSigma_{\star,3}\|_{\fro}}{(1-\epsilon)^{2}\sigma_{\min}(\bcX_{\star})},
\end{align*}
allow us to continue the bound 
\begin{align*}
|\mfk{P}_{\cS}^{\ptb,1,2}| &\le C_Y \left(p^{-1}\log^3 n + \sqrt{p^{-1} n \log^5 n}\right) \sqrt{\frac{\mu^3 r^4}{n_1n_2n_3}}\frac{(1-\epsilon)^{-2}C_B+1}{(1-\epsilon)^{5}}\kappa \|\bDelta_{\cS}\|_{\fro}\\
&\qquad\left((1-\epsilon)\|\bDelta_{V}\bSigma_{\star,2}\|_{\fro} + \|\bDelta_{W}\bSigma_{\star,3}\|_{\fro}\right).
\end{align*}

\item A similar strategy bounds $\mfk{P}_{\cS}^{\ptb,1,3}$ as
\begin{align*}
|\mfk{P}_{\cS}^{\ptb,1,3}| &\le C_Y \left(p^{-1}\log^3 n + \sqrt{p^{-1} n \log^5 n}\right)\left\|\bU(\bU^{\top}\bU)^{-1}\cM_{1}(\bDelta_{\cS})\right\|_{\fro} \left\|\bU\cM_{1}(\bcS)\right\|_{2,\infty} \\
&\qquad \left\|\bV(\bV^{\top}\bV)^{-1}\right\|_{2,\infty}\left\|\bW(\bW^{\top}\bW)^{-1}\right\|_{2,\infty}\left(\|\bDelta_{V}\|_{\fro}\|\bW\|_{\fro} + \|\bV_{\star}\|_{\fro}\|\bDelta_{W}\|_{\fro}\right).
\end{align*}
Further combine \eqref{eq:TC_U_2inf} and \eqref{eq:perturb_Uinv_2} to see
\begin{align*}
\left\|\bV(\bV^{\top}\bV)^{-1}\right\|_{2,\infty} \le \|\bV\|_{2,\infty}\left\|(\bV^{\top}\bV)^{-1}\right\|_{\op} \le (1-\epsilon)^{-5}C_B\sqrt{\frac{\mu r}{n_2}}\kappa; \\
\left\|\bW(\bW^{\top}\bW)^{-1}\right\|_{2,\infty} \le \|\bW\|_{2,\infty}\left\|(\bW^{\top}\bW)^{-1}\right\|_{\op} \le (1-\epsilon)^{-5}C_B\sqrt{\frac{\mu r}{n_3}}\kappa. 
\end{align*}
These taken collectively with the results in Lemma~\ref{lemma:TC_perturb_bounds} yield 
\begin{align*}
|\mfk{P}_{\cS}^{\ptb,1,3}| &\le C_Y \left(p^{-1}\log^3 n + \sqrt{p^{-1} n \log^5 n}\right) \sqrt{\frac{\mu^3 r^4}{n_1n_2n_3}}\frac{C_B^3\kappa^3}{(1-\epsilon)^{13}} \|\bDelta_{\cS}\|_{\fro}\left((1+\epsilon)\|\bDelta_{V}\bSigma_{\star,2}\|_{\fro} + \|\bDelta_{W}\bSigma_{\star,3}\|_{\fro}\right).
\end{align*}
\end{itemize}

In the end, we conclude that
\begin{align*}
|\mfk{P}_{\cS}^{\ptb,1}| \le |\mfk{P}_{\cS}^{\ptb,1,1}| + |\mfk{P}_{\cS}^{\ptb,1,2}| + |\mfk{P}_{\cS}^{\ptb,1,3}| \lesssim \delta \dist^2(\bF_{t},\bF_{\star}),
\end{align*}
where we recall the definition of $\delta$ in \eqref{eq:delta_Def}.

\paragraph{Step 2: bounding $\mfk{P}_{\cS}^{\ptb,2}$.}
Write $\mfk{P}_{\cS}^{\ptb,2}$ as
\begin{multline*}
\mfk{P}_{\cS}^{\ptb,2} = \Big\langle\left(\bU(\bU^{\top}\bU)^{-2}\bU^{\top}, \bV(\bV^{\top}\bV)^{-2}\bV^{\top}, \bW(\bW^{\top}\bW)^{-2}\bW^{\top}\right)\bcdot\left((\bU,\bV,\bW)\bcdot\bcS_{\star}-\bcX_{\star}\right),  \\
(p^{-1}\cP_{\Omega}-\cI)((\bU,\bV,\bW)\bcdot\bcS-\bcX_{\star})\Big\rangle.
\end{multline*}
Compared to $\mfk{P}_{\cS}^{\ptb,1}$, the only difference is that the quantity $\bDelta_{\cS}$ in the first argument of the inner product is replaced by 
\begin{align*}
\left((\bU^{\top}\bU)^{-1}\bU^{\top},(\bV^{\top}\bV)^{-1}\bV^{\top},(\bW^{\top}\bW)^{-1}\bW^{\top}\right)\bcdot\left((\bU,\bV,\bW)\bcdot\bcS_{\star}-\bcX_{\star}\right),
\end{align*}
whose Frobenius norm can be bounded by
\begin{align*}
&\left\|\left((\bU^{\top}\bU)^{-1}\bU^{\top},(\bV^{\top}\bV)^{-1}\bV^{\top},(\bW^{\top}\bW)^{-1}\bW^{\top}\right)\bcdot\left((\bU,\bV,\bW)\bcdot\bcS_{\star}-\bcX_{\star}\right)\right\|_{\fro} \\
&\qquad\le \left\|\bU(\bU^{\top}\bU)^{-1}\right\|_{\fro}\left\|\bV(\bV^{\top}\bV)^{-1}\right\|_{\fro}\left\|\bW(\bW^{\top}\bW)^{-1}\right\|_{\fro}\left\|(\bU,\bV,\bW)\bcdot\bcS_{\star}-\bcX_{\star}\right\|_{\fro} \\
&\qquad\le\frac{1+\epsilon+\frac{1}{3}\epsilon^{2}}{(1-\epsilon)^{3}}\left(\|\bDelta_{U}\bSigma_{\star,1}\|_{\fro}+\|\bDelta_{V}\bSigma_{\star,2}\|_{\fro}+\|\bDelta_{W}\bSigma_{\star,3}\|_{\fro}\right).
\end{align*}
We can then repeat the same argument as bounding $\mfk{P}_{\cS}^{\ptb,1}$ to obtain
\begin{align*}
|\mfk{P}_{\cS}^{\ptb,2}| \lesssim \delta \dist^2(\bF_{t},\bF_{\star}).
\end{align*}
For the sake of space, we omit the details.

\paragraph{Step 3: bounding $\mfk{P}_{\cS}^{\ptb,3}$.}
Use the variational representation of the Frobenius norm to write
\begin{align*}
\sqrt{\mfk{P}_{\cS}^{\ptb,3}} = \left\langle\left(\bU(\bU^{\top}\bU)^{-1},\bV(\bV^{\top}\bV)^{-1},\bW(\bW^{\top}\bW)^{-1}\right)\bcdot\widetilde{\bcS}, (p^{-1}\cP_{\Omega}-\cI)((\bU,\bV,\bW)\bcdot\bcS-\bcX_{\star})\right\rangle
\end{align*}
for some $\widetilde{\bcS}\in\RR^{n_1\times n_2\times n_3}$ obeying $\|\widetilde{\bcS}\|_{\fro}=1$. Repeating the same argument as bounding $\mfk{P}_{\cS}^{\ptb,1}$ with proper modifications to yield the bound 
\begin{align*}
\mfk{P}_{\cS}^{\ptb,3} \lesssim \delta^{2} \dist^2(\bF_{t},\bF_{\star})
\end{align*}
then complete the proof.


\section{Proof for Tensor Regression} \label{sec:proof_TR}

Before embarking on the proof, we state a useful lemma regarding TRIP
(cf.~Definition~\ref{def:TRIP}).

\begin{lemma}[{\cite[Lemma~E.7]{han2020optimal}}] \label{lemma:2r-TRIP}

Suppose that $\cA(\cdot)$ obeys the $2\br$-TRIP with a constant
$\delta_{2\br}$. Then for all $\bcX_{1},\bcX_{2}\in\RR^{n_{1}\times n_{2}\times n_{3}}$
of multilinear rank at most $\br$, one has 
\begin{align*}
\big|\langle\cA(\bcX_{1}),\cA(\bcX_{2})\rangle-\langle\bcX_{1},\bcX_{2}\rangle\big|\le\delta_{2\br}\|\bcX_{1}\|_{\fro}\|\bcX_{2}\|_{\fro},
\end{align*}
or equivalently, 
\begin{align*}
\big|\langle(\cA^{*}\cA-\cI)(\bcX_{1}),\bcX_{2}\rangle\big|\le\delta_{2\br}\|\bcX_{1}\|_{\fro}\|\bcX_{2}\|_{\fro}.
\end{align*}
\end{lemma} 

\subsection{Proof of local convergence (Lemma~\ref{lemma:contraction_TR})}

Given that $\dist(\bF_{t},\bF_{\star})\le\epsilon\sigma_{\min}(\bcX_{\star})$,
the conclusion $\|(\bU_{t},\bV_{t},\bW_{t})\bcdot\bcS_{t}-\bcX_{\star}\|_{\fro}\le3\dist(\bF_{t},\bF_{\star})$
directly follows from the relation~\eqref{eq:tensor2factors} in
Lemma~\ref{lemma:perturb_bounds}. Hence we will focus on controlling
$\dist(\bF_{t},\bF_{\star})$. 

As in the proof of Theorem~\ref{thm:TF},
we reuse the notations in \eqref{eq:short_notations} and \eqref{eq:additional_notation_aligned}, and the definition of $\dist(\bF_{t+1},\bF_{\star})$ to obtain 
\begin{align}
\dist^{2}(\bF_{t+1},\bF_{\star}) & \le\left\Vert (\bU_{t+1}\bQ_{t,1}-\bU_{\star})\bSigma_{\star,1}\right\Vert _{\fro}^{2}+\left\Vert (\bV_{t+1}\bQ_{t,2}-\bV_{\star})\bSigma_{\star,2}\right\Vert _{\fro}^{2}+\left\Vert (\bW_{t+1}\bQ_{t,3}-\bW_{\star})\bSigma_{\star,3}\right\Vert _{\fro}^{2}\nonumber \\
 & \qquad+\left\Vert (\bQ_{t,1}^{-1},\bQ_{t,2}^{-1},\bQ_{t,3}^{-1})\bcdot\bcS_{t+1}-\bcS_{\star}\right\Vert _{\fro}^{2}.\label{eq:TR_expand}
\end{align}
We shall bound each square in the right hand side of the bound~\eqref{eq:TR_expand} separately. The final result is 
\begin{align}
\dist^{2}(\bF_{t+1},\bF_{\star}) & \le(1-\eta)^{2}\left(\left\Vert \bDelta_{U}\bSigma_{\star,1}\right\Vert _{\fro}^{2}+\left\Vert \bDelta_{V}\bSigma_{\star,2}\right\Vert _{\fro}^{2}+\left\Vert \bDelta_{W}\bSigma_{\star,3}\right\Vert _{\fro}^{2}+\|\bDelta_{\cS}\|_{\fro}^{2}\right)\nonumber \\
 & \qquad-\eta(2-5\eta)\left\Vert \bcT_{U}+\bcT_{V}+\bcT_{W}\right\Vert _{\fro}^{2} -\eta(2-5\eta)\left(\left\Vert \bD_{U}\right\Vert _{\fro}^{2}+\left\Vert \bD_{V}\right\Vert _{\fro}^{2}+\left\Vert \bD_{W}\right\Vert _{\fro}^{2}\right)\nonumber \\
 & \qquad+2\eta(1-\eta)C(\epsilon+\delta_{2\br}+\delta_{2\br}^{2})\dist^{2}(\bF_{t},\bF_{\star})+\eta^{2}C(\epsilon+\delta_{2\br}+\delta_{2\br}^{2})\dist^{2}(\bF_{t},\bF_{\star}),\label{eq:TR_bound}
\end{align}
where $C>1$ is some universal constant. As long as $\eta\le2/5$,
and $\epsilon$, $\delta_{2\br}$ are sufficiently small constants, one reaches the desired
conclusion $\dist(\bF_{t+1},\bF_{\star})\le(1-0.6\eta)\dist(\bF_{t},\bF_{\star})$.

In the following subsections, we provide bounds on the four
terms in the right hand side of \eqref{eq:TR_expand}. In a nutshell, the bounds that are sought after are reminiscent of those established in \eqref{eq:final-bounds}, with additional perturbation terms introduced due to incomplete measurements, manifested via the TRIP parameter $\delta_{2\br}$. Once established, the claimed bound~\eqref{eq:TR_bound} easily follows. In light of the
symmetry among $\bU,\bV$, and $\bW$, we omit the control
of the terms pertaining to $\bV$ and $\bW$.

\subsubsection{Bounding the term pertaining to $\bU$}

The first term in \eqref{eq:TR_expand} is given by
\begin{align*}
(\bU_{t+1}\bQ_{t,1}-\bU_{\star})\bSigma_{\star,1} & =\Big(\bU-\eta\cM_{1}\left(\cA^{*}\cA((\bU,\bV,\bW)\bcdot\bcS-\bcX_{\star})\right)\breve{\bU}(\breve{\bU}^{\top}\breve{\bU})^{-1}-\bU_{\star}\Big)\bSigma_{\star,1}\\
 & =(1-\eta)\bDelta_{U}\bSigma_{\star,1}-\eta\bU_{\star}(\breve{\bU}-\breve{\bU}_{\star})^{\top}\breve{\bU}(\breve{\bU}^{\top}\breve{\bU})^{-1}\bSigma_{\star,1}\\
 & \qquad-\eta\cM_{1}\left((\cA^{*}\cA-\cI)((\bU,\bV,\bW)\bcdot\bcS-\bcX_{\star})\right)\breve{\bU}(\breve{\bU}^{\top}\breve{\bU})^{-1}\bSigma_{\star,1},
\end{align*}
where we separate the population term from the perturbation term.
Take the squared norm of both sides to see 
\begin{align*}
 & \left\Vert (\bU_{t+1}\bQ_{t,1}-\bU_{\star})\bSigma_{\star,1}\right\Vert _{\fro}^{2}=\underbrace{\left\Vert (1-\eta)\bDelta_{U}\bSigma_{\star,1}-\eta\bU_{\star}(\breve{\bU}-\breve{\bU}_{\star})^{\top}\breve{\bU}(\breve{\bU}^{\top}\breve{\bU})^{-1}\bSigma_{\star,1}\right\Vert _{\fro}^{2}}_{\eqqcolon\mfk{R}_{U}^{\main}}\\
 & \qquad-2\eta(1-\eta)\underbrace{\left\langle \bDelta_{U}\bSigma_{\star,1},\cM_{1}\left((\cA^{*}\cA-\cI)((\bU,\bV,\bW)\bcdot\bcS-\bcX_{\star})\right)\breve{\bU}(\breve{\bU}^{\top}\breve{\bU})^{-1}\bSigma_{\star,1}\right\rangle }_{\eqqcolon\mfk{R}_{U}^{\ptb,1}}\\
 & \qquad+2\eta^{2}\underbrace{\left\langle \bU_{\star}(\breve{\bU}-\breve{\bU}_{\star})^{\top}\breve{\bU}(\breve{\bU}^{\top}\breve{\bU})^{-1}\bSigma_{\star,1},\cM_{1}\left((\cA^{*}\cA-\cI)((\bU,\bV,\bW)\bcdot\bcS-\bcX_{\star})\right)\breve{\bU}(\breve{\bU}^{\top}\breve{\bU})^{-1}\bSigma_{\star,1}\right\rangle }_{\eqqcolon\mfk{R}_{U}^{\ptb,2}}\\
 & \qquad+\eta^{2}\underbrace{\left\Vert \cM_{1}\left((\cA^{*}\cA-\cI)((\bU,\bV,\bW)\bcdot\bcS-\bcX_{\star})\right)\breve{\bU}(\breve{\bU}^{\top}\breve{\bU})^{-1}\bSigma_{\star,1}\right\Vert _{\fro}^{2}}_{\eqqcolon\mfk{R}_{U}^{\ptb,3}}.
\end{align*}
The main term $\mfk{R}_{U}^{\main}$ has been handled in Section~\ref{sec:proof_TF}; see \eqref{eq:first_square} and the bound~\eqref{eq:U-bound}.
In the sequel, we shall bound the three perturbation terms.

\paragraph{Step 1: bounding $\mfk{R}_{U}^{\ptb,1}$.}

Use the definition of $\breve{\bU}$, we can translate the inner product in the matrix space to that in the tensor space
\begin{align*}
\mfk{R}_{U}^{\ptb,1} & =\left\langle \left(\bDelta_{U}\bSigma_{\star,1}^{2}(\breve{\bU}^{\top}\breve{\bU})^{-1},\bV,\bW\right)\bcdot\bcS,(\cA^{*}\cA-\cI)((\bU,\bV,\bW)\bcdot\bcS-\bcX_{\star})\right\rangle \\
 & =\left\langle \left(\bDelta_{U}\bSigma_{\star,1}^{2}(\breve{\bU}^{\top}\breve{\bU})^{-1},\bV,\bW\right)\bcdot\bcS,(\cA^{*}\cA-\cI)((\bU,\bV,\bW)\bcdot\bDelta_{\cS})\right\rangle \\
 & \qquad+\left\langle \left(\bDelta_{U}\bSigma_{\star,1}^{2}(\breve{\bU}^{\top}\breve{\bU})^{-1},\bV,\bW\right)\bcdot\bcS,(\cA^{*}\cA-\cI)((\bDelta_{U},\bV,\bW)\bcdot\bcS_{\star})\right\rangle \\
 & \qquad+\left\langle \left(\bDelta_{U}\bSigma_{\star,1}^{2}(\breve{\bU}^{\top}\breve{\bU})^{-1},\bV,\bW\right)\bcdot\bcS,(\cA^{*}\cA-\cI)((\bU_{\star},\bDelta_{V},\bW)\bcdot\bcS_{\star})\right\rangle \\
 & \qquad+\left\langle \left(\bDelta_{U}\bSigma_{\star,1}^{2}(\breve{\bU}^{\top}\breve{\bU})^{-1},\bV,\bW\right)\bcdot\bcS,(\cA^{*}\cA-\cI)((\bU_{\star},\bV_{\star},\bDelta_{W})\bcdot\bcS_{\star})\right\rangle ,
\end{align*}
where the second relation uses the decomposition \eqref{eq:decomp_T}.
Apply Lemma~\ref{lemma:2r-TRIP} to each of the four terms to obtain 
\begin{align*}
|\mfk{R}_{U}^{\ptb,1}| & \le\delta_{2\br}\left\Vert \left(\bDelta_{U}\bSigma_{\star,1}^{2}(\breve{\bU}^{\top}\breve{\bU})^{-1},\bV,\bW\right)\bcdot\bcS\right\Vert _{\fro}\\
 & \qquad\left(\left\Vert (\bU,\bV,\bW)\bcdot\bDelta_{\cS})\right\Vert _{\fro}+\left\Vert (\bDelta_{U},\bV,\bW)\bcdot\bcS_{\star})\right\Vert _{\fro}+\left\Vert (\bU_{\star},\bDelta_{V},\bW)\bcdot\bcS_{\star})\right\Vert _{\fro}+\left\Vert (\bU_{\star},\bV_{\star},\bDelta_{W})\bcdot\bcS_{\star})\right\Vert _{\fro}\right).
\end{align*}
For the prefactor, we have 
\begin{align*}
\left\Vert \left(\bDelta_{U}\bSigma_{\star,1}^{2}(\breve{\bU}^{\top}\breve{\bU})^{-1},\bV,\bW\right)\bcdot\bcS\right\Vert _{\fro} & =\left\Vert \bDelta_{U}\bSigma_{\star,1}^{2}(\breve{\bU}^{\top}\breve{\bU})^{-1}\breve{\bU}^{\top}\right\Vert _{\fro}\\
 & \le\|\bDelta_{U}\bSigma_{\star,1}\|_{\fro}\left\Vert \breve{\bU}(\breve{\bU}^{\top}\breve{\bU})^{-1}\bSigma_{\star,1}\right\Vert _{\op}\\
 & \le\|\bDelta_{U}\bSigma_{\star,1}\|_{\fro}(1-\epsilon)^{-3},
\end{align*}
where the last step arises from Lemma~\ref{lemma:perturb_bounds}.
In addition, the same argument as in \eqref{eq:perturb_T_fro} yields
\begin{align*}
 & \left\Vert (\bU,\bV,\bW)\bcdot\bDelta_{\cS})\right\Vert _{\fro}+\left\Vert (\bDelta_{U},\bV,\bW)\bcdot\bcS_{\star})\right\Vert _{\fro}+\left\Vert (\bU_{\star},\bDelta_{V},\bW)\bcdot\bcS_{\star})\right\Vert _{\fro}+\left\Vert (\bU_{\star},\bV_{\star},\bDelta_{W})\bcdot\bcS_{\star})\right\Vert _{\fro}\\
 & \qquad\le(1+\frac{3}{2}\epsilon+\epsilon^{2}+\frac{1}{4}\epsilon^{3})\left(\|\bDelta_{U}\bSigma_{\star,1}\|_{\fro}+\|\bDelta_{V}\bSigma_{\star,2}\|_{\fro}+\|\bDelta_{W}\bSigma_{\star,3}\|_{\fro}+\|\bDelta_{\cS}\|_{\fro}\right).
\end{align*}
Take the previous two bounds collectively to arrive at 
\begin{align*}
|\mfk{R}_{U,p1}| & \le\delta_{2\br}\frac{1+\frac{3}{2}\epsilon+\epsilon^{2}+\frac{1}{4}\epsilon^{3}}{(1-\epsilon)^{3}}\|\bDelta_{U}\bSigma_{\star,1}\|_{\fro}\left(\|\bDelta_{U}\bSigma_{\star,1}\|_{\fro}+\|\bDelta_{V}\bSigma_{\star,2}\|_{\fro}+\|\bDelta_{W}\bSigma_{\star,3}\|_{\fro}+\|\bDelta_{\cS}\|_{\fro}\right) \\
&\lesssim \delta_{2\br} \dist^2(\bF_{t},\bF_{\star}), 
\end{align*}
with the proviso that $\epsilon$ is small enough.

\paragraph{Step 2: bounding $\mfk{R}_{U}^{\ptb,2}$.}

Rewrite the inner product in the tensor space to see 
\begin{align*}
\mfk{R}_{U}^{\ptb,2}=\left\langle \left(\bU_{\star}(\breve{\bU}-\breve{\bU}_{\star})^{\top}\breve{\bU}(\breve{\bU}^{\top}\breve{\bU})^{-1}\bSigma_{\star,1}^{2}(\breve{\bU}^{\top}\breve{\bU})^{-1},\bV,\bW\right)\bcdot\bcS,(\cA^{*}\cA-\cI)((\bU,\bV,\bW)\cdot\bcS-\bcX_{\star})\right\rangle .
\end{align*}
Similar to the control of $\mfk{R}_{U}^{\ptb,1}$, we have 
\begin{align*}
|\mfk{R}_{U}^{\ptb,2}| & \le\delta_{2\br}\left\Vert \bU_{\star}(\breve{\bU}-\breve{\bU}_{\star})^{\top}\breve{\bU}(\breve{\bU}^{\top}\breve{\bU})^{-1}\bSigma_{\star,1}^{2}(\breve{\bU}^{\top}\breve{\bU})^{-1}\breve{\bU}^{\top}\right\Vert _{\fro}\\
 & \qquad(1+\frac{3}{2}\epsilon+\epsilon^{2}+\frac{1}{4}\epsilon^{3})\left(\|\bDelta_{U}\bSigma_{\star,1}\|_{\fro}+\|\bDelta_{V}\bSigma_{\star,2}\|_{\fro}+\|\bDelta_{W}\bSigma_{\star,3}\|_{\fro}+\|\bDelta_{\cS}\|_{\fro}\right).
\end{align*}
For the prefactor, we can use \eqref{eq:perturb_Rinv} and \eqref{eq:perturb_R_fro}
to obtain 
\begin{align*}
 \left\Vert \bU_{\star}(\breve{\bU}-\breve{\bU}_{\star})^{\top}\breve{\bU}(\breve{\bU}^{\top}\breve{\bU})^{-1}\bSigma_{\star,1}^{2}(\breve{\bU}^{\top}\breve{\bU})^{-1}\breve{\bU}^{\top}\right\Vert _{\fro}  & \le\|\breve{\bU}-\breve{\bU}_{\star}\|_{\fro}\left\Vert \breve{\bU}(\breve{\bU}^{\top}\breve{\bU})^{-1}\bSigma_{\star,1}\right\Vert _{\op}^{2}\\
 & \le\frac{1+\epsilon+\frac{1}{3}\epsilon^{2}}{(1-\epsilon)^{6}}\left(\|\bDelta_{V}\bSigma_{\star,2}\|_{\fro}+\|\bDelta_{W}\bSigma_{\star,3}\|_{\fro}+\|\bDelta_{\cS}\|_{\fro}\right),
\end{align*}
which further implies
\begin{align*}
|\mfk{R}_{U}^{\ptb,2}| & \le\delta_{2\br}\frac{(1+\frac{3}{2}\epsilon+\epsilon^{2}+\frac{1}{4}\epsilon^{3})(1+\epsilon+\frac{1}{3}\epsilon^{2})}{(1-\epsilon)^{6}}\left(\|\bDelta_{V}\bSigma_{\star,2}\|_{\fro}+\|\bDelta_{W}\bSigma_{\star,3}\|_{\fro}+\|\bDelta_{\cS}\|_{\fro}\right)\\
 & \qquad\left(\|\bDelta_{U}\bSigma_{\star,1}\|_{\fro}+\|\bDelta_{V}\bSigma_{\star,2}\|_{\fro}+\|\bDelta_{W}\bSigma_{\star,3}\|_{\fro}+\|\bDelta_{\cS}\|_{\fro}\right) \\
 &\lesssim \delta_{2\br} \dist^2(\bF_{t},\bF_{\star}),
\end{align*}
as long as $\epsilon$ is sufficiently small.

\paragraph{Step 3: bounding $\mfk{R}_{U}^{\ptb,3}$.}

The last perturbation term needs special care. We first use the variational
representation of the Frobenius norm to write 
\begin{align*}
\sqrt{\mfk{R}_{U}^{\ptb,3}}=\left\langle \left(\widetilde{\bU}\bSigma_{\star,1}(\breve{\bU}^{\top}\breve{\bU})^{-1},\bV,\bW\right)\bcdot\bcS,(\cA^{*}\cA-\cI)((\bU,\bV,\bW)\bcdot\bcS-\bcX_{\star})\right\rangle 
\end{align*}
for some $\widetilde{\bU}\in\RR^{n_{1}\times r_{1}}$ obeying $\|\widetilde{\bU}\|_{\fro}=1$.
Repeat the same argument as used in controlling $\mfk{R}_{U}^{\ptb,1}$
to see 
\begin{align*}
\sqrt{\mfk{R}_{U}^{\ptb,3}} & \le\delta_{2\br}\left\Vert \widetilde{\bU}\bSigma_{\star,1}(\breve{\bU}^{\top}\breve{\bU})^{-1}\breve{\bU}^{\top}\right\Vert _{\fro}(1+\frac{3}{2}\epsilon+\epsilon^{2}+\frac{1}{4}\epsilon^{3})\left(\|\bDelta_{U}\bSigma_{\star,1}\|_{\fro}+\|\bDelta_{V}\bSigma_{\star,2}\|_{\fro}+\|\bDelta_{W}\bSigma_{\star,3}\|_{\fro}+\|\bDelta_{\cS}\|_{\fro}\right)\\
 & \le\delta_{2\br}\frac{1+\frac{3}{2}\epsilon+\epsilon^{2}+\frac{1}{4}\epsilon^{3}}{(1-\epsilon)^{3}}\left(\|\bDelta_{U}\bSigma_{\star,1}\|_{\fro}+\|\bDelta_{V}\bSigma_{\star,2}\|_{\fro}+\|\bDelta_{W}\bSigma_{\star,3}\|_{\fro}+\|\bDelta_{\cS}\|_{\fro}\right),
\end{align*}
where the last line uses the bound~\eqref{eq:perturb_Rinv} in Lemma~\ref{lemma:perturb_bounds}. Then take the square on both sides to conclude 
\begin{align*}
\mfk{R}_{U}^{\ptb,3}&\le\delta_{2\br}^{2}\frac{(1+\frac{3}{2}\epsilon+\epsilon^{2}+\frac{1}{4}\epsilon^{3})^{2}}{(1-\epsilon)^{6}}\left(\|\bDelta_{U}\bSigma_{\star,1}\|_{\fro}+\|\bDelta_{V}\bSigma_{\star,2}\|_{\fro}+\|\bDelta_{W}\bSigma_{\star,3}\|_{\fro}+\|\bDelta_{\cS}\|_{\fro}\right)^{2} \\
&\lesssim \delta_{2\br}^{2} \dist^2(\bF_{t},\bF_{\star})
\end{align*}
as long as $\epsilon$ is sufficiently small.

\subsubsection{Bounding the term pertaining to $\bcS$}

The last term of \eqref{eq:TR_expand} can be rewritten as 
\begin{align*}
 & (\bQ_{t,1}^{-1},\bQ_{t,2}^{-1},\bQ_{t,3}^{-1})\bcdot\bcS_{t+1}-\bcS_{\star}\\
 & \quad=\bcS-\eta\left((\bU^{\top}\bU)^{-1}\bU^{\top},(\bV^{\top}\bV)^{-1}\bV^{\top},(\bW^{\top}\bW)^{-1}\bW^{\top}\right)\bcdot\cA^{*}\cA\left((\bU,\bV,\bW)\bcdot\bcS-\bcX_{\star}\right)-\bcS_{\star}\\
 & \quad=(1-\eta)\bDelta_{\cS}-\eta\left((\bU^{\top}\bU)^{-1}\bU^{\top},(\bV^{\top}\bV)^{-1}\bV^{\top},(\bW^{\top}\bW)^{-1}\bW^{\top}\right)\bcdot\left((\bU,\bV,\bW)\bcdot\bcS_{\star}-\bcX_{\star}\right)\\
 & \qquad-\eta\left((\bU^{\top}\bU)^{-1}\bU^{\top},(\bV^{\top}\bV)^{-1}\bV^{\top},(\bW^{\top}\bW)^{-1}\bW^{\top}\right)\bcdot(\cA^{*}\cA-\cI)((\bU,\bV,\bW)\bcdot\bcS-\bcX_{\star}),
\end{align*}
which further gives 
\begin{align*}
 & \left\Vert (\bQ_{t,1}^{-1},\bQ_{t,2}^{-1},\bQ_{t,3}^{-1})\bcdot\bcS_{t+1}-\bcS_{\star}\right\Vert _{\fro}^{2}\\
 & \quad=\underbrace{\left\Vert (1-\eta)\bDelta_{\cS}-\eta\left((\bU^{\top}\bU)^{-1}\bU^{\top},(\bV^{\top}\bV)^{-1}\bV^{\top},(\bW^{\top}\bW)^{-1}\bW^{\top}\right)\bcdot\left((\bU,\bV,\bW)\bcdot\bcS_{\star}-\bcX_{\star}\right)\right\Vert _{\fro}^{2}}_{\eqqcolon \mfk{R}_{\cS}^{\main}}\\
 & \qquad-2\eta(1-\eta)\underbrace{\left\langle \bDelta_{\cS},\left((\bU^{\top}\bU)^{-1}\bU^{\top},(\bV^{\top}\bV)^{-1}\bV^{\top},(\bW^{\top}\bW)^{-1}\bW^{\top}\right)\bcdot(\cA^{*}\cA-\cI)((\bU,\bV,\bW)\bcdot\bcS-\bcX_{\star})\right\rangle }_{\eqqcolon\mfk{R}_{\cS}^{\ptb,1}}\\
 & \qquad+2\eta^{2}\underbrace{\begin{aligned} & \Big\langle\left((\bU^{\top}\bU)^{-1}\bU^{\top},(\bV^{\top}\bV)^{-1}\bV^{\top},(\bW^{\top}\bW)^{-1}\bW^{\top}\right)\bcdot\left((\bU,\bV,\bW)\bcdot\bcS_{\star}-\bcX_{\star}\right),\\
 & \qquad\left((\bU^{\top}\bU)^{-1}\bU^{\top},(\bV^{\top}\bV)^{-1}\bV^{\top},(\bW^{\top}\bW)^{-1}\bW^{\top}\right)\bcdot(\cA^{*}\cA-\cI)((\bU,\bV,\bW)\bcdot\bcS-\bcX_{\star} )\Big\rangle
\end{aligned}
}_{\eqqcolon\mfk{R}_{\cS}^{\ptb,2}}\\
 & \qquad+\eta^{2}\underbrace{\left\Vert \left((\bU^{\top}\bU)^{-1}\bU^{\top},(\bV^{\top}\bV)^{-1}\bV^{\top},(\bW^{\top}\bW)^{-1}\bW^{\top}\right)\bcdot(\cA^{*}\cA-\cI)((\bU,\bV,\bW)\bcdot\bcS-\bcX_{\star})\right\Vert _{\fro}^{2}}_{\eqqcolon\mfk{R}_{\cS}^{\ptb,3}}.
\end{align*}
Note that the main term $\mfk{R}_{\cS}^{\main}$ has already been characterized in Section~\ref{sec:proof_TF}; see \eqref{eq:last_square} and the bound~\eqref{eq:S-bound}. Therefore we concentrate on the
remaining perturbation terms. 

\paragraph{Step 1: bounding $\mfk{R}_{\cS}^{\ptb,1}$.}

Use the property \eqref{eq:tensor_properties_d} to write $\mfk{R}_{\cS}^{\ptb,1}$ as 
\begin{align*}
\mfk{R}_{\cS}^{\ptb,1}=\left\langle \left(\bU(\bU^{\top}\bU)^{-1},\bV(\bV^{\top}\bV)^{-1},\bW(\bW^{\top}\bW)^{-1}\right)\bcdot\bDelta_{\cS},(\cA^{*}\cA-\cI)((\bU,\bV,\bW)\bcdot\bcS-\bcX_{\star})\right\rangle .
\end{align*}
We can use the decomposition \eqref{eq:decomp_T} and Lemma~\ref{lemma:2r-TRIP}
to derive 
\begin{align*}
|\mfk{R}_{\cS}^{\ptb,1}| & \le\delta_{2\br}\left\Vert \left(\bU(\bU^{\top}\bU)^{-1},\bV(\bV^{\top}\bV)^{-1},\bW(\bW^{\top}\bW)^{-1}\right)\bcdot\bDelta_{\cS}\right\Vert _{\fro}\\
 & \qquad(1+\frac{3}{2}\epsilon+\epsilon^{2}+\frac{1}{4}\epsilon^{3})\left(\|\bDelta_{U}\bSigma_{\star,1}\|_{\fro}+\|\bDelta_{V}\bSigma_{\star,2}\|_{\fro}+\|\bDelta_{W}\bSigma_{\star,3}\|_{\fro}+\|\bDelta_{\cS}\|_{\fro}\right).
\end{align*}
In addition, Lemma \ref{lemma:perturb_bounds} tells us that 
\begin{align*}
 & \left\Vert \left(\bU(\bU^{\top}\bU)^{-1},\bV(\bV^{\top}\bV)^{-1},\bW(\bW^{\top}\bW)^{-1}\right)\bcdot\bDelta_{\cS}\right\Vert _{\fro}\\
 & \qquad\le\left\Vert \bU(\bU^{\top}\bU)^{-1}\right\Vert _{\op}\left\Vert \bV(\bV^{\top}\bV)^{-1}\right\Vert _{\op}\left\Vert \bW(\bW^{\top}\bW)^{-1}\right\Vert _{\op}\|\bDelta_{\cS}\|_{\fro} \le(1-\epsilon)^{-3}\|\bDelta_{\cS}\|_{\fro}.
\end{align*}
Combine the above two bounds to reach
\begin{align*}
|\mfk{R}_{\cS}^{\ptb,1}| &\le \delta_{2\br}\frac{1+\frac{3}{2}\epsilon+\epsilon^{2}+\frac{1}{4}\epsilon^{3}}{(1-\epsilon)^{3}}\|\bDelta_{\cS}\|_{\fro}\left(\|\bDelta_{U}\bSigma_{\star,1}\|_{\fro}+\|\bDelta_{V}\bSigma_{\star,2}\|_{\fro}+\|\bDelta_{W}\bSigma_{\star,3}\|_{\fro}+\|\bDelta_{\cS}\|_{\fro}\right) \\
&\lesssim \delta_{2\br} \dist^2(\bF_{t},\bF_{\star})
\end{align*}
as long as $\epsilon$ is a sufficiently small constant.

\paragraph{Step 2: bounding $\mfk{R}_{\cS}^{\ptb,2}$.}

Similarly, we can bound $\mfk{R}_{\cS}^{\ptb,2}$ by 
\begin{align*}
|\mfk{R}_{\cS}^{\ptb,2}| & \le\delta_{2\br}\left\Vert \left(\bU(\bU^{\top}\bU)^{-2}\bU^{\top},\bV(\bV^{\top}\bV)^{-2}\bV^{\top},\bW(\bW^{\top}\bW)^{-2}\bW^{\top}\right)\bcdot\left((\bU,\bV,\bW)\bcdot\bcS_{\star}-\bcX_{\star}\right)\right\Vert _{\fro}\\
 & \qquad(1+\frac{3}{2}\epsilon+\epsilon^{2}+\frac{1}{4}\epsilon^{3})\left(\|\bDelta_{U}\bSigma_{\star,1}\|_{\fro}+\|\bDelta_{V}\bSigma_{\star,2}\|_{\fro}+\|\bDelta_{W}\bSigma_{\star,3}\|_{\fro}+\|\bDelta_{\cS}\|_{\fro}\right)\\
 & \le\delta_{2\br}\frac{(1+\epsilon+\frac{1}{3}\epsilon^{2})(1+\frac{3}{2}\epsilon+\epsilon^{2}+\frac{1}{4}\epsilon^{3})}{(1-\epsilon)^{6}}\left(\|\bDelta_{U}\bSigma_{\star,1}\|_{\fro}+\|\bDelta_{V}\bSigma_{\star,2}\|_{\fro}+\|\bDelta_{W}\bSigma_{\star,3}\|_{\fro}\right)\\
 & \qquad\left(\|\bDelta_{U}\bSigma_{\star,1}\|_{\fro}+\|\bDelta_{V}\bSigma_{\star,2}\|_{\fro}+\|\bDelta_{W}\bSigma_{\star,3}\|_{\fro}+\|\bDelta_{\cS}\|_{\fro}\right) \\
 &\lesssim \delta_{2\br} \dist^2(\bF_{t},\bF_{\star}).
\end{align*}

\paragraph{Step 3: bounding $\mfk{R}_{\cS}^{\ptb,3}$.}

Apply the variational representation of the Frobenius norm to write
\begin{align*}
\sqrt{\mfk{R}_{\cS}^{\ptb,3}}=\left\langle \left(\bU(\bU^{\top}\bU)^{-1},\bV(\bV^{\top}\bV)^{-1},\bW(\bW^{\top}\bW)^{-1}\right)\bcdot\widetilde{\bcS},(\cA^{*}\cA-\cI)((\bU,\bV,\bW)\bcdot\bcS-\bcX_{\star})\right\rangle 
\end{align*}
for some $\widetilde{\bcS}\in\RR^{r_{1}\times r_{2}\times r_{3}}$
obeying $\|\widetilde{\bcS}\|_{\fro}=1$. Repeat the same argument as in bounding $\mfk{R}_{\bU}^{\ptb,3}$ 
to see 
\begin{align*}
\sqrt{\mfk{R}_{\cS}^{\ptb,3}} & \le\delta_{2\br}\left\Vert \left(\bU(\bU^{\top}\bU)^{-1},\bV(\bV^{\top}\bV)^{-1},\bW(\bW^{\top}\bW)^{-1}\right)\bcdot\widetilde{\bcS}\right\Vert _{\fro}\\
 & \qquad(1+\frac{3}{2}\epsilon+\epsilon^{2}+\frac{1}{4}\epsilon^{3})\left(\|\bDelta_{U}\bSigma_{\star,1}\|_{\fro}+\|\bDelta_{V}\bSigma_{\star,2}\|_{\fro}+\|\bDelta_{W}\bSigma_{\star,3}\|_{\fro}+\|\bDelta_{\cS}\|_{\fro}\right)\\
 & \le\delta_{2\br}\frac{1+\frac{3}{2}\epsilon+\epsilon^{2}+\frac{1}{4}\epsilon^{3}}{(1-\epsilon)^{3}}\left(\|\bDelta_{U}\bSigma_{\star,1}\|_{\fro}+\|\bDelta_{V}\bSigma_{\star,2}\|_{\fro}+\|\bDelta_{W}\bSigma_{\star,3}\|_{\fro}+\|\bDelta_{\cS}\|_{\fro}\right).
\end{align*}
Then take the square on both sides to conclude 
\begin{align*}
\mfk{R}_{\cS}^{\ptb,3}& \le\delta_{2\br}^{2}\frac{(1+\frac{3}{2}\epsilon+\epsilon^{2}+\frac{1}{4}\epsilon^{3})^{2}}{(1-\epsilon)^{6}}\left(\|\bDelta_{U}\bSigma_{\star,1}\|_{\fro}+\|\bDelta_{V}\bSigma_{\star,2}\|_{\fro}+\|\bDelta_{W}\bSigma_{\star,3}\|_{\fro}+\|\bDelta_{\cS}\|_{\fro}\right)^{2} \\
&\lesssim \delta_{2\br}^2 \dist^2(\bF_{t},\bF_{\star}).
\end{align*}

\subsection{Proof of spectral initialization (Lemma~\ref{lemma:init_TR}) }
\label{proof:lemma_init_TR}
In view of Lemma~\ref{lemma:Procrustes}, we can relate $\dist(\bF_{0},\bF_{\star})$
to $\|(\bU_{0},\bV_{0},\bW_{0})\bcdot\bcS_{0}-\bcX_{\star}\|_{\fro}$
as 
\begin{align*}
\dist(\bF_{0},\bF_{\star})\le(\sqrt{2}+1)^{3/2}\left\Vert (\bU_{0},\bV_{0},\bW_{0})\bcdot\bcS_{0}-\bcX_{\star}\right\Vert _{\fro}.
\end{align*}
To proceed, we need to control $\left\Vert (\bU_{0},\bV_{0},\bW_{0})\bcdot\bcS_{0}-\bcX_{\star}\right\Vert _{\fro}$,
where $(\bU_{0},\bV_{0},\bW_{0})\bcdot\bcS_{0}$ is the output of
HOSVD. Similar results have been established in \cite{zhang2021low,han2020optimal,zhang2020islet}, which involve sophisticated subspace perturbation bounds. For conciseness and completeness, we provide an alternative proof directly tackling the distance. 

Define $\bP_{U}\coloneqq\bU_{0}\bU_{0}^{\top}$ as the projection matrix onto the column space of $\bU_{0}$, $\bP_{U_\perp}\coloneqq\bI_{n_{1}}-\bP_{U}$ as the projection onto its orthogonal complement, and define $\bP_{V},\bP_{V_\perp}, \bP_{W},\bP_{W_\perp}$ analogously. Similar to \eqref{eq:TC_init_expand}, we have the decomposition
\begin{align}
 & \left\|(\bU_{0},\bV_{0},\bW_{0})\bcdot\bcS_{0}-\bcX_{\star}\right\|_{\fro}^{2} \nonumber \\
 &\qquad \le \left\|(\bP_{U},\bP_{V},\bP_{W})\bcdot(\bcY-\bcX_{\star})\right\|_{\fro}^{2} + \left\|\bP_{U_\perp}\cM_{1}(\bcX_{\star})\right\|_{\fro}^{2} + \left\|\bP_{V_\perp}\cM_{2}(\bcX_{\star})\right\|_{\fro}^{2} + \left\|\bP_{W_\perp}\cM_{3}(\bcX_{\star})\right\|_{\fro}^{2}.\label{eq:TR_init_expand}
\end{align}
Below we bound the terms on the right hand side of \eqref{eq:TR_init_expand} in order. 

\paragraph{Bounding $\left\|(\bP_{U},\bP_{V},\bP_{W})\bcdot(\bcY-\bcX_{\star})\right\|_{\fro}$.}
For the first term in the upper bound~\eqref{eq:TR_init_expand}, apply the variational representation of the Frobenius norm to write
\begin{align*}
\left\|(\bP_{U},\bP_{V},\bP_{W})\bcdot(\bcY-\bcX_{\star})\right\|_{\fro} = \left\langle (\bP_{U},\bP_{V},\bP_{W})\bcdot(\bcY-\bcX_{\star}), \widetilde{\bcT}\right\rangle =  \left\langle (\cA^{*}\cA-\cI)\bcX_{\star}, (\bP_{U},\bP_{V},\bP_{W})\bcdot\widetilde{\bcT}\right\rangle,
\end{align*}
for some $\widetilde{\bcT}\in\RR^{n_1\times n_3\times n_3}$ obeying $\big\|\widetilde{\bcT} \big\|_{\fro}=1$, where the last equality follows from \eqref{eq:tensor_properties_d}. Under the Gaussian design, we know from \cite[Theorem~2]{rauhut2017low} that $\cA(\cdot)$ obeys $2\br$-TRIP with a constant $\delta_{2\br} \asymp \sqrt{\frac{nr + r^3}{m}}$. Therefore we can apply Lemma~\ref{lemma:2r-TRIP} to obtain
\begin{align*}
\left\|(\bP_{U},\bP_{V},\bP_{W})\bcdot(\bcY-\bcX_{\star})\right\|_{\fro} &\le \delta_{2\br} \|\bcX_{\star}\|_{\fro} \big\|(\bP_{U},\bP_{V},\bP_{W})\bcdot\widetilde{\bcT}\big\|_{\fro} \le \delta_{2\br} \|\bcX_{\star}\|_{\fro} \\
&\lesssim \sqrt{\frac{nr + r^3}{m}}\|\bcX_{\star}\|_{\fro} \le \sqrt{\frac{nr^2 + r^4}{m}}\kappa\sigma_{\min}(\bcX_{\star}).
\end{align*}

\paragraph{Bounding $\left\|\bP_{U_\perp}\cM_{1}(\bcX_{\star})\right\|_{\fro}$. }

For the second term in \eqref{eq:TR_init_expand}, first bound it by 
\begin{align*}
\left\|\bP_{U_\perp}\cM_{1}(\bcX_{\star})\right\|_{\fro} \le \frac{\sqrt{r_1}}{\sigma_{\min}(\bcX_{\star})} \left\|\bP_{U_\perp}\cM_{1}(\bcX_{\star})\cM_{1}(\bcX_{\star})^{\top}\right\|_{\op},
\end{align*}
where we use the facts that $\bP_{U_\perp}\cM_{1}(\bcX_{\star})$ has rank at most $r_1$ and $\|\bA\bB\|_{\op}\ge\|\bA\|_{\op}\sigma_{\min}(\bB)$. For notation simplicity, we abbreviate 
\begin{align*}
\bG \coloneqq \cM_{1}(\cA^*(\by))\cM_{1}(\cA^*(\by))^{\top} - \frac{\|\by\|_{2}^2}{m}(n_2n_3-r_1)\bI_{n_1}, \quad\mbox{and}\quad \bG_{\star} \coloneqq \cM_{1}(\bcX_{\star})\cM_{1}(\bcX_{\star})^{\top}.
\end{align*} 
We claim for the moment that with overwhelming probability that 
\begin{align}
\|\bG-\bG_{\star}\|_{\op} \lesssim \frac{\sqrt{n_1n_2n_3} + n\log n}{m}\|\bcX_{\star}\|_{\fro}^2 + \sqrt{\frac{n\log n}{m}}\|\bcX_{\star}\|_{\fro}\sigma_{\max}(\bcX_{\star}), \label{eq:perturb_G}
\end{align}
whose proof is deferred to Appendix~\ref{proof:perturb_G}. Under the sample size condition
\begin{align*}
m \gtrsim \epsilon_{0}^{-1}\sqrt{n_1n_2n_3}r^{3/2}\kappa^2 + \epsilon_{0}^{-2}(nr^2\kappa^4\log n + r^{4}\kappa^{2})
\end{align*}
for some small constant $\epsilon_{0}$, we have $\|\bG-\bG_{\star}\|_{\op}\le \epsilon_{0}\sigma_{\min}^{2}(\bcX_{\star})$, which implies that $\bG$ is positive semi-definite. Therefore, the top-$r_1$ eigenvectors of $\bG$ coincide with $\bU_{0}$, the top-$r_1$ left singular vectors of $\cM_{1}(\cA^*(\by))$, which implies $\|\bP_{U_\perp}\bG\|_{\op}=\sigma_{r_1+1}(\bG)$. By the triangle inequality, we obtain
\begin{align*}
\left\|\bP_{U_\perp}\bG_{\star}\right\|_{\op} &\le \left\|\bP_{U_\perp}\left(\bG - \bG_{\star}\right)\right\|_{\op} + \left\|\bP_{U_\perp}\bG\right\|_{\op}  \le \left\|\bG - \bG_{\star}\right\|_{\op} + \sigma_{r_1+1}(\bG) \\
&\le \left\|\bG - \bG_{\star}\right\|_{\op} + \sigma_{r_1+1}(\bG_{\star}) + \left\|\bG - \bG_{\star}\right\|_{\op} = 2\left\|\bG - \bG_{\star}\right\|_{\op},
\end{align*}
where the second line follows from Weyl's inequality and that $\bG_{\star}$ has rank $r_1$. In total, the second term of~\eqref{eq:TR_init_expand} is bounded by
\begin{align*}
\left\|\bP_{U_\perp}\cM_{1}(\bcX_{\star})\right\|_{\fro} \le \frac{2\sqrt{r_1}}{\sigma_{\min}(\bcX_{\star})}\|\bG-\bG_{\star}\|_{\op} \lesssim \left(\frac{(\sqrt{n_1n_2n_3} + n\log n)r^{3/2}}{m} + \sqrt{\frac{nr^2\log n}{m}}\right)\kappa^2\sigma_{\min}(\bcX_{\star}).
\end{align*}

\paragraph{Completing the proof.} The third and fourth terms of \eqref{eq:TR_init_expand} can be bounded similarly. In all, we conclude that
\begin{align*}
\dist(\bF_{0},\bF_{\star}) \le (\sqrt{2}+1)^{3/2}\left\|(\bU_{0},\bV_{0},\bW_{0})\bcdot\bcS_{0}-\bcX_{\star}\right\|_{\fro} \le \epsilon_{0}\sigma_{\min}(\bcX_{\star})
\end{align*}
under the assumed sample size.

\subsubsection{Proof of \eqref{eq:perturb_G}}
\label{proof:perturb_G}
We start with stating a few useful concentration inequalities.
\begin{lemma}\label{lemma:Gaussian_bounds} Suppose that $\bA_i\in\RR^{n_1\times n_2}$ has i.i.d.~$\cN(0,1/m)$ entries, and $\by_i=\langle\bA_i,\bX\rangle$ for a fixed $\bX\in\RR^{n_1\times n_2}$, $i=1,\dots,m$. Further suppose that $\bB\in\RR^{n_1\times n_2}$ has i.i.d.~$\cN(0,\sigma^2)$ entries. Then there exists a universal constant $C > 0$ such that for any $t > 0$, the following concentration inequalities hold:
\begin{enumerate}
\item Gaussian ensemble \cite[Lemma~4]{zhang2020islet}:
\begin{align}
\PP\left(\Big\|\sum_{i=1}^m y_i\bA_i - \bX\Big\|_{\op} \ge C \|\bX\|_{\fro}\sqrt{n_1+n_2}\left(\sqrt{\frac{\log(n_1+n_2) + t}{m}} + \frac{\log(n_1+n_2) + t}{m}\right)\right) \le \exp(-t).\label{eq:Gaussian_ensemble}
\end{align}
\item Chi-square upper tail \cite[Lemma~1]{laurent2000adaptive}:
\begin{align}
\PP\left(\|\by\|_2^2 \ge \|\bX\|_{\fro}^2\frac{m+2\sqrt{mt}+2t}{m}\right) \le \exp(-t).\label{eq:Gaussian_chi}
\end{align}
\item Gaussian covariance \cite[Theorem~5]{cai2020non}:
\begin{align}
\PP\left(\left\|\bB\bB^{\top}-\EE[\bB\bB^{\top}]\right\|_{\op} \ge C\sigma^2\left((\sqrt{n_1} + \sqrt{n_2} + \sqrt{\log(n_1\wedge n_2)} + \sqrt{t})^2 - n_2\right)\right) \le \exp(-t).\label{eq:Gaussian_cov}
\end{align}
\end{enumerate}
\end{lemma}

We now proceed to prove \eqref{eq:perturb_G}. In what follows, we take $t \asymp \log n$, and assume $m \gtrsim \log n$ to keep only the dominant terms when invoking the concentration inequalities in Lemma~\ref{lemma:Gaussian_bounds}. 

Let $\cM_1(\bcX_{\star})=\bU_{\star}\bSigma_{\star,1}\bR_{\star}^{\top}$ be its rank-$r_1$ SVD, with $\bR_{\star}\in\RR^{n_2n_3\times r_1}$ containing right singular vectors. Denote $\bR_{\star\perp}$ as the orthogonal complement of $\bR_{\star}$. We have the following decomposition
\begin{align*}
\cM_{1}(\cA^*(\by))\cM_{1}(\cA^*(\by))^{\top} = \cM_{1}(\cA^*(\by))\bR_{\star}\bR_{\star}^{\top}\cM_{1}(\cA^*(\by))^{\top} + \cM_{1}(\cA^*(\by))\bR_{\star\perp}\bR_{\star\perp}^{\top}\cM_{1}(\cA^*(\by))^{\top}.
\end{align*}
By the triangle inequality, we bound
\begin{align}
\|\bG-\bG_{\star}\|_{\op} &\le \left\|\cM_{1}(\cA^*(\by))\bR_{\star}\bR_{\star}^{\top}\cM_{1}(\cA^*(\by))^{\top}-\cM_1(\bcX_{\star})\cM_1(\bcX_{\star})^{\top}\right\|_{\op} \nonumber\\
&\quad + \underbrace{\left\|\cM_{1}(\cA^*(\by))\bR_{\star\perp}\bR_{\star\perp}^{\top}\cM_{1}(\cA^*(\by))^{\top} - \frac{\|\by\|_{2}^2}{m}(n_2n_3-r_1)\bI_{n_1}\right\|_{\op}}_{\eqqcolon \mfk{A}_2} \nonumber\\
&\le \underbrace{\left\|\cM_{1}(\cA^*(\by))\bR_{\star}-\bU_{\star}\bSigma_{\star,1}\right\|_{\op}^2}_{=(\mfk{A}_1)^2} + 2\underbrace{\left\|\cM_{1}(\cA^*(\by))\bR_{\star}-\bU_{\star}\bSigma_{\star,1}\right\|_{\op}}_{\eqqcolon \mfk{A}_1}\sigma_{\max}(\bcX_{\star}) + \mfk{A}_{2}. \label{eq:G_perturb_expand}
\end{align}
Here, the second line follows by applying the triangle inequality to the relation
\begin{align*}
& \cM_{1}(\cA^*(\by))\bR_{\star}\bR_{\star}^{\top}\cM_{1}(\cA^*(\by))^{\top} -\cM_1(\bcX_{\star})\cM_1(\bcX_{\star})^{\top}  = \cM_{1}(\cA^*(\by))\bR_{\star}\bR_{\star}^{\top}\cM_{1}(\cA^*(\by))^{\top} -\bU_{\star}\bSigma_{\star,1}^2 \bU_{\star}^{\top} \\
&\quad =\left(  \cM_{1}(\cA^*(\by))\bR_{\star}  - \bU_{\star}\bSigma_{\star,1} \right)\left(  \cM_{1}(\cA^*(\by))\bR_{\star}  - \bU_{\star}\bSigma_{\star,1} \right)^{\top}  +  \bU_{\star}\bSigma_{\star,1} \left(  \cM_{1}(\cA^*(\by))\bR_{\star}  - \bU_{\star}\bSigma_{\star,1} \right)^{\top}   \\
& \qquad\qquad\qquad + \left(  \cM_{1}(\cA^*(\by))\bR_{\star}  - \bU_{\star}\bSigma_{\star,1} \right) \left( \bU_{\star}\bSigma_{\star,1}\right)^{\top}.
\end{align*}
We proceed to bound the terms in \eqref{eq:G_perturb_expand} separately.
\begin{itemize} 
\item For the first term $\mfk{A}_1$, we can expand
\begin{align*}
\cM_{1}(\cA^*(\by))\bR_{\star}=\sum_{i=1}^{m}y_i\cM_{1}(\bcA_i)\bR_{\star},
\end{align*}
where $\cM_1(\bcA_i)\bR_{\star}\in\RR^{n_1\times r_1}$ has i.i.d.~$\cN(0,1/m)$ entries, and 
\begin{align*}
y_i=\langle\cM_1(\bcA_i)\bR_{\star}, \bU_{\star}\bSigma_{\star,1}\rangle\sim\cN(0,\|\bcX_{\star}\|_{\fro}^2/m).
\end{align*}
Apply inequality~\eqref{eq:Gaussian_ensemble} in Lemma~\ref{lemma:Gaussian_bounds} to obtain with overwhelming probability that
\begin{align}
\mfk{A}_1 = \left\|\sum_{i=1}^{m}y_i\cM_1(\bcA_i)\bR_{\star} - \bU_{\star}\bSigma_{\star,1}\right\|_{\op} \lesssim \sqrt{\frac{n \log n}{m}}\|\bcX_{\star}\|_{\fro}. \label{eq:A1_bound}
\end{align}
\item Regarding the second term $\mfk{A}_2$, one has
\begin{align*}
\cM_{1}(\cA^*(\by))\bR_{\star\perp}=\sum_{i=1}^{m}y_i\cM_{1}(\bcA_i)\bR_{\star\perp}.
\end{align*}
By construction, $y_i$ is independent of $\cM_{1}(\bcA_i)\bR_{\star\perp}$. Therefore, conditioned on $\by$, $\cM_{1}(\cA^*(\by))\bR_{\star\perp}\in\RR^{n_1\times (n_2n_3-r_1)}$ is a random matrix with i.i.d.~$\cN(0,\|\by\|_2^2/m)$ entries. We can apply inequality~\eqref{eq:Gaussian_cov} in Lemma~\ref{lemma:Gaussian_bounds} to obtain with overwhelming probability that 
\begin{align*}
\mfk{A}_2 &\lesssim \frac{\|\by\|_{2}^2}{m} \left((\sqrt{n_1} + \sqrt{n_2n_3-r_1} + c\sqrt{\log n})^2 - (n_2n_3 - r_1)\right) \\
&\lesssim \frac{\|\by\|_{2}^2}{m} \left(\sqrt{n_1n_2n_3} + n\sqrt{\log n}\right).
\end{align*}
Inequality~\eqref{eq:Gaussian_chi} in Lemma~\ref{lemma:Gaussian_bounds} tells that $\|\by\|_2^2 \lesssim \|\bcX_{\star}\|_{\fro}^2$ with overwhelming probability, which implies
\begin{align}
\mfk{A}_2 \lesssim \frac{\sqrt{n_1n_2n_3} + n\sqrt{\log n}}{m}\|\bcX_{\star}\|_{\fro}^2. \label{eq:A2_bound}
\end{align}
\end{itemize}

Finally, plug the bounds \eqref{eq:A1_bound} and \eqref{eq:A2_bound} into \eqref{eq:G_perturb_expand} to conclude
\begin{align*}
\|\bG-\bG_{\star}\|_{\op} \lesssim \frac{\sqrt{n_1n_2n_3} + n\log n}{m}\|\bcX_{\star}\|_{\fro}^2 + \sqrt{\frac{n\log n}{m}}\|\bcX_{\star}\|_{\fro}\sigma_{\max}(\bcX_{\star}).
\end{align*}

\end{document}